\documentclass[11pt]{article}

\usepackage[utf8]{inputenc} 

\pdfoutput=1

\usepackage{algorithm}
\usepackage{algorithmic}
\usepackage{amsmath}
\usepackage{amssymb}
\usepackage{amsthm,amsfonts}
\usepackage{array}
\usepackage{authblk}
\usepackage{babel}
\usepackage{booktabs}
\usepackage{caption}
\usepackage{colortbl}
\usepackage{enumitem}
\usepackage[T1]{fontenc}    
\usepackage{framed}
\usepackage{fullpage}
\usepackage{graphicx}
\usepackage[colorlinks, linkcolor=red, anchorcolor=blue, citecolor=blue]{hyperref}
\usepackage{indentfirst}
\usepackage[utf8]{inputenc}
\usepackage{lipsum}        
\usepackage{makecell}
\usepackage{mathrsfs}
\usepackage{mathtools}

\usepackage[framemethod=default]{mdframed}
\usepackage{microtype}     
\usepackage{multicol}
\usepackage{multirow}
\usepackage{natbib}
\usepackage{nicefrac}       
\usepackage{tablefootnote}
\usepackage{url}            
\usepackage{wrapfig,lipsum}
\usepackage{xcolor}         
\usepackage{xspace}

\usepackage{doi}
\usepackage{tikz}
\usetikzlibrary{decorations.pathreplacing}
\theoremstyle{plain}
\newtheorem{theorem}{Theorem}[section]

\newtheorem{lemma}[theorem]{Lemma}
\newtheorem{corollary}[theorem]{Corollary}
\theoremstyle{definition}

\theoremstyle{remark}
\newtheorem{remark}[theorem]{Remark}

\usepackage{amsmath,amsfonts,bm}

\def\1{\bm{1}}

\def\rvq{{\mathbf{q}}}

\def\rvx{{\mathbf{x}}}
\def\rvy{{\mathbf{y}}}
\def\rvz{{\mathbf{z}}}

\def\vzero{{\bm{0}}}
\def\vone{{\bm{1}}}

\def\va{{\bm{a}}}

\def\ve{{\bm{e}}}

\def\vq{{\bm{q}}}

\def\vw{{\bm{w}}}
\def\vx{{\bm{x}}}
\def\vy{{\bm{y}}}
\def\vz{{\bm{z}}}

\def\mA{{\bm{A}}}
\def\mB{{\bm{B}}}
\def\mC{{\bm{C}}}
\def\mD{{\bm{D}}}

\def\mI{{\bm{I}}}

\def\mM{{\bm{M}}}

\def\mQ{{\bm{Q}}}
\def\mR{{\bm{R}}}

\def\mU{{\bm{U}}}

\DeclareMathAlphabet{\mathsfit}{\encodingdefault}{\sfdefault}{m}{sl}
\SetMathAlphabet{\mathsfit}{bold}{\encodingdefault}{\sfdefault}{bx}{n}

\def\gG{{\mathcal{G}}}

\def\gK{{\mathcal{K}}}
\def\gL{{\mathcal{L}}}
\def\gM{{\mathcal{M}}}

\def\gQ{{\mathcal{Q}}}

\def\gV{{\mathcal{V}}}

\def\gY{{\mathcal{Y}}}

\def\Pr{\mathbb{P}}

\newcommand{\E}{\mathbb{E}}

\newcommand{\R}{\mathbb{R}}

\newcommand{\KL}[2]{\mathrm{KL}\left(#1 \big\| #2\right)}
\newcommand{\Breg}[2]{D_\mathrm{\phi}\left(#1 \big\| #2\right)}

\newcommand{\TVD}[2]{\mathrm{TV}\left(#1, #2\right)}

\newcommand{\der}{\mathrm{d}}

\DeclareMathOperator*{\argmin}{arg\,min}

\makeatletter
\def\thickhline{%
  \noalign{\ifnum0=`}\fi\hrule \@height \thickarrayrulewidth \futurelet
   \reserved@a\@xthickhline}
\def\@xthickhline{\ifx\reserved@a\thickhline
               \vskip\doublerulesep
               \vskip-\thickarrayrulewidth
             \fi
      \ifnum0=`{\fi}}
\makeatother

\newlength{\thickarrayrulewidth}
\newcommand{\numMask}[1]{\mathrm{numK}\left(#1\right)}
\newcommand{\idxK}{\mathrm{K}}

\newcommand{\DfId}[2]{\mathrm{DiffIdx}\left(#1, #2\right)}

\NewDocumentCommand{\yl}{ mO{} }{\textcolor{brown}{\textsuperscript{\textit{YL}}\textrm{{\small[#1]}}}}

\newcommand{\ourmethod}{AATU}

\NewDocumentCommand{\nikki}{ mO{} }{\textcolor{purple}{\textsuperscript{\textit{Nikki}}\textrm{{\small[#1]}}}}

\usepackage{graphbox}

\title{On the $\epsilon$-Free Inference Complexity of Absorbing Discrete Diffusion}

\author[1]{\normalsize Xunpeng Huang$^*$}
\author[2]{Yingyu Lin$^*$}
\author[3]{Nishant Jain}
\author[1]{Kaibo Wang}
\author[4]{\\Difan Zou$^\dagger$}
\author[2]{Yian Ma$^\dagger$}
\author[3]{Tong Zhang}

\affil[1]{The Hong Kong University of Science and Technology}
\affil[2]{University of California San Diego}
\affil[3]{University of Illinois Urbana-Champaign}
\affil[4]{The University of Hong Kong}

\begin{document}

\date{}

\def\thefootnote{*}\footnotetext{Equal contribution}
\def\thefootnote{$\dagger$}\footnotetext{Mail to \href{yianma@ucsd.edu}{yianma@ucsd.edu}, \href{dzou@cs.hku.hk}{dzou@cs.hku.hk}}

\maketitle

\begin{abstract}

Absorbing discrete diffusion has emerged as a dominant framework for discrete data generation. 
However, a significant disparity remains between its empirical success and theoretical understanding: existing analyses fail to demonstrate a complexity advantage over the $\mathcal{O}(d \ln(d/\epsilon))$ baseline established for \emph{uniform} discrete diffusion. 
We bridge this gap by identifying a critical structural advantage: whereas uniform diffusion redundantly re-denoises valid elements, the absorbing scheme denoises each absorbing state exactly once. 
Leveraging this insight, we introduce \emph{Absorbing-Aware Truncated Uniformization} (\ourmethod). 
We prove that \ourmethod\ achieves $\epsilon$-TV convergence with $\mathcal{O}(d \ln d)$ complexity—\emph{independent} of the error tolerance $\epsilon$—thereby strictly outperforming existing uniform baselines. 
Beyond improving convergence rates, our analysis eliminates the restrictive bounded-score assumption commonly required in prior studies of uniformization-based inference.
Furthermore, we extend \ourmethod\ to time-invariant parameterizations, showing that it naturally adopts an imputation-type inference with a uniformly randomized denoising order. 
When combined with a lazy update strategy, TV convergence requires only $\mathcal{O}(d)$ discrete score evaluations.
These results not only establish a rigorous foundation for absorbing discrete diffusion---confirming its efficiency in high-accuracy generation---but also open new avenues for analyzing diffusion-based language models under the masking paradigm.


\end{abstract}

\section{Introduction}

Diffusion language models~\citep{sohl2015deep,hoogeboomautoregressive,austin2021structured,lou2024discrete,ou2024your} 
have recently emerged as a powerful class of generative models, often seen as complements or competitors to auto-regressive approaches~\citep{achiam2023gpt,touvron2023llama,zhao2023survey}.
While auto-regressive methods build token-by-token conditional distributions, diffusion language models target the entire joint distribution of a token sequence through a noising--denoising framework.
During the forward (\emph{noising}) stage, tokens are gradually corrupted to an absorbing state, driving the distribution toward an absorbing Dirac measure.
Then, in the reverse (\emph{denoising}) stage, the model estimates discrete scores (i.e., density ratios) to reconstruct the original text from corrupted samples.

Although absorbing discrete diffusion empirically outperforms its \emph{uniform} counterpart (where the forward process converges to a uniform stationary distribution)~\citep{lou2024discrete,austin2021structured,campbell2022continuous}, theoretically characterizing its computational efficiency in high-precision regimes remains a challenge.
As summarized in Table~\ref{tab:comp_old}, the majority of theoretical analyses focus on \emph{uniform discrete diffusion}.
Within this context, Euler-type samplers have received the most attention due to their ability to parallelize denoising across dimensions via conditional independence.
Variants such as exponential-integrator methods~\citep{zhang2024convergence} and \(\tau\)-leaping~\citep{campbell2022continuous, lou2024discrete} approximate continuous-time scores by holding them constant over short intervals; however, this approach yields a query complexity that is polynomial in the inverse total variation (TV) error tolerance \(\epsilon\).
In contrast, uniformization-based techniques~\citep{chen2024convergence, huang2025almost} achieve \(\mathcal{O}(d\ln(d/\epsilon))\) complexity by performing unbiased simulations of the reverse Markov chain.
Regarding \emph{absorbing discrete diffusion}, \citet{liang2025absorb} analyzed \(\tau\)-leaping and established a similar error dependence, albeit with improved dimensional scaling relative to uniform diffusion.
While their uniformization-style analysis also recovers the \(\mathcal{O}(d\ln(d/\epsilon))\) rate, it suffers from two notable limitations: (1) reliance on stronger bounded-score assumptions, and (2) an inability to improve upon the \(\epsilon\)-dependence of uniform baselines~\citep{huang2025almost}.
Consequently, existing theory has yet to fully substantiate the empirical advantages observed in the absorbing setting.

In this paper, we first examine uniformization-based methods for absorbing discrete diffusion. Uniformization reformulates the reverse Continuous-Time Markov Chain (CTMC) as a Discrete-Time Markov Chain (DTMC) by sampling random Poisson jump times, preserving its transition structure while interpreting discrete scores as transition logits. However, controlling transition rates in this framework often requires bounded neural discrete score approximations.
To address this limitation, we propose \emph{Absorbing-Aware Truncated Uniformization} (\ourmethod). By truncating neural scores with a state-dependent threshold, \ourmethod\ removes the need for bounded-score assumptions while retaining the unbiased nature of the simulation, thereby leaving the training objective intact. Crucially, because the threshold scales with the number of absorbing states, \ourmethod\ achieves $\epsilon$-TV convergence in a complexity independent of $\epsilon$. This strictly improves upon the $\mathcal{O}(\ln(1/\epsilon))$ overhead in uniform discrete diffusion, reflecting a key structural insight: unlike uniform diffusion, which may repeatedly update the same token, absorbing diffusion guarantees that each token is denoised exactly once during inference.
Beyond standard discrete score parameterizations, we extend \ourmethod\ to time-invariant parameterizations~\citep{ou2024your,zheng2024masked}, where transition logits decompose into a time-dependent coefficient and a time-independent term approximating the clean-data conditional distribution. In this setting, the generative process for high-dimensional data depends not only on transition probabilities but also on the specific sequence in which variables are denoised. We show that applying \ourmethod\ to these parameterizations naturally induces an imputation-based inference algorithm with a principled denoising schedule. Finally, by coupling \ourmethod\ with lazy updates, the generated distribution achieves $\epsilon$-TV convergence using only $\mathcal{O}(d)$ discrete score evaluations.
Our main contributions are summarized as follows:
\begin{itemize}
    \item We propose a new method called \emph{Absorbing-Aware Truncated Uniformization} (\ourmethod). Unlike simply applying uniformization to absorbing discrete diffusion~\citep{liang2025absorb}, our approach leverages a truncated outgoing rate, eliminating the need for a score-bounded assumption. Furthermore, our truncation adapts to the number of absorbing states, in contrast to~\citet{huang2025almost} which uses a fixed uniform constant, thus fully exploiting the properties of absorbing discrete diffusion.
    
    \item By exploiting the fact that tokens cannot be denoised multiple times, \ourmethod\ substantially speeds convergence on the discrete space \(\{1,2,\dots,\idxK\}^d\). Specifically, to achieve \(\epsilon\)-TV convergence, \ourmethod\ requires an expected number of discrete score calls upper bounded by
    \[
        2K(d-\epsilon^2/4) + 12Kd\ln d.
    \]
    Compared to uniformization-based sampling in uniform discrete diffusion~\citep{huang2025almost,liang2025absorb}, this result not only improves upon the \(\mathcal{O}\bigl(\ln(1/\epsilon)\bigr)\) rate but also surpasses linear convergence constraints. Moreover, its dependence on vocabulary size \(K\) and dimension \(d\) aligns with state-of-the-art performance~\citep{zhang2024convergence}.
    
    \item Extending our analysis beyond time-dependent parameterizations, we investigate the application of \ourmethod\ to time-invariant settings~\citep{ou2024your,zheng2024masked}. In this context, we demonstrate that coupling \ourmethod\ with a lazy update strategy guarantees \(\epsilon\)-TV convergence using only \(\mathcal{O}(d)\) discrete score evaluations. Furthermore, this extension provides a theoretical justification for the use of random denoising orders in modern diffusion language models.

\end{itemize}

\paragraph{Comparison with concurrent work.}
Following the submission of this manuscript to a conference, a concurrent study by \citet{liang2026sharpconvergenceratesmasked} appeared on arXiv. Similar to our work, \citet{liang2026sharpconvergenceratesmasked} investigates convergence guarantees for absorbing-rate discrete diffusion models and establishes an $O(d)$ complexity bound for the time-invariant parameterization (alongside an analysis of the Euler-type sampler, which falls outside our scope). Despite these thematic intersections, our work diverges in three significant aspects.
First, our analysis is feasible for the general time-variant parameterization setting, achieving a complexity of $O(d\ln d)$; this regime is not considered in \citet{liang2026sharpconvergenceratesmasked}. Second, we explicitly bridge the theory and implementation between uniformization-type samplers and iterative imputation algorithms—the most widely used class of inference methods for discrete diffusion—a connection that remains unexplored in the concurrent work. Finally, while both papers derive $O(d)$ complexity bounds for iterative imputation, they apply to distinct algorithmic variants: our analysis focuses on the standard implementation using uniformly random choices, whereas \citet{liang2026sharpconvergenceratesmasked} analyzes the FHS scheme proposed by \citet{zhengmasked}.

\section{Preliminaries}
\label{sec:pre}
In this section, we establish the notation and setup for both forward and reverse Markov processes in general discrete diffusion models. We discuss marginal and conditional distributions, the transition rate function, neural-network-parameterized discrete scores (density ratios), and a standard training objective. We also present the commonly adopted assumption on score estimation error, which underlies many theoretical and empirical works~\citep{zhang2024convergence,lou2024discrete,chen2024convergence,huang2025almost,liang2025absorb}. 
A comprehensive summary of the notation can be found in Table~\ref{tab:notations} of Appendix~\ref{sec:app_notations}.

\paragraph{The forward process notations.}
In this paper, we consider discrete distributions over 
\(
\gY = \{1,2,\ldots,\idxK\}^d
\).
For any functions \(f,g: \gY\rightarrow \mathbb{R}\), we define their inner product as
\[
\left<f,g\right>_{\gY} = \sum_{\vy\in\gY} f(\vy)\cdot g(\vy).
\]
Given a target distribution \(q_*\), we define a forward Markov process 
\(
\{\rvy^\to_t\}_{t=0}^T
\)
with \(q_0^\to = q_*\), which converges to a stationary distribution \(q^\to_\infty\) as \(T \to \infty\).
We denote by \(q_t^\to\) its marginal at time \(t\), and use 
\(
q_{t^\prime,t}^\to(\vy^\prime,\vy)
\)
and
\(
q_{t^\prime|t}^\to(\vy^\prime|\vy)
\)
to represent the joint and conditional distributions over times \(t^\prime\) and \(t\), respectively:
\[
(\rvy^\to_{t^\prime}, \rvy^\to_t)
\,\sim\,
q^\to_{t^\prime,t},
\quad
q_{t^\prime|t}^\to(\vy^\prime|\vy)
\,=\,
q^\to_{t^\prime,t}(\vy^\prime,\vy)/{q_t^\to(\vy)}
\quad
\text{for }
t^\prime>t.
\]
Both absorbing and uniform discrete diffusion models treat this forward process as a time-homogeneous CTMC with a transition rate function
\(
R^\to\colon \gY\times \gY \rightarrow \mathbb{R}
\)
which denotes the instantaneous transition rate from \(\vy^\prime\) to \(\vy\). Formally,
\begin{equation}
    \label{def:mean_of_R}
    R^\to(\vy, \vy^\prime) \coloneqq  \lim_{\Delta t\rightarrow 0}\left[(q^\to_{\Delta t|0}(\vy|\vy^\prime)-\delta_{\vy^\prime}(\vy))/{\Delta t}\right]
\end{equation}
where
\(\delta_{\vy^\prime}(\vy) = 1\) if 
\(\vy = \vy^\prime\) and \(0\) otherwise.
We further define
\(
R^\to(\vy^\prime)
\coloneqq 
\sum_{\vy\neq\vy^\prime} R^\to(\vy,\vy^\prime)
\)
as the outgoing rate, which denotes the instantaneous transition rate from $\vy^\prime$ to all other feasible states.
Under this condition,
the discrete forward process follows
\begin{equation}
    \label{eq:fwd_func}
    \frac{\der q^\to_{t|s}}{\der t}(\vy|\vy_0) = \left<R^\to(\vy,\cdot), q^\to_{t|s}(\cdot|\vy_0)\right>_{\gY}, 
    \quad
    \frac{\der q^\to_t}{\der t}(\vy)
    = \left<R^\to(\vy,\cdot), q^\to_t(\cdot)\right>_{\gY}.
\end{equation}
More details and derivation can be found in Appendix~\ref{sec:a1_fwd_and_infi}.

\paragraph{The reverse process notations.}
To sample from \(q_* = q_0^\to\), discrete diffusion models define a reverse process 
\(\{\rvy^\gets_t\}_{t=0}^T\)
such that
\(\rvy_t^\gets \sim q_t^\gets = q_{T-t}^\to\)
and
\((\rvy_{t^\prime}^\gets,\rvy_t^\gets)\sim q_{t^\prime,t}^\gets\).
By Lemma~\ref{lem:absorbing_reverse_de} (proof in Appendix~\ref{app_sec:prof_absorbing_reverse}), this time-inhomogeneous Markov chain satisfies:

\begin{lemma}
\label{lem:absorbing_reverse_de}
The probability mass function \(q^\gets_t\) in the reverse process follows
\begin{equation}
    \label{eq:rev_func}
    \frac{\mathrm{d}\,q^\gets_t}{\mathrm{d}\,t}(\vy) 
    \;=\; 
    \langle R^\gets_t(\vy,\cdot),\,q^\gets_t(\cdot)\rangle_{\gY}
    \quad 
    \text{where}
    \quad 
    R_t^\gets(\vy, \vy^\prime) 
    \;\coloneqq\; 
    R^\to(\vy^\prime, \vy)\,\frac{q^\gets_t(\vy)}{q^\gets_t(\vy^\prime)},
\end{equation}
and the reverse transition function \(\mR^\gets_t\) arises as the infinitesimal operator of the reverse process:
\begin{equation}
    \label{def:mean_of_R_gets}
    R^\gets_t(\vy, \vy^\prime) 
    \;\coloneqq\;  
    \lim_{\Delta t\rightarrow 0}
    \Bigl[\,
        (q^\gets_{t+\Delta t \mid t}(\vy \mid \vy^\prime) - \delta_{\vy^\prime}(\vy))/{\Delta t}
    \Bigr],
\end{equation}
while the outgoing rate is 
\(
R^\gets_t(\vy^\prime) 
= 
\sum_{\vy\neq\vy^\prime} R^\gets_t(\vy,\vy^\prime).
\)
\end{lemma}

Under this formulation, the reverse transition rate \(R^\gets_t\) depends on the forward transition rate \(R^\to\) as well as the \emph{discrete score}, defined as the density ratio 
\({q_t^\gets(\vy)} / {q_t^\gets(\vy^\prime)}\).
Since this ratio is generally intractable, it is approximated in practice by a neural network \(\tilde{v}\):
\begin{equation}
    \label{def:discrete_score_est}
    \tilde{v}_{t,\vy^\prime}(\cdot) 
    \;\approx\; 
    v_{t,\vy^\prime}(\cdot) 
    \;=\; 
    q_t^\gets(\cdot)/{q_t^\gets(\vy^\prime)},
\end{equation}
yielding an approximate reverse transition rate \(\tilde{R}_t^\gets\) via Eq.~\eqref{eq:rev_func}. 
To train \(\tilde{v}\), one typically uses the \emph{score entropy} loss~\citep{lou2024discrete,benton2024denoising},
\begin{equation}
    \label{eq:score_estimation_main}
    L_{\text{SE}}(\tilde{v}) 
    \;=\; 
    \frac{1}{T}\,\int_0^T 
    \mathbb{E}_{\rvy_t\sim q^\to_t}
    \Biggl[
        \sum_{\vy\neq\rvy_t} 
        R^\to(\rvy_t, \vy)\,
        \Breg{v_{T-t, \rvy_t}(\vy)}{\tilde{v}_{T-t, \rvy_t}(\vy)}
    \Biggr]
    \mathrm{d}t,
\end{equation}
where \(\Breg{\cdot}{\cdot}\) is the Bregman divergence associated with \(\phi(c) = c \ln c\). 
As in continuous diffusion~\citep{chen2023sampling}, practitioners often replace \(L_{\text{SE}}\) by \emph{implicit} or \emph{denoising score entropy}~\citep{lou2024discrete,benton2024denoising} for more tractable optimization but invariant minimum.

\paragraph{General Assumptions.}
To analyze both convergence properties and the computational effort required for achieving TV distance convergence in practical settings, we assume the score entropy loss will be upper-bounded. Formally:
\begin{enumerate}[label=\textbf{[A{\arabic*}]},start=1]
    \item \label{ass:score_approximation_error} 
    \textbf{Score approximation error.} 
    The discrete score \(\tilde{v}_t\) obtained from Eq.~\eqref{eq:score_estimation_main} is well-trained, and its estimation error is small enough so that
    \(
        L_{\text{SE}}(\tilde{v}) \;\le\; \epsilon_{\text{score}}^2.
    \)
\end{enumerate}
This assumption is standard in theoretical inference research~\citep{chen2024convergence,zhang2024convergence,lou2024discrete}, where it is commonly presumed that the score can be trained arbitrarily well such that \(\epsilon_{\text{score}}\le \epsilon\) for any desired \(\epsilon>0\).

\section{The Forward Process of Absorbing Discrete Diffusion}

In this section, we specialize the framework outlined in Section~\ref{sec:pre} to the setting of absorbing discrete diffusion. 
We subsequently construct a family of auxiliary distributions that converge exponentially fast to the ideal forward marginal distribution. 
By leveraging the forward transition kernel of absorbing discrete diffusion for any \(0 < s < t < T\), this construction serves as an alternative to the reverse initialization strategy proposed by~\citet{liang2025absorb}.

\paragraph{Setup and Notation.}  
Let the scalar state space be \(\{1,2,\ldots,\idxK\}\), where \(\idxK\) designates the absorbing state. 
We consider to generate high-dimensional data \(\vy \in \gY = \{1,2,\ldots,\idxK\}^d\). 
For ease of explanation, we denote the number of absorbing tokens in a vector $\vy$ and the Hamming distance between two vectors $\vy$ and $\vy^\prime$ as
\[
    \numMask{\vy} 
    \;\coloneqq\;  
    \sum_{i=1}^d 
    \delta_{\idxK}(\vy_i)\quad \text{and}\quad \mathrm{Ham}(\vy,\vy^\prime) = d-\sum_{i=1}^d \delta_{\vy_i}(\vy_i^\prime)
\] 
respectively.
We adopt the standard assumption that the absorbing state does not appear in the support of the target distribution:
\begin{enumerate}[label=\textbf{[A{\arabic*}]},start=2]
    \item 
    \label{ass:mask_init} 
    \textbf{No mask in the target distribution.} 
    The target distribution \(q_0^\to = q_* \colon \gY \rightarrow \mathbb{R}\) assigns positive probability only to states containing no absorbing tokens; that is, 
    \(q_*(\vy) > 0\) if and only if \(\numMask{\vy} = 0\).
\end{enumerate}

\paragraph{Absorbing discrete diffusion instantiation and approximation.}
We begin by specifying the absorbing forward transition rate function for discrete diffusion:
\begin{equation}
    \label{eq:fwd_transtion_rate_func}
    R^\to(\vy, \vy^\prime) \;=\;
    \begin{cases}
        1 
        & \text{if } 
          \mathrm{Ham}(\vy,\vy^\prime) = 1 
          \;\text{and}\; 
          \vy_{\DfId{\vy}{\vy^\prime}} = \idxK\\
        -\sum_{i=1}^{d} 
          \bigl[1 - \delta_{K}(\vy_i)\bigr] 
        & \text{if } 
          \vy=\vy^\prime\\
        0 
        & \text{otherwise}
    \end{cases}.
\end{equation}
Here, \(\DfId{\vy}{\vy^\prime}\) denotes the single coordinate where \(\vy\) and \(\vy^\prime\) differ.
Under this transition rule, each non-masked coordinate tends to become masked at an exponential rate. 
Concretely, for any \(0 < s < t < T\), the forward transition kernel satisfies
\begin{equation}
\label{eq:fwd_transition_kernel}
    \begin{aligned}
        q^\to_{t|s}(\vy|\vy^\prime) =\prod_{i=1}^{d} & \left[ \delta_{(\idxK,\idxK)}(\vy_i, \vy_i^\prime) + \left(1- \delta_{(\idxK,\idxK)}(\vy_i, \vy_i^\prime)\right)\cdot \delta_{0}(\vy_i-\vy_i^\prime)\cdot e^{-(t-s)} \right.\\
        &\left. + \left(1- \delta_{(\idxK,\idxK)}(\vy_i, \vy_i^\prime)\right)\cdot \delta_{\idxK}(\vy_i)\cdot (1-e^{-(t-s)}) \right],
    \end{aligned}
\end{equation}
as shown in Lemma~\ref{lem:fwd_trans_ker}. 
To approximate the forward marginal distribution \(q^\to_t\) at time \(t\), we exploit this 
exponential decay by modeling each non-absorbing coordinate under a uniform distribution and absorbing 
coordinates at a constant rate. Specifically, we define
\begin{equation}
    \label{def:tilde_q_t}
    \tilde{q}_t(\vy) \;\propto\; 
    \prod_{i=1}^d \exp\!\Bigl(-t \cdot \bigl[1 - \delta_{\idxK}(\vy_i)\bigr]\Bigr)
    \;=\; 
    \exp\!\Bigl(-\,t \cdot \bigl[d - \numMask{\vy}\bigr]\Bigr).
\end{equation}
so that \(\tilde{q}_t\) factorizes over coordinates and is straightforward to sample from. 
Moreover, as established in Lemma~\ref{lem:fwd_convergence_0}, the KL divergence between \(q^\to_t\) and \(\tilde{q}_t\) decreases exponentially with \(t\).
\begin{lemma}[Exponentially decreasing KL divergence between \(q^\to_t\) and \(\tilde{q}_t\)]
    \label{lem:fwd_convergence_0}
    Suppose the CTMC \(\{\rvy^\to_t\}_{t=0}^{T}\) has transition rates \(R^\to\) from Eq.~\eqref{eq:fwd_transtion_rate_func}, with \(\rvy^\to_t \sim q^\to_t\). 
    Let \(\tilde{q}_t\) be the approximation of \(q^\to_t\) defined by Eq.~\eqref{def:tilde_q_t}. 
    Then,
    \[
        \KL{q_t^\to}{\tilde{q}_t} \;\le\; (1 + e^{-t})^d - 1.
    \]
    Consequently, to ensure \(\KL{q^\to_t}{\tilde{q}_t} \le \epsilon\), it suffices to choose \(t \ge \ln\bigl(4d/\epsilon\bigr)\).
\end{lemma}
From Lemma~\ref{lem:fwd_convergence_0}, the running time \(T\) required for \(\tilde{q}_T\) to approximate \(q^\to_T\) falls on the order of \(\mathcal{O}(\ln(d/\epsilon))\). 
It precisely matches the forward mixing time for uniform discrete diffusion~\citep{chen2024convergence,zhang2024convergence,huang2025almost} and continuous diffusion~\citep{chen2023sampling} converging to their stationary distributions. 
Although the final results exhibit a similar convergence rate, the underlying analytical techniques differ substantially because the one-hot stationary distribution of masked discrete diffusion does not satisfy the modified log-Sobolev condition. 
Further technical details are deferred to Appendix~\ref{app_sec:prof_lem2}.

\section{Truncated Uniformization in Absorbing Discrete Diffusion}

This section investigates uniformization-based sampling algorithms for the reverse process in absorbing discrete diffusion.
We begin by reviewing the fundamentals of unbiased reverse process simulation via uniformization. 
We then demonstrate that the expected complexity of such inference is governed by the outgoing rates of the reverse transition; crucially, absorbing discrete diffusion inherently yields lower outgoing rates than its uniform counterpart, thereby facilitating faster convergence. 
Leveraging this insight, we introduce \emph{Absorbing-Aware Truncated Uniformization} (\ourmethod). 
Building on~\citet{huang2025almost}, \ourmethod\ modulates these outgoing rates to obviate bounded-score assumptions while maintaining the unbiased nature of the simulation. 
Finally, we establish theoretical guarantees regarding the convergence and computational complexity of \ourmethod, positioning our contributions within the context of existing literature.

\paragraph{Uniformization and the expected number of discrete score calls.}
Consider a time-dependent reverse transition rate \(R^\gets_t\) defined over the interval \([a,b]\). 
The evolution of the ideal reverse process for any \(\vy,\vy'\) can be described by
\begin{equation}
    \label{def:ideal_reverse_trans}
    q^\gets_{t+\Delta t\mid t}(\vy^\prime \mid \vy)
    \;=\;
    \begin{cases}
       \Delta t \cdot R^\gets_t(\vy',\vy), &\quad \vy' \neq \vy,\\[5pt]
       1-\Delta t \cdot R^\gets_t(\vy), &\quad \vy' = \vy,
    \end{cases}
    \quad \text{as } \Delta t\to 0,
\end{equation}
following Eq.~\eqref{def:mean_of_R_gets}.
If the total outgoing rate--denoting the instantaneous transition rate from $\vy$ to all other feasible states--is uniformly bounded by some \(\beta\), i.e.,
\begin{equation}
    \label{def:out_going_rate_def}
    R^\gets_t(\vy)
    \;=\;
    \sum_{\vy' \neq \vy} R^\gets_t(\vy', \vy)
    \;\le\;
    \beta_t
    \;\le\;
    \max_{t\in[a,b]} \beta_t 
    \;=\; 
    \beta,
\end{equation}
then with probability \(1 - \Delta t \cdot \beta\), the particle remains in the same state in each infinitesimal time step, thus requiring no additional score computation.

Based on this observation, the standard \emph{uniformization} method \citep{van1992approximate,van2018uniformization,chen2024convergence} simulates the reverse dynamics over \([a,b]\) by iterating the following two-step procedure in the limit \(\Delta t \to 0\):
\begin{enumerate}
    \item[\textbf{1.}] \label{step:start_transition} Sample whether a transition occurs with probability \(\Delta t \cdot \beta\).
    \item[\textbf{2.}] 
    \label{step:what_transition} If a transition occurs, move \(\rvy^\gets_t\) from \(\vy\) to \(\vy'\) with probability
    \begin{equation}
        \label{eq:step2_unif}
        \mM_t(\vy'\mid \vy)
        \;=\;
        \begin{cases}
            \beta^{-1} \,R^\gets_t(\vy', \vy), & \vy' \neq \vy,\\[2pt]
            1 - \beta^{-1} R^\gets_t(\vy), & \text{otherwise}.
        \end{cases}
    \end{equation}
\end{enumerate}
Under this update scheme, the reverse transitions of uniformization will be equivalent to Eq.~\eqref{def:ideal_reverse_trans} exactly and introduce no time-discretization error (see Appendix~\ref{app_sec:prof_convergence_reverse} for details). 
Moreover, since the number of transitions (and hence the number of discrete score computations) over \([a,b]\) follows a Poisson distribution with mean \(\beta\cdot (b-a)\), any tighter bound on \(R^\gets_t(\vy)\) reduces \(\beta\) and thereby lowers the expected inference complexity.

\paragraph{The comparison of computational complexity and outgoing rate.}
By the previous discussion of uniformization, the expected number of discrete score calls over the time interval \([0,T]\) can be approximated by
\begin{equation}
    \label{def:expected_score_eva_rough}
    \sum_{w=1}^W \max_{t\in[t_{w-1}, t_w]} \beta_t \cdot (t_w - t_{w-1}) \stackrel{W\rightarrow \infty}{\approx} \int_{t=0}^T \beta_t \der t,
\end{equation}
where \([t_0, t_1, \ldots, t_W]\) is a partition of \([0,T]\). 
In uniform discrete diffusion, \citet{chen2024convergence,huang2025almost} show that the ideal reverse process satisfies
\begin{equation}
    \label{ineq:ogr_upb_unif}
    \beta_t\coloneqq  2K\cdot {d}\cdot \max\{1, (T-t)^{-1}\}\le \beta\coloneqq 2K\cdot {d}\cdot \max\{1, (T-b)^{-1}\}\quad \forall\ t\in[a,b],
\end{equation}
providing a uniform upper bound on the total outgoing rate \(R^\gets_t(\vy)\).
For \emph{absorbing} discrete diffusion, the outgoing rates at varying times can be upper bounded as established in Lemma~\ref{lem:out_degree_rate_wrt_time} (see Appendix~\ref{app_sec:prof_out_degree_rate} for the proof).
\begin{lemma}[Bound of the outgoing rate]
    \label{lem:out_degree_rate_wrt_time}
    Consider a CTMC whose transition rate function $R^\to$ is defined as Eq.~\eqref{eq:fwd_transtion_rate_func}. Then, for any $\vy$, the reverse transition rate function satisfies
    \begin{equation}
        \label{def:beta_t_}
        \sum_{\vy^\prime\not=\vy}R^\gets_t(\vy^\prime,\vy) = R^\gets_t(\vy)\le \beta_t(\vy)\coloneqq  \frac{\numMask{\vy}\cdot K}{e^{(T-t)}-1}.
    \end{equation}
\end{lemma}
Compared to \eqref{ineq:ogr_upb_unif}, this bound explicitly depends on \(\numMask{\vy}\), the number of absorbing states in \(\vy\). Since \(\numMask{\vy}\le d\), it is strictly smaller than the uniform bound in \eqref{ineq:ogr_upb_unif}. Furthermore, \(\numMask{\vy}\) decreases monotonically as the reverse process proceeds, which progressively enlarges the gap in outgoing rate between absorbing and uniform discrete diffusion. Because a lower outgoing rate implies fewer expected discrete score evaluations for each time $t$, absorbing discrete diffusion can be significantly more computationally efficient.

From an empirical perspective, a central observation is: \emph{during inference, absorbing discrete diffusion only updates (denoises) absorbing states, whereas uniform discrete diffusion attempts to re-denoise states that have already been denoised.}
Hence, in absorbing discrete diffusion, particles are more likely to remain unchanged at each step, leading to a lower outgoing rate (and thus smaller \(\beta_t\)) over \([0,T]\). Consequently, fewer discrete score evaluations are required, underscoring the computational advantages of absorbing compared to uniform discrete diffusion.

\paragraph{Absorbing-aware truncation and algorithm proposal.}  In practice, we approximate the reverse transition rate \(R^\gets_t(\vy',\vy)\) by a learned neural score \(\tilde{v}_{t,\vy}(\vy')\), yielding approximate reverse transition rate, i.e.,
\[
    \tilde{R}_t(\vy^\prime,\vy)
    \;=\;
    R^\to(\vy,\vy^\prime)\,\tilde{v}_{t,\vy}(\vy^\prime),
\]
as dictated by Lemma~\ref{lem:absorbing_reverse_de} and Eq.~\eqref{def:discrete_score_est}. 
Because \(\tilde{v}\) is a learned estimator, the outgoing rate \(\tilde{R}_t(\vy)\) may have no explicit upper bounds, complicating control over the expected number of discrete score evaluations.
To mitigate unbounded transition rates, prior work typically imposes a bounded-score assumption 
on \(\tilde{R}_t(\vy)\), restricting it to remain below a fixed constant~\citep{liang2025absorb} or 
to grow as a function of the inference time~\citep{chen2024convergence}. However, such assumptions 
can severely impact inference efficiency because the chosen upper bound \(\beta\) directly governs 
Step 2 of uniformization, as described in Eq.~\eqref{eq:step2_unif}. When \(\beta\) is unknown, 
it can be treated as a hyperparameter. Setting \(\beta\) too small may yield an infeasible 
probability $1 - \beta^{-1} \tilde{R}_t(\vy) < 0$, forcing the algorithm to fail; setting it too 
large preserves feasibility but inflates complexity in direct proportion to \(\beta\). Thus, 
tightening this bounding scheme is crucial for balancing both correctness 
and computational efficiency in uniformization-based inference.

\begin{algorithm}[t]
    \caption{\sc Absorbing-Aware Truncated Uniformization (\ourmethod)}
    \label{alg:uni_inf}
    \begin{algorithmic}[1]
            \STATE {\bfseries Input:}  Total time $T$, a time partition $0=t_0<\ldots<t_W=T-\delta$, parameters $\beta_{t_1}, \ldots, \beta_{t_W}$ set as Eq.~\eqref{def:beta_t_}, a reverse transition rate function $\hat{R}^\gets_t$ obtained by the learnt score function $\tilde{v}_{t,\vy^\prime}(\cdot)$.
            \STATE Draw an initial sample $\hat{\rvy}_{t_0} = [\idxK, \idxK,\ldots, \idxK]$. \label{step:infer_with_trunc_uni_start}
            \FOR{$w = 1$ {\bfseries to} $W$}
                \STATE Choose $\beta_{t_w} = K\cdot \mathrm{numK}(\hat{\rvy}_{t_{w-1}})/(e^{T-t_w}-1)$
                \STATE Draw $N\sim \mathrm{Poisson}(\beta_{t_w} (t_w-t_{w-1}))$;
                \STATE Sample $N$ points i.i.d. uniformly from $[t_{w-1}, t_w]$ and sort them as $\tau_1<\tau_2<\ldots<\tau_N$;
                \STATE Set $\rvz_0 = \hat{\rvy}_{t_{w-1}}$;
                \FOR{$n = 1$ {\bfseries to} $N$}
                    \STATE Find the index set $\gM$ of absorbing states appeared in random vector $\rvz_{n-1}$
                    \STATE \label{step:rtrm_sample} For any $i\in\gM$ and $k\in\{1,2,\ldots, K-1\}$, update $\vz_{n-1}$ with
                    \begin{equation*}
                        \rvz_{n} = \left\{
                            \begin{aligned}
                                & \rvz_{n-1}[\vz_i\colon K\rightarrow k] && w.p.\ \beta_{t_{w}}^{-1}\cdot \hat{R}_{\tau_n,\vz_0}(\vz_{n-1}[\vz_i\colon K\rightarrow k],\vz_{n-1}),\\
                                & \rvz_{n-1}, && w.p.\ 1- \beta_{t_w}^{-1}\cdot \hat{R}_{\tau_n,\vz_0}(\vz_{n-1}).
                            \end{aligned}
                        \right.
                    \end{equation*}
                \ENDFOR
                \STATE Set $\hat{\rvy}_{t_w} = \rvz_{N}$. \label{step:approximate_discrte_sample}
            \ENDFOR
            \STATE {\bfseries return} $\hat{\rvy}_{t_W}$.
    \end{algorithmic}
\end{algorithm}
Motivated by \citet{huang2025almost}, we propose a \emph{absorbing-aware truncation} scheme to rescale the practical outgoing rate \(\tilde{R}_t(\vy',\vy)\). 
This ensures that the non time-discretization property is preserved without additional cost, even when \(\tilde{R}_t(\vy)\) becomes large. 
Specifically, consider simulating the reverse process over the \((w\)-th) time segment \([t_{w-1},t_w]\), assuming the state at time \(t_{w-1}\) is \(\hat{\rvy}_{t_{w-1}} = \vy_{t_{w-1}}\). 
Following from the monotonicity of $(e^{T-t}-1)^{-1}$ and $\numMask{\hat{\rvy}_t}$ in Lemma~\ref{lem:out_degree_rate_wrt_time}, the absorbing-aware truncation is chosen as $\beta_{t_w}(\vy_{t_{w-1}})$, then we set
\begin{equation}
    \label{def:prac_infi_oper_1}
    \hat{R}_{t,\vy_{t_{w-1}}}(\vy,\vy^\prime) 
    \;=\; 
    \begin{cases}
        \tilde{R}_t(\vy,\vy^\prime)\;{\beta_{t_w}(\vy_{t_{w-1}})}/{\tilde{R}_t(\vy^\prime)}, 
            & \text{if } \tilde{R}_t(\vy^\prime)>\beta_{t_w}(\vy_{t_{w-1}}),\\[0.6em]
        \tilde{R}_t(\vy,\vy^\prime), 
            & \text{otherwise},
    \end{cases}
    \quad 
    \forall \vy^\prime \!\neq\! \vy,
\end{equation}
and
\begin{equation}
    \label{def:prac_infi_oper_2}
    \hat{R}_{t,\vy_{t_{w-1}}}(\vy^\prime,\vy^\prime) 
    \;=\; 
    -\,\sum_{\vy \neq \vy^\prime}\!\hat{R}_{t,\vy_{t_{w-1}}}(\vy,\vy^\prime).
\end{equation}
With these truncations, the corrected outgoing rate will definitely be upper-bounded by $\beta_{t_w}(\vy_{t_{w-1}})$.
Then, we obtain a practical and efficient inference algorithm, summarized in Alg.~\ref{alg:uni_inf}.

\begin{table*}[t]
    \centering
    \caption{\small Comparison with prior works simulating reverse particle SDEs, where \textcolor{red}{\bf[A3]} denotes the bounded-score assumption used in~\citet{chen2024convergence} and \textcolor{red}{\bf[A3]+} denotes the bounded-score assumption used in~\citet{liang2025absorb} which is a little bit stronger than \textcolor{red}{\bf[A3]} due to the time-invariant requirement.  All complexities are on TV convergence (or TV convergence deduced from KL convergence), which are achieved by assuming $\epsilon_{\text{score}} = \tilde{o}(\epsilon)$ and setting early-stopping parameters $\delta=\epsilon/d$. Besides, the complexity presented by $\tilde{\mathcal{O}}(\cdot)$ means the $\ln$ dependencies are omitted.} 
    \small
    \renewcommand\arraystretch{1.3}
    \begin{tabular}{ccccc}
    \toprule
     Results & Forward Type & Inference Sampler &  Assumptions &  Complexity \\
     \midrule
    \citet{zhang2024convergence} &  Uniform & Exponential Integrator  & \ref{ass:score_approximation_error}, \textcolor{red}{\bf[A3]} 
    & $\tilde{\mathcal{O}}(d^{5/3}\epsilon^{-2})$\\
     \citet{ren2024discrete}  & Uniform & $\tau$-leaping & 
     \ref{ass:score_approximation_error},\textcolor{red}{\bf[A3]}
 & $\tilde{\mathcal{O}}(d^2\epsilon^{-2})$\\
     \citet{chen2024convergence}  & Uniform & Uniformization & \ref{ass:score_approximation_error},\textcolor{red}{\bf[A3]} & $\mathcal{O}(d\ln(d/\epsilon))$\\
     \citet{huang2025almost}  & Uniform & Truncated Uniformization & \ref{ass:score_approximation_error}
     & $\mathcal{O}(d\ln(d/\epsilon))$\\
     \midrule
     \citet{liang2025absorb}  & Absorbing & $\tau$-leaping & 
     \ref{ass:score_approximation_error},\ref{ass:mask_init},\textcolor{red}{\bf[A3]+}
 & $\mathcal{O}(d\epsilon^{-2})$\\
     \citet{liang2025absorb}  & Absorbing & Uniformization & 
     \ref{ass:score_approximation_error},\ref{ass:mask_init},\textcolor{red}{\bf[A3]}
 & $\mathcal{O}(d\ln(d/\epsilon))$\\
     Theorem~\ref{thm:main_mask_unif}  & Absorbing & \ourmethod & \ref{ass:score_approximation_error},\ref{ass:mask_init}  & \textcolor{red}{$\mathcal{O}(d\ln d)$}\\
     \bottomrule 
    \end{tabular}
    \label{tab:comp_old}
    \vspace{-.15in}
\end{table*}
\paragraph{Theoretical results.} Theorem~\ref{thm:main_mask_unif} summarizes the convergence guarantees and complexity analysis of Algorithm~\ref{alg:uni_inf} in approximating $q_*$. Detailed proofs are provided in Appendices~\ref{app_sec:prof_convergence_reverse} and~\ref{app_sec:prof_complexity_reverse}.
\begin{theorem}[Combination of Theorem~\ref{thm:convergence_unif_reverse} and Theorem~\ref{thm:mask_unif_complexity}]
    \label{thm:main_mask_unif}
    Suppose Assumptions~\ref{ass:score_approximation_error} and~\ref{ass:mask_init} hold. For 
    Alg.~\ref{alg:uni_inf}, if we require 
    \begin{equation*}
        T = \ln(4d/\epsilon^2),\quad \delta\le d^{-1}\epsilon,\quad  \epsilon_{\text{score}}\le T^{-1/2} \epsilon, \quad \epsilon<1,
    \end{equation*}
    and the partition of the reverse process satisfies
    \begin{equation*}
        \eta = \epsilon/2d,\quad  W = (T-\delta)/\eta,\quad t_0=0,\quad t_W=T-\delta,\quad   t_{w} - t_{w-1} = \eta \quad \forall w\in\{1,2,\ldots W\}
    \end{equation*}
    the expected iteration and score estimation complexity of Alg.~\ref{alg:uni_inf} is upper bounded by
    \begin{equation}
        \label{def:mask_unif_complexity_fin}
        2K(d-\epsilon^2/4) + 12Kd\ln d
    \end{equation}
    to achieve $\TVD{p_*}{\hat{p}}\le 2\epsilon$, where $\hat{p}$ denotes the underlying distribution of generated samples.
\end{theorem}

Equation~\eqref{def:mask_unif_complexity_fin} suggests that exact inference might be theoretically achievable by setting $\epsilon = 0$. 
However, this regime is practically infeasible as it necessitates infinite mixing time $T$, perfect score estimation ($\epsilon_{\text{score}}=0$), and an infinite number of discretization intervals $W$. 
Crucially, although the step size $\eta = \epsilon/(2d)$ implies a total of $\mathrm{poly}(d/\epsilon)$ intervals in the reverse process, the total number of score evaluations remains essentially independent of $\epsilon$. This occurs because the majority of intervals involve no state transitions (see Eq.~\eqref{def:expected_score_eva_rough}). 
Consequently, the use of fine-grained intervals serves primarily to ensure the validity of the outgoing rate upper bound without inflating the overall computational complexity.
Table~\ref{tab:comp_old} presents a comparative complexity analysis. 
\ourmethod\ achieves state-of-the-art complexity results, notably remaining $\epsilon$-free while avoiding the restrictive assumption of a bounded score estimator.
Given the arbitrary choice of the absorbing state, Assumption~\ref{ass:mask_init} is relatively mild and aligns with practical configurations, such as those found in diffusion language models~\cite{nie2025large}.
To provide a comprehensive analysis, we also address the convergence properties of \ourmethod\ in scenarios where Assumption~\ref{ass:mask_init} does not hold:

\begin{corollary}
    \label{cor:conv_wo_A2}
    Suppose Assumption~\ref{ass:score_approximation_error} holds. Adopting the parameter settings from Theorem~\ref{thm:convergence_unif_reverse}, the expected iteration and score estimation complexity of Alg.~\ref{alg:uni_inf} is upper bounded by
    \begin{equation*}
        \mathcal{O}\left(Kd\cdot \min\left\{\ln(d/\epsilon), \E[\numMask{\rvy^\to_0}]/\epsilon \right\}\right) + \mathcal{O}(Kd\ln d)
    \end{equation*}
    to achieve $\TVD{q_*}{\hat{q}}\le 2\epsilon$, where $\hat{p}$ denotes the underlying distribution of generated samples.
\end{corollary}

This result implies that the complexity remains independent of $\epsilon$ when the expected number of absorbing states in the target distribution is on the order of the error tolerance. Conversely, the presence of absorbing states in the data distribution prevents the sharp contraction of the outgoing rate upper bound observed previously, potentially introducing a logarithmic dependence on $\epsilon$ (specifically, a $\ln(1/\epsilon)$ factor). 
The analysis of this corollary is deferred to Appendix.~\ref{app_sec:prof_cor_conv_wo_A2}.

\paragraph{Empirical Results.} 
We empirically validate the acceleration of absorbing discrete diffusion against uniform baselines on synthetic and real-world data, including text generation tasks following~\citet{lou2024discrete}. Due to space constraints, detailed results are deferred to Appendix~\ref{app_sec:exps}.


\section{Time-Invariant Parameterization for \ourmethod}

In this section, we extend the analysis of \ourmethod\ to absorbing diffusion models that utilize a time-invariant parameterization of discrete scores.
While our previous sections focus on the general time-variant parameterization, recent work \citep{ou2024your,zheng2024masked} has shown that under certain conditions, absorbing diffusion can be streamlined into time-agnostic masked models.
By adapting \ourmethod\ to this parameterization, we derive Algorithm~\ref{alg:dlm_imple} and prove that it achieves a competitive complexity of $\mathcal{O}(d)$.

This section is organized as follows. First, we formally define this parameterization and emphasize a crucial structural difference: while standard density-ratio models inherently embed a \emph{denoising order} in Step~\ref{step:rtrm_sample}, the time-invariant parameterization instead requires specifying an external \emph{denoising order} during inference. We then demonstrate that adapting \ourmethod\ to this parameterization naturally deduces a principled denoising schedule that aligns with practical implementations. Finally, we establish the $\epsilon$-TV convergence of \ourmethod\ when implemented with a lazy update strategy under this parameterization. 

\paragraph{Time-Invariant Parameterization.}
Recall that in our general formulation, the neural network $\tilde{v}_{T-t,\vy^\prime}(\cdot)$ approximates the density ratio $q_t^\to(\cdot)/{q_t^\to(\vy^\prime)}$ of the forward marginal distributions, constituting a time-variant parameterization.
Recent studies~\citep{ou2024your,zheng2024masked}, however, establish that this ratio admits a factorization into a time-dependent coefficient and a time-invariant term derived from the clean data.
Formally, for any adjacent states $\vy, \vy^\prime$ with
\begin{equation*}
    \mathrm{Ham}(\vy,\vy^\prime)=1,\quad \DfId{\vy}{\vy^\prime} = i, \quad \gK(\vy) = \{j|\vy_j = \idxK\},\quad \text{and}\quad \gK(\vy^\prime) = \gG(\vy) + \{i\},
\end{equation*}
the density ratio satisfies the identity:
\begin{equation}
    \label{eq:time_inv_para}
    \frac{1-e^{-t}}{e^{-t}}\cdot \frac{q^\to_t(\vy)}{q^\to_t(\vy^\prime)} = \frac{\sum_{\hat{\vy}\in \gY} \vone[\hat{\vy}_j = \vy_j, \forall  j \not\in \gK(\vy)]\cdot  p_0(\hat{\vy})}{\sum_{\hat{\vy}\in \gY} \vone[\hat{\vy}_j = \vy^\prime_j, \forall  j \not\in \gK(\vy^\prime)]\cdot  p_0(\hat{\vy})} = \mathrm{Pr}_{\text{data}}\left[\vy_i | \vy^\prime_j, j\in \gK^c(\vy^\prime)\right].
\end{equation}
This relationship allows the network to directly parametrize the conditional distribution of the clean data, $\tilde{q}_{0,i}(\vy_i | \vy^\prime_{\gK^c(\vy^\prime)})\approx \mathrm{Pr}_{\text{data}}[\vy_i | \vy^\prime_j, j\in \gK^c(\vy^\prime)]$, eliminating the dependence on $t$. 
Such a perspective effectively reframes inference as the iterative imputation of clean data, rather than merely the reversal of a time-dependent CTMC.

While this reduction is elegant, Eq.~\eqref{eq:time_inv_para} relies on two critical assumptions:
(1) the data distribution is supported exclusively on non-absorbing states (Assumption~\ref{ass:mask_init}), and (2) the forward infinitesimal operator exhibits \emph{isotropic absorption}, requiring uniform transition rates $R^\to(\vy,\vy^\prime)$ across all adjacent states ($\mathrm{Ham}(\vy,\vy^\prime) = 1$).
In contrast, time-variant parameterization remains robust when these restrictions are relaxed.
For instance, theoretical convergence persists even if the initialization assumption is violated (Corollary~\ref{cor:conv_wo_A2}),  and the general approach naturally accommodates \emph{anisotropic absorption} via the standard calculation of the infinitesimal operator (Eq.~\eqref{eq:rev_func}).


\begin{algorithm}[t]
    \caption{\sc Iterative Imputation (uniformly randomized denoising order)}
    \label{alg:dlm_imple}
    \begin{algorithmic}[1]
            \STATE {\bfseries Input:}  The sequence length $d$, the vocabulary $\gV=\{1,2,\ldots, \idxK\}$ where $\idxK$ is the mask token,  the noise schedule $\alpha_t$ and its inverse function $\alpha^{-1}$, , the pretrained masked diffusion model $p_{\theta}$
            \STATE $\overline{\rvy}_d = [\idxK, \idxK, \ldots, \idxK]$.
            \FOR{$n = d$ {\bfseries downto} $1$}
                \STATE Randomly and uniformly select an index $i$ from $\{j: \overline{\rvy}_{n,j} = \idxK\}$;
                \STATE Draw the target state for denoising by $\overline{\rvy}_{n-1, i}\sim \tilde{q}_{0,i}(\cdot | \overline{\rvy}_{n, \gK^c(\overline{\rvy}_n)})$
                \STATE $\overline{\rvy}_{n-1} = \overline{\rvy}_n[i: \idxK \rightarrow \overline{\rvy}_{n-1,i}]$
            \ENDFOR
            \STATE {\bfseries return} $\overline{\rvy}_{0}$.
    \end{algorithmic}
\end{algorithm}

\paragraph{Denoising Order and Adaptation of \ourmethod.}
In absorbing discrete diffusion with time-variant parameterization, the generative trajectory—i.e., the sequence of denoised elements—is implicitly governed by the sampling mechanism (Step~\ref{step:rtrm_sample} in Alg.~\ref{alg:uni_inf}).
By drawing samples directly from the transition rate matrix, the process simultaneously determines both the index $i$ of the absorbing state to be updated and its corresponding target value $k$.
However, this inherent ordering becomes obscured when inference is reformulated via the time-invariant parameterization $\tilde{q}_{0,i}(\cdot | \vy^\prime_{\gK^c(\vy^\prime)})$. 
In this setting, the selection of the denoising order is ambiguous.
Consider a state $\vy^\prime$ containing multiple absorbing indices, $i_1, i_2 \in \gK(\vy^\prime)$.
While sampling a candidate value for $i_1$ from $\tilde{q}_{0,i_1}(\cdot | \vy^\prime_{\gK^c(\vy^\prime)})$ is straightforward, the parameterization lacks a canonical metric to compare conditional probabilities across different indices.
Specifically, without a unified scale to weigh $\tilde{q}_{0,i_1}(\vy_i | \vy^\prime_{\gK^c(\vy^\prime)})$ against $\tilde{q}_{0,i_2}(\vy_i | \vy^\prime_{\gK^c(\vy^\prime)})$, prioritizing one absorbing state over another is undefined.
This indeterminacy regarding the denoising schedule has been similarly highlighted in recent literature~\citep{ou2024your,kim2025train}.

Crucially, when \ourmethod\ is adapted to the time-invariant parameterization, it naturally resolves this ambiguity by inducing a specific denoising order.
Specifically, the sampling step (Step~\ref{step:rtrm_sample} in Alg.~\ref{alg:uni_inf}) admits a factorization into two stages: (1) selecting the index of the absorbing state to denoise, and (2) sampling the target state for that index.
Under time-invariant conditions, the marginal probability of selecting index $i$ simplifies to:
\begin{equation*}
    \begin{aligned}
        \frac{\sum_{k=1}^{\idxK-1} \hat{R}_{t,\vy_0}(\vy^\prime[i:\idxK \rightarrow k], \vy^\prime)}{\sum_{i\in \gK(\vy^\prime)}\sum_{k=1}^{\idxK-1} \hat{R}_{t,\vy_0}(\vy^\prime[i:\idxK \rightarrow k]} =  \frac{\sum_{k=1}^{\idxK-1} \tilde{q}_{0,i}(\vy^\prime[i:\idxK \rightarrow k]| \vy^\prime_{\gK^c(\vy^\prime)})}{\sum_{i\in \gK(\vy^\prime)}\sum_{k=1}^{\idxK-1} \tilde{q}_{0,i}(\vy^\prime[i:\idxK \rightarrow k]| \vy^\prime_{\gK^c(\vy^\prime)})} = \left|\gK(\vy^\prime)\right|^{-1},
    \end{aligned}
\end{equation*}
which implies that \ourmethod\ assigns a uniform selection probability to every currently absorbing position.
Furthermore, for the second stage, the probability of selecting the target state aligns exactly with the conditional distribution $\tilde{q}_{0,i}(\vy_i | \vy^\prime_{\gK^c(\vy^\prime)})$. 
Consequently, adapting \ourmethod\ to the time-invariant parameterization yields a procedure similar to performing iterative imputation with a \emph{uniformly randomized} denoising order, as outlined in Alg.~\ref{alg:dlm_imple}.

\paragraph{Lazy Update and $\mathcal{O}(d)$ Discrete Score Calls.}

Building on the previous discussion, a comparison of \ourmethod\ under time-variant versus time-invariant parameterizations reveals a critical computational distinction. 
In the time-variant setting, at each potential transition time $\tau_n$, the particle remains in its current state with probability $1- \beta_{t_w}^{-1}\cdot \hat{R}_{\tau_n,\vy_0}(\vy_{n-1})$. 
Consequently, if no transition occurs, the computed discrete scores $\hat{R}_{\tau_n,\vy_0}(\cdot, \vy_{n-1})$ must be discarded. 
This inefficiency arises because the scores are time-dependent, implying that the values at the next potential transition time will differ, i.e., $\hat{R}_{\tau_n,\vy_0}(\cdot, \vy_{n-1})\not=\hat{R}_{\tau_{n+1},\vy_0}(\cdot, \vy_{n-1})$.
Conversely, within the time-invariant parameterization, transition probabilities depend solely on the current state and are independent of time. 
Thus, if a particle is not updated at $\tau_n$, the calculated scores, e.g., $\tilde{q}_{0}(\cdot | \vy^\prime_{\gK^c(\vy^\prime)})$, can be cached and reused for subsequent steps—a strategy known as ``lazy update.''
With such lazy updates, since each absorbing state is denoised exactly once, the total number of score evaluations is strictly bounded by $d$. 
Relative to the complexity bounds in Theorem~\ref{thm:main_mask_unif}, this eliminates an overhead factor of $\mathcal{O}(\ln d)$, underscoring the efficiency gains of the static formulation.
In this regime, the execution of Alg.~\ref{alg:uni_inf} closely mirrors Alg.~\ref{alg:dlm_imple}, with one notable divergence. 
Due to the scaling factor $\beta_{t_w}$ in Alg.~\ref{alg:uni_inf}, transitions remain probabilistic, rendering the number of absorbing states in the final sample a random variable.
In contrast, the imputation-based inference (Alg.~\ref{alg:dlm_imple}) enforces a deterministic reduction in entropy, guaranteeing that one absorbing state is denoised at each update until the generated particle is fully observed.
Therefore, to rigorously establish that the imputation-type inference—viewed as a reduction of \ourmethod\ in the time-invariant setting—achieves Total Variation (TV) convergence, we present the following theorem. The proof is deferred to Appendix~\ref{app_sec: convergence_fhs_reverse}.
\begin{theorem}[Convergence of Alg.~\ref{alg:dlm_imple}]
    \label{thm:convergence_fhs_reverse}
    Suppose Assumption~\ref{ass:score_approximation_error} and~\ref{ass:mask_init} hold, if the discrete scores are parameterized by time-independent neural network as~\cite{ou2024your}, the TV distance between the target discrete distribution $q_*$ and the underlying distribution of the output particle $\overline{q}_{0}$ of Alg.~\ref{alg:dlm_imple} will satisfy $\TVD{q_*}{\hat{q}_{T-\delta}}\le 4\epsilon$.
\end{theorem}
\section{Conclusion}
In this paper, we present a rigorous analysis of absorbing discrete diffusion. 
Motivated by the computational redundancy of uniform diffusion—which repeatedly re-denoises valid elements—we introduce \ourmethod, a sampler designed to update each state exactly once. 
This approach eliminates the $\ln(1/\epsilon)$ factor characteristic of uniform settings, thereby achieving nearly $\epsilon$-free complexity. 
Furthermore, we adapt \ourmethod\ to the time-invariant parameterization prevalent in practice, demonstrating that it naturally aligns with iterative imputation-type inference under a uniformly randomized denoising order. 
In this regime, we establish Total Variation (TV) convergence guarantees for imputation algorithms.
We empirically corroborate these acceleration mechanisms against uniform baselines on synthetic data and validate the effectiveness of \ourmethod\ on real-world text generation benchmarks following~\citet{lou2024discrete}. 
Collectively, our results provide the first rigorous theoretical foundation for absorbing discrete diffusion, elucidating its efficiency and paving the way for advanced sampling techniques.

\bibliographystyle{plainnat}
\bibliography{0_contents/ref}  





\newpage
\appendix
\tableofcontents

\newpage

\section{Experiments}
\label{app_sec:exps}

\subsection{Synthetic Experiments.}
We conduct synthetic experiments to validate our theoretical findings and compare the sampling efficiency of our Masked Discrete Diffusion model against the uniform baseline.

\paragraph{Experiment Setup.} We utilize a state space defined by vocabulary size $K=3$ and sequence length $d=4$. The ground truth distribution, $p^*$, is constructed by assigning a random mass sampled uniformly from $(0,1)$ to each of the $K^d$ possible sequences and normalizing the distribution. We report results averaged over 5 independent random seeds. For each seed, we generate $5000$ trajectories using our method (Algorithm~\ref{alg:uni_inf}, AATU) and the truncated uniformization baseline with a uniform stationary distribution (adapted from \cite{huang2025almost}). Performance is evaluated via the Total Variation (TV) distance between the empirical marginal distribution and $p^*$, plotted as a function of the Number of (Score) Function Evaluations (NFE). Quantitative results are shown in Figure~\ref{fig:synthetic}, and illustrative sampling trajectories are visualized in Figure~\ref{fig:trajectory}.

\begin{figure}[!htbp]
    \centering
    \includegraphics[width=0.63\linewidth, align = t]{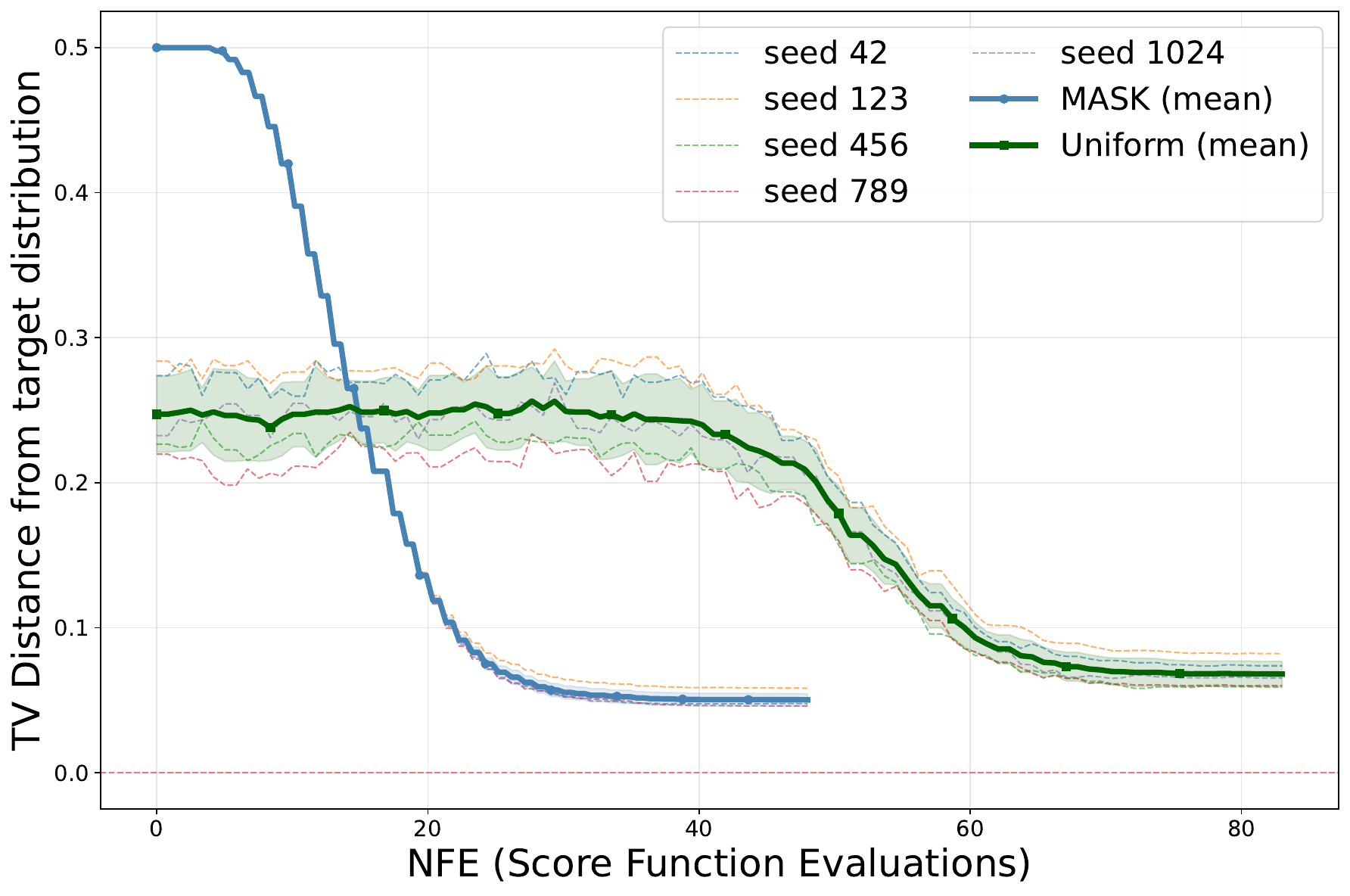}
    \includegraphics[width=0.35\linewidth, align = t]{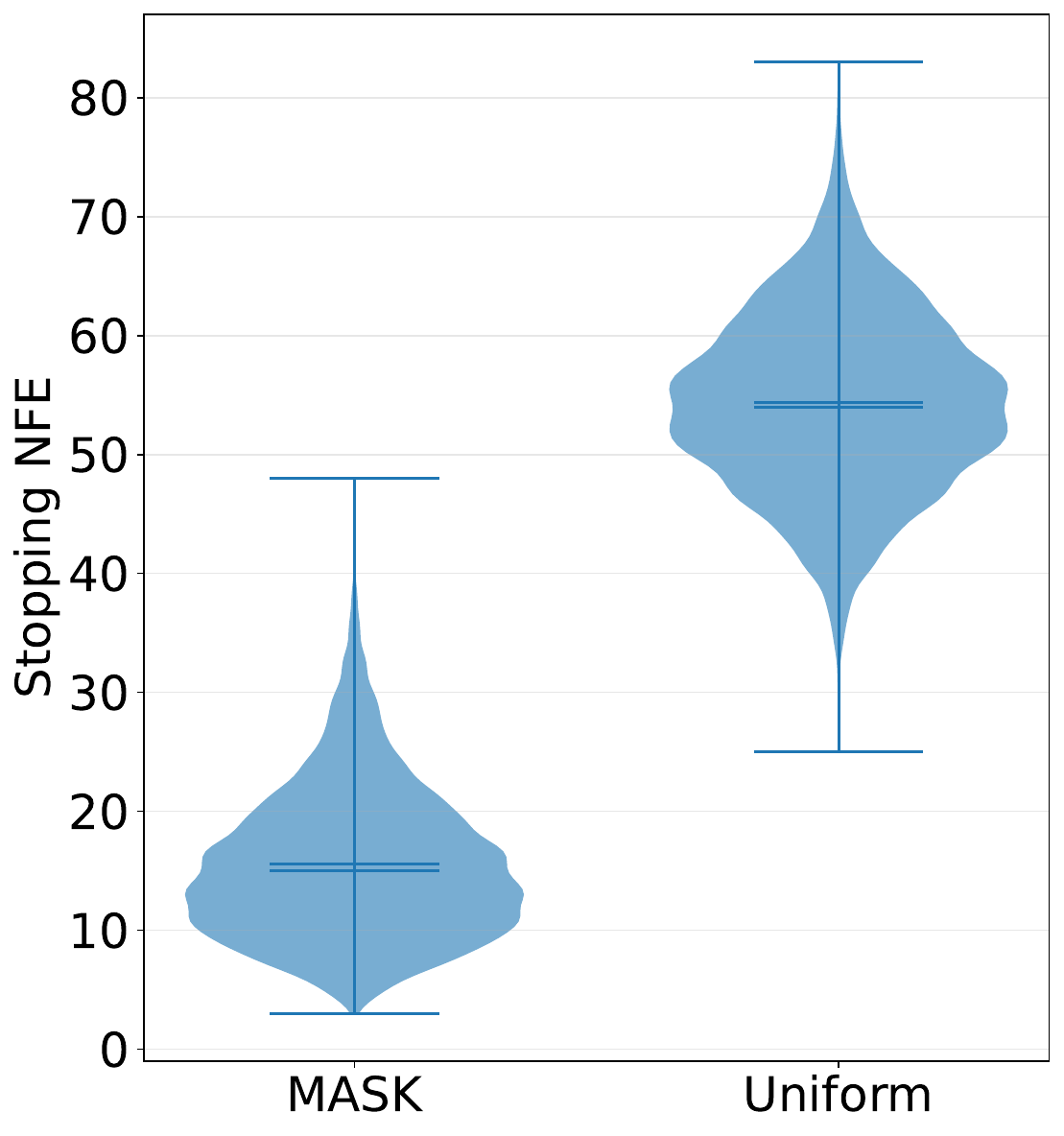}
    \caption{\textbf{Synthetic experiment results on sampling efficiency.} We compare our proposed Masked Discrete Diffusion (MASK) against the Uniform baseline with vocabulary size $K=3$ and sequence length $d=4$. \textbf{Left:} The Total Variation (TV) distance between the empirical and ground truth distributions as a function of the Number of (Score) Function Evaluations (NFE). The solid lines represent the mean over 5 seeds, and shaded regions indicate the standard deviations. Our method achieves faster convergence to the target distribution. \textbf{Right:} Violin plots illustrating the distribution of Stopping NFE. The MASK method requires significantly fewer evaluations to terminate compared to the Uniform baseline.}    
    \label{fig:synthetic}
\end{figure}

\begin{figure}
    \centering
    \includegraphics[width=0.9\linewidth]{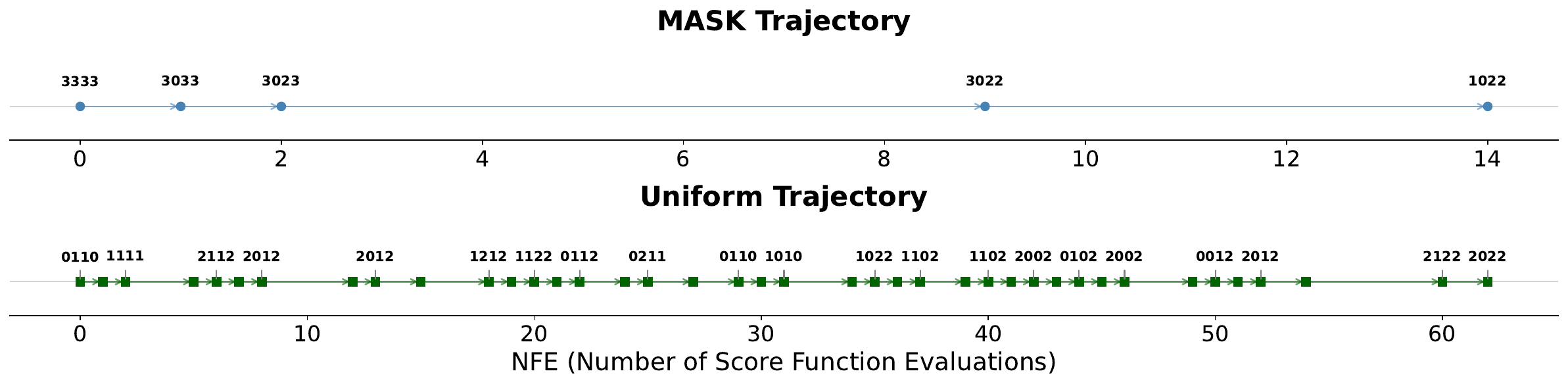}
    \caption{Visualization of individual sampling trajectories.  The plots show single sampling paths, with labels indicating the intermediate discrete states. The MASK method (top) navigates the state space efficiently with few steps. In contrast, the Uniform baseline (bottom) exhibits diffusive behavior with many small steps—often reverting previous changes—resulting in a high NFE cost.}
    \label{fig:trajectory}
\end{figure}

\subsection{Real World Experiments}
We consider to introduce our Alg.~\ref{alg:uni_inf} (AATU) into the text generation task.

\paragraph{Experimental Settings} In this paragraph, we follow the problem setting as SEDD shown in~\citet{lou2024discrete}, and consider the unconditional text generation task with the small pretrained SEDD Absorbing model.
The sequence length of generated sample is constrains as $d=1024$, and the the vocabulary size will be $K=50258$, including the mask token.
We choose the typical Euler and Tweedie’s $\tau$-leaping (analytic samples in~\citet{lou2024discrete}'s implementation) as our baselines.
For the step number choice, we only consider $\{1024, 2048\}$. 
Because AATU does not consider the conditional independent assumption for the reverse process.
Under this condition, it requires at least $d$ steps to generate one no-mask sample.

\paragraph{The inexact adaptation from AATU.} In SEDD experiments, 
The exact implementation of Alg.~\ref{alg:uni_inf} will require the inference complexity to be $K\times d = 50258 \times 1024$, which is far beyond an acceptable inference complexity.
Since the choice of $K$ can be used to control the inference complexity, in the following experiment we will choose
\begin{equation*}
    K = \text{required steps} / \text{generated sequence length},
\end{equation*}
which is an inexact implementation of Alg.~\ref{alg:uni_inf} (AATU), while makes it to be possible to be tuned via the choice of the step number.
Moreover, the implementation of Euler and Tweedie’s $\tau$-leaping is based on log-linear noise schedule, which means the transition rate matrix of the forward process satisfies
\begin{equation*}
    R^\to_t(\vy,\vy^\prime) = \sigma(t)R^\to(\vy,\vy^\prime)\quad \text{where}\quad \sigma(t) = \frac{1-\epsilon}{1-(1-\epsilon)\cdot t}
\end{equation*}
and $R^\to$ follows from Eq.~\ref{eq:fwd_transtion_rate_func}.
Under this condition, the reverse transition rate matrix will become 
\begin{equation*}
    \begin{aligned}
        & R_t^\gets(\vy, \vy^\prime) 
            \;\coloneqq\; 
            \sigma(1-t)\cdot R^\to(\vy^\prime, \vy)\,\frac{q^\gets_t(\vy)}{q^\gets_t(\vy^\prime)}\\
        & = \sigma(1-t)\cdot R^\to(\vy^\prime,\vy)\cdot s_{\theta, 1-t,\vy^\prime}(\vy).
    \end{aligned}
\end{equation*}

\paragraph{Empirical Results.} We use PPL and entropy as two criteria to measure the generation quality for different samplers. The results are summarized in the following Table~\ref{tab:comp_experiment}.

\begin{table*}[!hbtp]
    \centering
    \caption{\small Comparison of the inference generation performance, we calculate the average perplexity and entropy for 32 samples generated by Euler, Analytic and AATU. The experiments show even with an inexact implementation, AATU still outperform then other samplers consistently.} 
    \small
    \renewcommand\arraystretch{1.3}
    \begin{tabular}{cccccc}
    \toprule
     Samplers & Steps & Avg Perplexity & Std Perplexity &   Avg Entropy & Std Entropy  \\
     \midrule
    Euler &  $1024$ & $41.42$ & $11.68$ 
    & $7.588$ & $0.301$  \\
    Analytic &  $1024$ & $41.81$ & $11.57$  & $7.597$ & $0.286$  \\
    AATU & $1024$ & $\mathbf{40.54}$ & $11.20$ & {$7.554$} & $0.230$ \\
     \midrule
     Euler & $2048$ & $33.32$ & $7.141$ & $7.492$ & $0.258$ \\
     Analytic & $2048$ & $32.50$ & $6.952$ & $7.489$ & $0.250$ \\
     AATU & $2048$ & $\mathbf{31.82}$ & $6.717$ & {$7.394$} & $0.332$ \\
     \bottomrule 
    \end{tabular}
    \label{tab:comp_experiment}
    \vspace{-.15in}
\end{table*}

\section{Notation Summary}
\label{sec:app_notations}
We summarize all notations used in the main paper and appendix in Table~\ref{tab:notations}.

\begin{table}[t]
\centering
\caption{Summary of key notations used in the paper.}
\label{tab:notations}
\renewcommand{\arraystretch}{1.11}
\begin{tabular}{@{}lp{0.75\linewidth}@{}}
\toprule
\textbf{Symbol} & \textbf{Description} \\
\midrule
$q_*$ & Discrete distribution on $\gY=\{1, 2,\ldots, K\}^{d}$\\
$\rvy^\to_t$ & Forward-time CTMC on $\gY$ \\
$q^\to_t$ & Marginal distribution of forward process at time $t$, i.e., $\rvy^\to_t\sim q^\to_t$ \\
$q^\to_{t',t}$ & Joint distribution of $(\rvy^\to_{t^\prime},\rvy^\to_t)$ \\
$\tilde{q}_t$ & Aapproximation of $q^\to_t$ constructing the reverse initialization, Eq.~\eqref{def:tilde_q_t}\\
$q^\to_{t'|t}(\vy'|\vy)$ & Conditional transition probability in forward process, Eq.~\eqref{eq:fwd_func_close_form} \\
$\rvy^\gets_t$ & Reverse-time CTMC defined by $q^\gets_t := q^\to_{T-t}$, $\rvy^\gets_t\sim q_t^\gets$ \\
$q^\gets_t$ & Marginal distribution of reverse process at time $t$, $q^\gets_t=q^\to_{T-t}$ \\
$q^\gets_{t',t}$ & Joint distribution of $(\rvy^\gets_{t^\prime},\rvy^\gets_t)$\\
$q^\gets_{t'|t}(\vy'|\vy)$ & Conditional transition probability of the ideal reverse process \\
$\hat{q}_t$ & Marginal distribution of reverse process at time $t$ implemented by Alg.~\ref{alg:uni_inf} \\
$\hat{q}_{t',t}$ & Joint distribution of $(\hat{\rvy}_{t^\prime},\hat{\rvy}_t)$\\
$\hat{q}_{t'|t}(\vy'|\vy)$ & Conditional transition probability of the ideal reverse process \\
\midrule
$R^\to(\vy,\vy')$ & Forward transition rate, i.e., Eq.~\eqref{eq:fwd_transtion_rate_func}, from state $\vy'$ to $\vy$. This follows the ordering of the conditional distribution $p(\vy | \vy')$, which is the \textit{transpose} of the convention used in some other works. \\
$R^\gets_t(\vy,\vy')$ & Reverse transition rate at time $t$ from state $\vy'$ to $\vy$, $R_t^\gets(\vy, \vy^\prime) \coloneqq   R^\to(\vy^\prime, \vy)\cdot\frac{q^\gets_t(\vy)}{q^\gets_t(\vy^\prime)}$, Eq.~\eqref{eq:rev_func}\\
$\tilde{R}_t(\vy,\vy')$ & Estimated reverse transition rate using the learned density ratio, $\tilde{R}_{t}(\vy,\vy^\prime) = R^\to(\vy^\prime, \vy)\cdot \tilde{v}_{t,\vy^\prime}(\vy)$, Eq.~\eqref{eq:score_estimation_main}\\
$\hat{R}_t(\cdot,\cdot)$ & Truncated version of $\tilde{R}_t(\cdot,\cdot)$ 
with threshold $\beta_t$, Eq.~\eqref{def:prac_infi_oper_1}  \\
$R^\gets_t(\vy),\ \tilde{R}_t(\vy),\ \hat{R}_t(\vy)$ & Total reverse transition rate out of state $\vy$ for each rate type, defined as $R(\vy)\coloneqq\sum_{\vy' \ne \vy} R(\vy', \vy)$ with $R \in \{R^\gets_t,\ \tilde{R}_t,\ \hat{R}_t\}$ \\
$\beta_t$ & Upper bound on $R^\gets_t(\vy)$, $ \beta_t=  \mathrm{numK}(\vy)\cdot K / (T-t)$, Eq.~\eqref{def:beta_t_} \\
\midrule
$v_{t,\vy'}(\vy)$ & Density ratio $q^\gets_t(\vy) / q^\gets_t(\vy')$ \\
$\tilde{v}_{t,\vy'}(\vy)$ & Learned approximation to $v_{t,\vy'}(\vy) = q^\gets_t(\vy)/q^\gets_t(\vy')$ \\
$\mathrm{numK}(\cdot)$ & The number of [MASK] token (or token $K$) in a vector.\\
$L_{\text{SE}}(\hat{v})$ & Score entropy loss used to train $\tilde{v}$, Eq.~\eqref{eq:score_estimation_main} \\
\midrule
$\ve_i$ & One-hot vector with a $1$ at position $i$ and $0$ elsewhere \\
$\delta_{\vy}(\cdot)$ & Indicator function with $\delta_{\vy}(\vy)=1$ and $\delta_{\vy}(\vy^\prime)=0$ ($\vy^\prime\not= \vy$)\\
\bottomrule
\end{tabular}
\end{table}

\section{The Markov Processes of Discrete Diffusion Models}
\label{sec:a1_fwd_and_infi}

\subsection{The Formulations of the Forward Process}

\paragraph{Semigroup Formulation.} In general, the time-homogeneous CTMC 
can be described by a Markov semigroup $\gQ^\to_t$ defined as:
\begin{equation}
    \label{eq:app_semigroup_ope}
    \gQ^\to_{t}[f](\vy) = \E\left[f(\rvy_t)|\rvy_0 =\vy\right] = \left<f, q^\to_{t|0}(\cdot|\vy)\right>_{\gY}
\end{equation}
where the function $f\colon \gY\rightarrow \R$.
Due to the definition, the infinitesimal operator $\gL^\to$ of the time homogeneous $\gQ^\to_t$ is denoted as
\begin{equation}
    \label{def:infinitesimal_oper}
    \gL^\to[f](\vy) = \lim_{t\rightarrow 0} \left[\frac{\gQ^\to_t[f] - f}{t}\right](\vy) = \left<f, \partial_t q^\to_{t|0}(\cdot|\vy)\Big|_{t=0}\right>_{\gY}\coloneqq \left<f, R^\to(\cdot,\vy)\right>_{\gY}
\end{equation}
where 
\begin{equation}
    \label{def:app_mean_of_R}
    R^\to(\vy^\prime, \vy) \coloneqq  \partial_t q^\to_{t|0}(\vy^\prime|\vy)\Big|_{t=0} = \lim_{t\rightarrow 0}\left[\frac{q^\to_{t|0}(\vy^\prime|\vy)-\delta_{\vy}(\vy^\prime)}{t}\right].
\end{equation}
According to the time-homogeneous property, we have
\begin{equation*}
    q^\to_{t+\Delta t | t}(\vy^\prime|\vy) = \delta_{\vy}(\vy^\prime) + \Delta t\cdot R^\to(\vy^\prime, \vy)  + o(\Delta t)
\end{equation*}
for any $t$.
Here, the transition rate function $R^\to$ must satisfy
\begin{equation}
    \label{eq:transition_rate_prop}
    R^\to(\vy,\vy^\prime)\ge 0 \  \text{when}\  \vy^\prime\not=\vy\quad \text{and}\quad R^\to(\vy^\prime,\vy^\prime) = -\sum_{\vy\not=\vy^\prime}R^\to(\vy,\vy^\prime)\le 0
\end{equation}
due to the definition Eq.~\eqref{def:app_mean_of_R}.
Under this setting, we can provide the dynamic of $q_{t|0}$ for any $t$.
Specifically, we have
\begin{equation*}
    \begin{aligned}
        &\partial_t \gQ^\to_t[f](\vy) = \gQ^\to_t\left[\gL f\right](\vy) = \left<\gL^\to f, q^\to_{t|0}(\cdot|\vy)\right>_{\gY} = \sum_{\vy^\prime \in \gY} \gL^\to[f](\vy^\prime)\cdot q^\to_{t|0}(\vy^\prime|\vy)\\
        & = \sum_{y^\prime\in \gY}\left[  \sum_{\tilde{\vy}\in\gY} f(\tilde{\vy})\cdot R^\to(\tilde{\vy}, \vy^\prime) \cdot q_{t|0}(\vy^\prime|\vy)\right] = \sum_{\tilde{\vy}\in\gY}\left[f(\tilde{\vy})\cdot \sum_{\vy^\prime\in \gY}R^\to(\tilde{\vy},\vy^\prime)\cdot q_{t|0}(\vy^\prime|\vy)\right],
    \end{aligned}
\end{equation*}
where the first inequality follows from the semigroup property.
Combined with the fact
\begin{equation*}
    \partial_t\gQ^\to_t[f](\vy) = \left<f, \partial_t q^\to_{t|0}(\cdot|\vy)\right>_{\gY}
\end{equation*}
derived from Eq.~\eqref{eq:app_semigroup_ope},
we have
\begin{equation*}
    \partial_t q^\to_{t|0}(\tilde{\vy}|\vy) = \sum_{\vy^\prime\in \gY}R(\tilde{\vy},\vy^\prime)\cdot q^\to_{t|0}(\vy^\prime|\vy) = \left<R(\tilde{\vy},\cdot), q^\to_{t|0}(\cdot|\vy)\right>_{\gY}.
\end{equation*}
According to the time-homogeneous property, the above equation can be easily extended to 
\begin{equation}
    \label{eq:condi_fwd_dis}
    \partial_t q^\to_{t|s}(\tilde{\vy}|\vy) = \sum_{\vy^\prime\in \gY}R(\tilde{\vy},\vy^\prime)\cdot q^\to_{t|s}(\vy^\prime|\vy) = \left<R(\tilde{\vy},\cdot), q^\to_{t|s}(\cdot|\vy)\right>_{\gY}.
\end{equation}
Combining with Bayes' Theorem, the transition of the marginal distribution is
\begin{equation}
    \label{eq:app_fwd_func}
    \frac{\der q^\to_t}{\der t}(\vy) = \left<R(\vy,\cdot), q^\to_t\right>_{\gY}.
\end{equation}

\paragraph{Matrix Formulation.} Suppose the support set $\gY$ of $q^\to_t$ be written as $\gY =\{\vy_1, \vy_2,\ldots, \vy_{|\gY|}\}$, we may consider the marginal distribution $q^\to_s$ to be a vector, i.e.,
\begin{equation*}
    \vq^\to_t = \left[q_t(\vy_1), q_t(\vy_2),\ldots, q_t(\vy_{|\gY|})\right],
\end{equation*}
conditional transition probability function $q^\to_{t|s}$ to be a matrix, i.e.,
\begin{equation*}
    \mQ^\to_{t|s} = \left[
        \begin{matrix}
            q^\to_{t|s}(\vy_1|\vy_1) & q^\to_{t|s}(\vy_1|\vy_2) & \ldots & q^\to_{t|s}(\vy_1|\vy_{|\gY|})\\
            q^\to_{t|s}(\vy_2|\vy_1) & q^\to_{t|s}(\vy_2|\vy_2) & \ldots & q^\to_{t|s}(\vy_2|\vy_{|\gY|})\\
            \ldots & \ldots & \ldots & \ldots\\
            q^\to_{t|s}(\vy_{|\gY|}|\vy_1) & q^\to_{t|s}(\vy_{|\gY|}|\vy_2) & \ldots & q^\to_{t|s}(\vy_{|\gY|}|\vy_{|\gY|})\\
        \end{matrix}
    \right].
\end{equation*}
Similarly, the function $R$ can also be presented as
\begin{equation}
    \label{def:trans_ker_to_matrix}
    \mR^\to = \left[
        \begin{matrix}
            R^\to(\vy_1,\vy_1) & R^\to(\vy_1, \vy_2) & \ldots & R^\to(\vy_1, \vy_{|\gY|})\\
            R^\to(\vy_2,\vy_1) & R^\to(\vy_2, \vy_2) & \ldots & R^\to(\vy_2, \vy_{|\gY|})\\
            \ldots & \ldots & \ldots & \ldots\\
            R^\to(\vy_{|\gY|},\vy_1) & R^\to(\vy_{|\gY|}, \vy_2) & \ldots & R^\to(\vy_{|\gY|}, \vy_{|\gY|})\\
        \end{matrix}
    \right].
\end{equation}
Under this condition, Eq.~\eqref{eq:app_fwd_func} can be written as 
\begin{equation}
    \label{eq:fwd_vec}
    \der \vq^\to_t/\der t = \mR^\to\cdot \vq^\to_t
\end{equation}
matching the usual presentation shown in~\cite{chen2024convergence,zhang2024convergence}. 

\subsection{The Proof of Lemma~\ref{lem:absorbing_reverse_de}}
\label{app_sec:prof_absorbing_reverse}

\begin{proof}
For any $t\in[0,T]$, the marginal, joint, and conditional distribution w.r.t. $\{\rvy_t^\gets\}$ are denoted as
\begin{equation*}
    \rvy^\gets_t\sim q^\gets_t,\quad (\rvy^\gets_t, \rvy^\gets_{t^\prime}) \sim q^\gets_{t,t^\prime}, \quad \text{and}\quad q^\gets_{t^\prime|t} = q^\gets_{t^\prime, t}/q^\gets_t,
\end{equation*}
which have $q^\gets_t = q^\to_{T-t}$. 
Considering the reverse CTMC, we have
\begin{equation*}
    \begin{aligned}
        & \frac{\der q^\gets_t}{\der t}(\vy) = \lim_{\Delta t\rightarrow 0} \frac{q^\gets_{t+\Delta t}(\vy) - q^\gets_t(\vy)}{\Delta t} = \lim_{\Delta t\rightarrow 0} \frac{\sum_{\vy^\prime} q^\gets_{t+\Delta t|t}(\vy|\vy^\prime)\cdot q^\gets_t(\vy^\prime) - q_t^\gets(\vy)}{\Delta t}\\
        & = \lim_{\Delta t\rightarrow 0} \frac{\sum_{\vy^\prime}\left(q^\gets_{t+\Delta t|t}(\vy|\vy^\prime) - \delta_{\vy^\prime}(\vy)\right)\cdot q^\gets_{t}(\vy^\prime)}{\Delta t}.
    \end{aligned}
\end{equation*}
Therefore, due to the following definition
\begin{equation*}
    R_t^\gets(\vy, \vy^\prime)\coloneqq \lim_{\Delta t\rightarrow 0} \left[\frac{q^\gets_{t+\Delta t| t}(\vy|\vy^\prime)-\delta_{\vy^\prime}(\vy)}{\Delta t}\right],
\end{equation*}
we have 
\begin{equation*}
    \frac{\mathrm{d}\,q^\gets_t}{\mathrm{d}\,t}(\vy) 
    \;=\; 
    \langle R^\gets_t(\vy,\cdot),\,q^\gets_t(\cdot)\rangle_{\gY}.
\end{equation*}

Then, we start to check the property of $\mR_t^\gets(\vy, \vy^\prime)$ when $\vy\not=\vy^\prime$ in the following, which has
\begin{equation*}
    R_t^\gets(\vy, \vy^\prime) = \lim_{\Delta t\rightarrow 0} \frac{q^\gets_{t+\Delta t|t}(\vy|\vy^\prime)}{\Delta t} = \lim_{s\rightarrow t}\partial_t q^\gets_{t|s}(\vy|\vy^\prime).
\end{equation*}
Specifically, with the Bayes Theorem, it has
\begin{equation}
    \label{ineq:discrete_rev_0}
    \begin{aligned}
        & \lim_{s\rightarrow t}\partial_t q_{t|s}^\gets(\vy|\vy^\prime) = -1\cdot \lim_{s\rightarrow t}\partial_{T-t} q^\to_{T-t|T-s}(\vy|\vy^\prime)\\
        & = -1\cdot \lim_{s\rightarrow t}\partial_{T-t}\left[\frac{q^\to_{T-s|T-t}(\vy^\prime|\vy)\cdot q^\to_{T-t}(\vy)}{q^\to_{T-s}(\vy^\prime)}\right]\\
        & = - \underbrace{\lim_{s\rightarrow t}\partial_{T-t} q^\to_{T-s|T-t}(\vy^\prime|\vy)\cdot \frac{q^\to_{T-t}(\vy)}{q^\to_{T-s}(\vy^\prime)}}_{\text{Term 1}} - \underbrace{\lim_{s\rightarrow t}\frac{q^\to_{T-s|T-t}(\vy^\prime|\vy)}{q^\to_{T-s}(\vy^\prime)}\cdot \partial_{T-t} q^\to_{T-t}(\vy)}_{\text{Term 2}}.
    \end{aligned}
\end{equation}
For $\text{Term 1}$ of Eq.~\eqref{ineq:discrete_rev_0}, we have
\begin{equation*}
    \begin{aligned}
        \text{Term 1} = - R^\to(\vy^\prime,\vy)\cdot \frac{q^\to_{T-t}(\vy)}{q^\to_{T-t}(\vy^\prime)},
    \end{aligned}
\end{equation*}
which follows from the Kolmogorov backward theorem (Lemma~\ref{lem:bkw_kolmo}) and Eq.~\eqref{def:infinitesimal_oper}:
\begin{equation*}
    \partial_{T-t} q^\to_{T-s|T-t}(\vy^\prime|\vy) = -\gL^\to[q^\to_{T-s|T-t}(\vy^\prime| \cdot)](\vy) = - \left<q^\to_{T-s|T-t}(\vy^\prime|\cdot), R^\to(\cdot,\vy)\right>_{\gY} = R^\to(\vy^\prime,\vy).
\end{equation*}
Here, the last equation follows from the fact
\begin{equation*}
    \lim_{s\rightarrow t} q^\to_{T-s|T-t}(\vy^\prime|\tilde{\vy}) = 1 \quad \text{only when}\quad \vy^\prime = \tilde{\vy}.
\end{equation*}
Additionally, with the fact $\vy^\prime\not=\vy$, it has
\begin{equation*}
    \lim_{s\rightarrow t} q^\to_{T-s|T-t}(\vy|\vy^\prime) =0.
\end{equation*}
That means Term 2 of Eq.~\eqref{ineq:discrete_rev_0} will definitely be $0$. Hence, the proof is completed.
\end{proof}

\subsection{The Proof of Lemma~\ref{lem:fwd_convergence_0}}
\label{app_sec:prof_lem2}

\begin{lemma}
    \label{lem:close_solution_qt}
    The close solution of Eq.~\eqref{eq:fwd_vec} is
    \begin{equation*}
        \vq^\to_t = \exp(t\mR^\to)\cdot \vq_0^\to\quad \text{where}\quad \exp(t\mR^\to) = \sum_{i=0}^\infty \frac{1}{i!}(t\mR^\to)^i = \mI + t\mR^\to + \frac{(t\mR^\to)^2}{2}+\ldots.
    \end{equation*}
\end{lemma}
\begin{proof}
    We can easily verify that 
    \begin{equation*}
        \begin{aligned}
            \frac{\der \vq_t^\to}{\der t} = \frac{\der}{\der t}\left[\exp(t\mR^\to) \vq_0^\to\right] = \frac{\der }{\der t}\left[\exp(t\mR^\to)\right] \vq_0^\to.
        \end{aligned}
    \end{equation*}
    With the following equation,
    \begin{equation*}
        \begin{aligned}
            \frac{\der}{\der t}\left[\exp(t\mR^\to)\right] = \frac{\der }{\der t}\left[\sum_{i=0}^\infty \frac{(t\mR^\to)^i}{i!}\right] = \sum_{i=1}^\infty \frac{t^{i-1}}{(i-1)!}\cdot (\mR^\to)^{i} = \mR^\to \cdot \sum_{j=0}^\infty \frac{(t\mR^\to)^j}{j!} = \mR^\to\cdot \exp(t\mR^\to),
        \end{aligned}
    \end{equation*}
    we have
    \begin{equation*}
        \frac{\der \vq_t^\to}{\der t} = \mR^\to \cdot \exp(t\mR^\to) \cdot \vq_0^\to = \mR^\to \cdot \vq_t^\to.
    \end{equation*}
    Hence, the proof is completed.
\end{proof}

\begin{lemma}    \label{lem:forward_transition_rate_matrix_decomposition}
    Suppose the transition rate matrix $\mR^\to$ shown as Eq.~\eqref{def:trans_ker_to_matrix} satisfies Eq.~\eqref{eq:fwd_transtion_rate_func}. It can be decomposed as
    \begin{equation*}
        \mR^\to = \sum_{i=1}^{d} \mR_i^\to\quad \text{where}\quad \mR^\to_i =  \underbrace{\mI \otimes\, \cdots \,}_{i-1 \text{ terms}}\otimes\, \mA \,\otimes\, \cdots \,\otimes\, \mI,
    \end{equation*}
    where $\otimes$ denotes the Kronecker product, $\mI$ denotes the identity matrix on $\R^{K\times K}$, and $\mA$ satisfies
    \begin{equation}
        \label{eq:matrxA_def_decomp}
        \mA = \left[
            \begin{matrix}
                -1 & 0 & \ldots & 0\\
                0 & -1 & \ldots & 0 \\
                \vdots & \vdots & \ddots & \vdots\\
                1 & 1 & \ldots & 0
            \end{matrix}
        \right].
    \end{equation}
\end{lemma}
\begin{proof}
    According to the calculation of the Kronecker product, we have
    \begin{equation*}
        \mR_i^\to(\vy,\vy^\prime) = \mI(\vy_1, \vy^\prime_1)\cdot \ldots \cdot \mA(\vy_i, \vy^\prime_i)\cdot \ldots\cdot \mI(\vy_{d}, \vy^\prime_{d}).
    \end{equation*}
    Under this condition, suppose $\mathrm{Ham}(\vy, \vy^\prime)\ge 2$ and $\DfId{\vy}{\vy^\prime} = \{j_1, j_2,\ldots\}$ without loss of generality, for any $j\not\in\{j_1, j_2\}$, we have
    \begin{equation*}
        \mR^\to_j(\vy, \vy^\prime) = \mA(\vy_j, \vy^\prime_j) \cdot \mI(\vy_1, \vy^\prime_1)\cdot \ldots \cdot \underbrace{\mI(\vy_{j_1}, \vy^\prime_{j_1})}_{=0}\cdot \ldots \cdot \underbrace{\mI(\vy_{j_2}, \vy^\prime_{j_2}) }_{=0}\cdot \ldots \cdot \mI(\vy_{d}, \vy^\prime_{d}) = 0.
    \end{equation*}
    Besides, for $j=j_1$, we have
    \begin{equation*}
        \mR^\to_{j_1}(\vy, \vy^\prime) = \mA(\vy_{j_1}, \vy^\prime_{j_1}) \cdot \mI(\vy_1, \vy^\prime_1)\cdot \ldots \cdot \underbrace{\mI(\vy_{j_2}, \vy^\prime_{j_2}) }_{=0}\cdot \ldots \cdot \mI(\vy_{d}, \vy^\prime_{d}) = 0.
    \end{equation*}
    A similar result will be satisfied for $j=j_2$. Hence, it has
    \begin{equation*}
        \mR^\to(\vy,\vy^\prime) = \sum_{i=1}^{d} \mR_i^\to(\vy,\vy^\prime) = 0\quad \text{when}\quad \mathrm{Ham}(\vy, \vy^\prime)\ge 2
    \end{equation*}

    Then, suppose $\mathrm{Ham}(\vy, \vy^\prime)=1$ and $\DfId{\vy}{\vy^\prime} = j_1$, for any $j\not=j_1$, we have
    \begin{equation*}
        \mR^\to_{j}(\vy, \vy^\prime) = \mA(\vy_{j}, \vy^\prime_{j}) \cdot \mI(\vy_0, \vy^\prime_0)\cdot \ldots \cdot \underbrace{\mI(\vy_{j_1}, \vy^\prime_{j_1}) }_{=0}\cdot \ldots \cdot \mI(\vy_{d}, \vy^\prime_{d}) = 0.
    \end{equation*}
    Otherwise, when $j=j_1$, we have
    \begin{equation*}
        \mR^\to_{j_1}(\vy, \vy^\prime) = \mA(\vy_{j_1}, \vy^\prime_{j_1}) \cdot \mI(\vy_1, \vy^\prime_1)\cdot \ldots \cdot \mI(\vy_{d}, \vy^\prime_{d}) = \mA(\vy_{j_1}, \vy^\prime_{j_1})
    \end{equation*}
    where the second equation establishes since $\mathrm{Ham}(\vy, \vy^\prime) = 1$ and $\vy_j = \vy^\prime_j$ when $j\not=j_1$.
    Then, only when $\vy_{j_1} = K$, we will have $\mA(\vy_{j_1}, \vy^\prime_{j_1}) = 1$ otherwise $\mA(\vy_{j_1}, \vy^\prime_{j_1}) = 0$ due to the definition Eq.~\eqref{eq:matrxA_def_decomp}. That means
    \begin{equation*}
        \begin{aligned}
            & \mR^\to(\vy,\vy^\prime) = \sum_{i=1}^{d} \mR_i^\to(\vy,\vy^\prime) = 0\quad \text{when}\quad \mathrm{Ham}(\vy, \vy^\prime)= 1 \ \text{and}\ \vy_{\DfId{\vy}{\vy^\prime}} \not= K\\
            & \mR^\to(\vy,\vy^\prime) = \sum_{i=1}^{d} \mR_i^\to(\vy,\vy^\prime) = 1\quad \text{when}\quad \mathrm{Ham}(\vy, \vy^\prime)= 1 \ \text{and}\ \vy_{\DfId{\vy}{\vy^\prime}} = K.
        \end{aligned}
    \end{equation*}
    Then, suppose $\mathrm{Ham}(\vy, \vy^\prime) = 0$, i.e., $\vy = \vy^\prime$, for any $j\in\{1,2,\ldots, d\}$, we have
    \begin{equation*}
        \mR^\to_{j}(\vy, \vy^\prime) = \mA(\vy_{j}, \vy^\prime_{j}) \cdot \mI(\vy_1, \vy^\prime_1)\cdot \ldots \cdot \mI(\vy_{d}, \vy^\prime_{d}) = \mA(\vy_{j}, \vy^\prime_{j}),
    \end{equation*}
    and 
    \begin{equation*}
        \sum_{i=1}^{d} \mR^\to_{i}(\vy, \vy^\prime) = \sum_{j=1}^{d} \mA(\vy_j, \vy_j) = -\sum_{i=1}^{d} (1-\delta_{K}(\vy_i)),
    \end{equation*}
    which implies we have $\mR^\to(\vy,\vy^\prime) = \sum_{i=0}^{d-1} \mR_i^\to(\vy,\vy^\prime)$ when $\vy=\vy^\prime$. Hence, the proof is completed.
\end{proof}


\begin{lemma}
    \label{lem:exp_R_decomposition}
    With the decomposition shown in Lemma~\ref{lem:forward_transition_rate_matrix_decomposition}, i.e., 
    \begin{equation*}
        \mR^\to = \sum_{i=1}^{d} \mR_i^\to\quad \text{where}\quad \mR^\to_i =  \underbrace{\mI \otimes\, \ldots \, \otimes \mI}_{i-1 \text{ terms}}\otimes\, \mA \,\otimes\, \underbrace{\mI \otimes \ldots \,\otimes\, \mI}_{d-i\ \text{terms}},
    \end{equation*}
    for any $i,j\in\{1,2,\ldots, d\}$, the matrices $\mR_i^\to$ and $\mR_j^\to$ satisfy
    \begin{equation*}
        \mR_i^\to \cdot \mR_j^\to = \mR_j^\to \cdot \mR_i^\to,
    \end{equation*}
    which implies 
    \begin{equation*}
        \exp(t\mR^\to) = \exp\left(t\sum_{i=1}^d \mR^\to_i\right) = \prod_{i=1}^d\exp\left(t\mR_i^\to\right) = \exp(t\mA)^{\otimes d}
    \end{equation*}
\end{lemma}
\begin{proof}
    According to Lemma~\ref{lem:forward_transition_rate_matrix_decomposition}, the matrix $\mR^\to$ has the following decomposition, i.e.,
    \begin{equation*}
        \mR^\to = \sum_{i=1}^{d} \mR_i^\to\quad \text{where}\quad \mR^\to_i =  \underbrace{\mI \otimes\, \ldots \, \otimes \mI}_{i-1 \text{ terms}}\otimes\, \mA \,\otimes\, \underbrace{\mI \otimes \ldots \,\otimes\, \mI}_{d-i\ \text{terms}},
    \end{equation*}
where $\otimes$ denotes the Kronecker product, $\mI$ denotes the identity matrix on $\R^{K\times K}$, and $\mA$ satisfies
\begin{equation*}
    \mA = \left[
        \begin{matrix}
            -1 & 0 & \ldots & 0\\
            0 & -1 & \ldots & 0 \\
            \vdots & \vdots & \ddots & \vdots\\
            1 & 1 & \ldots & 0
        \end{matrix}
    \right].
\end{equation*}
We can easily verify that the matrix $\mA$ can be decomposed as
\begin{equation}
    \label{eq:matrxA_def_decomp_refine}
    \left[
        \begin{matrix}
            -\mI_{K-1} & \vzero \\
            \vone_{1\times (K-1)} & 0
        \end{matrix}
    \right] = \underbrace{\left[
        \begin{matrix}
            \mI_{K-1} & \vzero \\
            -\vone_{1\times (K-1)} & 1
        \end{matrix}
    \right]}_{\mU}\cdot \underbrace{\left[
        \begin{matrix}
            -\mI_{K-1} & 0\\
            \vzero & 0
        \end{matrix}
    \right]}_{\Lambda} \cdot \underbrace{\left[
        \begin{matrix}
            \mI_{K-1} & \vzero \\
            \vone_{1\times (K-1)} & 1
        \end{matrix}
    \right]}_{\mU^{-1}}\quad \text{where}\quad \mU\mU^{-1} = \mU^{-1}\mU = \mI_K.
\end{equation}
Under this condition, $\mR^\to_i$ can be reformulated as
\begin{equation*}
    \begin{aligned}
        \mR_i^\to = & \underbrace{(\mU  \mU^{-1}) \otimes \ldots \otimes (\mU\mU^{-1})}_{i-1\ \text{terms}}\otimes (\mU\Lambda \mU^{-1})\otimes (\mU\mU^{-1})\otimes \ldots (\mU\mU^{-1})\\
        = & \left(\mU\otimes \ldots \otimes \mU\right)\cdot \left(\underbrace{\mI\otimes \ldots \otimes \mI}_{i-1\ \text{terms}} \otimes \Lambda \otimes \mI \ldots \otimes \mI\right) \cdot \left(\mU^{-1}\otimes \ldots \otimes \mU^{-1}\right)\coloneqq \mU^{\otimes d}\cdot \Lambda_i \cdot (\mU^{-1})^{\otimes d}
    \end{aligned}
\end{equation*}
where the last inequality follows from Lemma~\ref{lem:mix_prod_K_prod}.
Under this condition, it has
\begin{equation*}
    \begin{aligned}
        \mR^\to_i \cdot \mR^\to_j & = \mU^{\otimes d}\cdot \Lambda_i \cdot (\mU^{-1})^{\otimes d} \cdot \mU^{\otimes d}\cdot \Lambda_j \cdot (\mU^{-1})^{\otimes d} =  \mU^{\otimes d}\cdot \Lambda_i \cdot \Lambda_j \cdot (\mU^{-1})^{\otimes d}\\
        & = \mU^{\otimes d}\cdot \Lambda_j \cdot \Lambda_i \cdot (\mU^{-1})^{\otimes d} = \mU^{\otimes d}\cdot \Lambda_i \cdot (\mU^{-1})^{\otimes d} \cdot \mU^{\otimes d}\cdot \Lambda_j \cdot (\mU^{-1})^{\otimes d}  = \mR^\to_j \cdot \mR^\to_i,
    \end{aligned}
\end{equation*}
where the second and forth equations follows from Lemma~\ref{lem:mix_prod_K_prod} and Eq.~\eqref{eq:matrxA_def_decomp_refine}.

For the property about the matrix exponential, we start from investigating the case of two commuting matrices, i.e., $\mR_1^\to$ and $\mR^\to_2$.
By definition, we have
\begin{equation*}
    \exp(\mR^\to_1 + \mR^\to_2) = \sum_{i=0}^{\infty} \frac{1}{i!}\left(\mR^\to_1 + \mR^\to_2\right)^i = \sum_{i=0}^\infty\frac{1}{i!}\sum_{j=0}^i C_{i}^j \cdot (\mR^\to_1)^j\cdot (\mR_2^\to)^{i-j}
\end{equation*}
where the last equation establishes since $\mR_1^\to$ and $\mR^\to_2$ are commute. 
Then, we have 
\begin{equation*}
    \begin{aligned}
        & \sum_{i=0}^\infty\frac{1}{i!}\sum_{j=0}^i C_{i}^j \cdot (\mR^\to_1)^j\cdot (\mR_2^\to)^{i-j} = \sum_{i=0}^\infty \sum_{j=0}^i \frac{1}{i!}\cdot \frac{i!}{j!(i-j)!} \cdot (\mR^\to_1)^j\cdot (\mR_2^\to)^{i-j}\\
        & = \sum_{i=0}^\infty \sum_{j=0}^i \frac{1}{j!(i-j)!}\cdot (\mR^\to_1)^j\cdot (\mR_2^\to)^{i-j} = \left(\sum_{j=0}^\infty \frac{(\mR^\to_1)^j}{j!} \right)\cdot \left(\sum_{i=0}^\infty \frac{(\mR^\to_2)^i}{i!} \right) = \exp(\mR_1^\to)\cdot \exp(\mR_2^\to).
    \end{aligned}
\end{equation*}
According to the definition of the matrix exponential, we will have $\exp(\mA\otimes \mB)=\exp(\mA)\otimes \exp(\mB)$ when one of the factors is the identity.
When we multiply all these exponentials, it has
\begin{equation*}
    \begin{aligned}
        &\exp(\mR_1^\to) \cdot  \exp(\mR_2^\to) = \left[\exp(\mA)\otimes \mI\otimes\ldots \otimes \mI\right]\cdot \left[\mI\otimes \exp(\mA)\otimes \ldots \otimes \mI\right] \\
        & = \left[\exp(\mA)\cdot \mI\right]\otimes \left[\mI\cdot \exp(\mA)\right]\otimes \mI \ldots \otimes \mI.
    \end{aligned}
\end{equation*}

Then, following a recursive manner, we have
\begin{equation*}
    \exp\left(t\sum_{i=1}^d \mR^\to_i\right) = \prod_{i=1}^d\exp\left(t\mR_i^\to\right) = \exp(t\mA)^{\otimes d},
\end{equation*}
hence the proof is completed.
\end{proof}

\begin{lemma}
    \label{lem:solution_exp_tA}
    Suppose matrix $\mA$ is 
    \begin{equation*}
        \mA = \left[
            \begin{matrix}
                -1 & 0 & \ldots & 0\\
                0 & -1 & \ldots & 0 \\
                \vdots & \vdots & \ddots & \vdots\\
                1 & 1 & \ldots & 0
            \end{matrix}
        \right],
    \end{equation*}
    the matrix exponential $\exp(t\mA)$ becomes
    \begin{equation*}
        \exp(t\mA) = \left[
            \begin{matrix}
                e^{-t} & 0 & \ldots & 0 & 0\\
                0 & e^{-t} & \ldots & 0 & 0\\
                \vdots & \vdots & \ddots & \vdots & \vdots\\
                1-e^{-t} & 1-e^{-t} & \ldots & 1-e^{-t} & 1
            \end{matrix}
        \right].
    \end{equation*}
\end{lemma}
\begin{proof}
    According to Lemma~\ref{lem:close_solution_qt}, $\bar{\mA}(t)\coloneqq \exp(t\mA)$ can be considered as the close solution of the following matrix ODE, i.e.,
    \begin{equation}
        \label{ineq:ode_for_A_bar}
        \frac{\der \bar{\mA}(t)}{\der t} = \mA \cdot \bar{\mA}(t),\quad \text{where}\quad \bar{\mA}(0) = \mI.
    \end{equation}
    To provide a close form of $\bar{\mA}_t$, we first decompose the matrix $\mA$ as follows
    \begin{equation*}
        \mA = \left[
            \begin{matrix}
                \mB & \vzero\\
                \mC & \vzero
            \end{matrix}
        \right]\quad \text{where}\quad \mB\coloneqq -\mI_{K-1}\in \R^{(K-1)\times(K-1)}\ \text{and}\ \mC \coloneqq [1,1,\ldots, 1]\in \R^{1\times(K-1)}.
    \end{equation*}
    Then, the ODE.~\eqref{ineq:ode_for_A_bar} can be equivalently think column-by-column, the $j$--th column of $\bar{\mA}(t)$ solves
    \begin{equation*}
        \frac{\der }{\der t}\bar{\va}(t) = \mA \bar{\va}(t)\quad \text{where}\quad \va(0) = \ve_j.
    \end{equation*}
    We use the block structure to split $\bar{\va}(t)\in\R^K$ into two parts, i.e., $\bar{\va}(t) = [\bar{\va}_1(t), \bar{\va}_K(t)]$ where $\rvq_1(t)\in\R^{K-1}$ and $\va_K(t)\in \R$ denotes the last coordinate.
    Under this condition, we have
    \begin{equation*}
        \frac{\der}{\der t}\bar{\va}_1(t) = \mB \bar{\va}_1(t) + \vzero\cdot \bar{\va}_K(t) = \mB \bar{\va}_1(t).
    \end{equation*}
    According to the definition of $\mB = - \mI_{K-1}$, we have
    \begin{equation*}
        \frac{\der}{\der t}\bar{\va}_1(t) = -\bar{\va}_1(t) \quad \Rightarrow \bar{\va}_1(t) = e^{-t}\bar{\va}_1(0).
    \end{equation*}
    If we consider the solution of $\bar{\va}_K(t)$, it has
    \begin{equation*}
        \frac{\der}{\der t}\bar{\va}_K(t) = \mC \cdot \bar{\va}_1(t) + \vzero \cdot \bar{\va}_K(t) = \mC\cdot e^{-t}\cdot \bar{\va}_1(0).
    \end{equation*}
    For the initial condition, i.e., $\bar{\va}(0)=\ve_j$, where  $j\in\{1,2,\ldots, K-1\}$ and $\mC \cdot \bar{\va}_1(0) = 1$, then it has
    \begin{equation*}
        \frac{\der}{\der t}\bar{\va}_K(t) = \mC \cdot \bar{\va}_1(t) + \vzero \cdot \bar{\va}_K(t) = e^{-t},
    \end{equation*}
    which implies
    \begin{equation*}
        \bar{\va}_K(t) = \bar{\va}_K(0) + 1 - e^{-t} = 1 - e^{-t}.
    \end{equation*}
    For the initial condition, $\bar{\va}(0)=\ve_K$, we have $\mC \cdot \bar{\va}_1(0) = 0$ and
    \begin{equation*}
        \bar{\va}_K(t) = \bar{\va}_K(0) + 0 = 1.
    \end{equation*}
    Therefore, we have
    \begin{equation*}
        \exp(t\mA) = \left[
            \begin{matrix}
                e^{-t} & 0 & \ldots & 0 & 0\\
                0 & e^{-t} & \ldots & 0 & 0\\
                \vdots & \vdots & \ddots & \vdots & \vdots\\
                1-e^{-t} & 1-e^{-t} & \ldots & 1-e^{-t} & 1
            \end{matrix}
        \right].
    \end{equation*}
\end{proof}

\begin{lemma}[Forward transition kernel]
    \label{lem:fwd_trans_ker}
    Consider the forward CTMC, i.e., $\{\rvy_t\}_{t=0}^T$ with the infinitesimal operator $R^\to$ given in Eq.~\eqref{eq:fwd_transtion_rate_func}. Then, for any two timestamps $s\le t$, the forward transition probability satisfies, for any $\vy, \vy^\prime \in \gY$, 
    \begin{equation}
        \label{eq:fwd_func_close_form}
        \begin{aligned}
            q^\to_{t|s}(\vy|\vy^\prime) =\prod_{i=1}^{d} & \left[ \delta_{(K,K)}(\vy_i, \vy_i^\prime) + \left(1- \delta_{(K,K)}(\vy_i, \vy_i^\prime)\right)\cdot \delta_{0}(\vy_i-\vy_i^\prime)\cdot e^{-(t-s)} \right.\\
            &\left. + \left(1- \delta_{(K,K)}(\vy_i, \vy_i^\prime)\right)\cdot \delta_{K}(\vy_i)\cdot (1-e^{-(t-s)}) \right].
        \end{aligned}
    \end{equation}
\end{lemma}
\begin{proof}
    Under the matrix presentation, Eq.~\eqref{eq:condi_fwd_dis} implies the transition matrix $\mQ^\to_{t|s}$ can be considered as the solution of the ODE
    \begin{equation*}
        \der \mQ^\to_{t|s}/\der t = \mR^\to \cdot  \mQ^\to_{t|s}\quad \text{where}\quad \mQ^\to_{s|s} =\mI.
    \end{equation*}
    Combining Lemma~\ref{lem:close_solution_qt} and~\ref{lem:exp_R_decomposition}, we have
    \begin{equation}
        \label{eq:conditional_mid}
        \mQ^\to_{t|s} = \exp\left((t-s)\mR^\to\right) = \exp\left((t-s) \mA\right)^{\otimes d},
    \end{equation}
    which implies
    \begin{equation*}
        \mQ^\to_{t|s} = \left[
            \begin{matrix}
                e^{-(t-s)} & 0 & \ldots & 0 & 0\\
                0 & e^{-(t-s)} & \ldots & 0 & 0\\
                \vdots & \vdots & \ddots & \vdots & \vdots\\
                1-e^{-(t-s)} & 1-e^{-(t-s)} & \ldots & 1-e^{-(t-s)} & 1
            \end{matrix}
        \right]^{\otimes d}
    \end{equation*}
    due to the close solution of $\exp((t-s)\mA)$ shown in Lemma~\ref{lem:solution_exp_tA}.
    Combining this result with the calculation of the Kronecker product Lemma~\ref{lem:basic_prod_K_prod}, we have
    \begin{equation*}
        \begin{aligned}
            q^\to_{t|s}(\vy|\vy^\prime) =\prod_{i=1}^{d} & \left[ \delta_{(K,K)}(\vy_i, \vy_i^\prime) + \left(1- \delta_{(K,K)}(\vy_i, \vy_i^\prime)\right)\cdot \delta_{0}(\vy_i-\vy_i^\prime)\cdot e^{-(t-s)} \right.\\
            &\left. + \left(1- \delta_{(K,K)}(\vy_i, \vy_i^\prime)\right)\cdot \delta_{K}(\vy_i)\cdot (1-e^{-(t-s)}) \right].
        \end{aligned}
    \end{equation*}
    where  $\vy,\vy^\prime\in \gY$.
    Hence, the proof is completed.
\end{proof}

\begin{proof}[The proof of Lemma~\ref{lem:fwd_convergence_0}.]
    According to Eq.~\eqref{eq:fwd_vec}, the solution of $\vq^\to_t$ can be calculated as
    \begin{equation*}
        \vq_t^\to = \exp(t\mR^\to) \cdot  \vq^\to_0 = \exp(t\mA)^{\otimes d}\cdot \vq_0^\to = \left[
                \begin{matrix}
                    e^{-t} & 0 & \ldots & 0 & 0\\
                    0 & e^{-t} & \ldots & 0 & 0\\
                    \vdots & \vdots & \ddots & \vdots & \vdots\\
                    0 & 0 & \ldots & e^{-t} & 0 \\
                    1-e^{-t} & 1-e^{-t} & \ldots & 1-e^{-t} & 1
                \end{matrix}
            \right]^{\otimes d}\cdot \vq^\to_0
    \end{equation*}
    where the first equation follows from Lemma~\ref{lem:close_solution_qt}, the second equation follows from Lemma~\ref{lem:exp_R_decomposition}, and the last equation follows from Lemma~\ref{lem:solution_exp_tA}.
    With the calculation of the Kronecker product Lemma~\ref{lem:basic_prod_K_prod}, we have
    \begin{equation}
        \label{eq:forward_trans_exten}
        \vq^\to_{t}(\vy) = \sum_{\vy^\prime\in \gY} \exp(t\mA)^{\otimes d}(\vy, \vy^\prime)\cdot \vq_0^\to(\vy^\prime) = \sum_{\vy^\prime}\left[\prod_{i=1}^d \exp(t\mA)(\vy_i, \vy_i^\prime)\right]\cdot \vq_0^\to(\vy^\prime).
    \end{equation}
    Under this condition, for any $\vy$, we denote the coordinate set of token $K$ as $\gK$ satisfying $\vy_i = K\quad \forall\ i\in \gK(\vy)$, and 
    \begin{equation*}
        \vy_{\gK^c(\vy)} = \vy^\prime_{\gK^c(\vy)}\quad \Leftrightarrow\quad \vy_{i} = \vy^\prime_{i}\ \forall\ i\not\in\gK(\vy).
    \end{equation*}
    Then, Eq.~\eqref{eq:forward_trans_exten} can be rewritten as
    \begin{equation*}
        \begin{aligned}
            \vq_t^\to(\vy) = & \sum_{\vy_{\gK^c(\vy)}^\prime = \vy_{\gK^c(\vy)}}\left[\prod_{j\not\in\gK} \exp(t\mA)(\vy_j, \vy^\prime_j)\cdot 
            \prod_{j\not=i}^d \exp(t\mA)(K,\vy_j^\prime)\right] \cdot \vq_0^\to(\vy^\prime)\\
            & + \sum_{\vy_{\gK^c(\vy)}^\prime \not= \vy_{\gK^c(\vy)}}\left[
            \prod_{j=1}^d \exp(t\mA)(\vy_j,\vy_j^\prime)\right] \cdot \vq_0^\to(\vy^\prime)\\ 
            =&  \sum_{\vy_{\gK^c(\vy)}^\prime = \vy_{\gK^c(\vy)}} \left[e^{-t\cdot |\gK^c(\vy)|}\cdot (1-e^{-t})^{|\gK(\vy)|}\right]\cdot \vq^\to_0(\vy^\prime)\\
            \le &  e^{-t\cdot(d-\numMask{\vy})}\cdot \sum_{\vy_{\gK^c(\vy)}^\prime = \vy_{\gK^c(\vy)}} \vq_0^\to(\vy^\prime) \le \exp(-t\cdot (d-\numMask{\vy})),
        \end{aligned}
    \end{equation*}
    where the second equation establishes since we have
    \begin{equation*}
        \exp(t\mA)(\vy_j,\vy^\prime_j) = \left\{
            \begin{aligned}
                & e^{-t}  && \vy_j = \vy^\prime_j\quad \text{and}\quad \vy_j\not=K\\
                & \vone_{K}(\vy_j^\prime)\cdot (1-e^{-t}) + (1-\vone_{K}(\vy_j^\prime)) && \vy_j = K\\
                & 0 && \text{otherwise}
            \end{aligned}
        \right. .
    \end{equation*}
    According to the definition of $\tilde{q}(\vy)$, we can calculate the normalizing constant of $\tilde{q}$ as
    \begin{equation*}
        \tilde{Z}_t = \sum_{\vy} \exp(-t\cdot(d-\numMask{\vy})) = \sum_{i=0}^d \sum_{\numMask{\vy}=i} \exp(-t\cdot (d-i)) = \sum_{i=1}^d C_d^i\cdot e^{-t\cdot i} = (1+e^{-t})^d.
    \end{equation*}
    Therefore, the KL divergence between $q^\to_t$ and $\tilde{q}_t$ can be written as
    \begin{equation*}
        \begin{aligned}
            & \KL{q_t^\to}{\tilde{q}_t} = \sum_{\vy\in\gY} q_t^\to(\vy)\cdot \ln \frac{q_t^\to(\vy)}{\tilde{q}_t(\vy)} = q_t^\to([K,\ldots,K])\cdot \ln \frac{q_t^\to([K,\ldots, K])}{\tilde{q}_t([K,\ldots, K])} + \sum_{\vy\not=[K,\ldots,K]} q_t^\to(\vy)\cdot \ln \frac{q_t^\to(\vy)}{\tilde{q}_t(\vy)}\\
            & \le \ln \tilde{Z}_t + \sum_{\vy\not=[K,\ldots,K]} q_t^\to(\vy) \ln \frac{q_t^\to(\vy)}{\exp(-t\cdot (d-\numMask{\vy}))/\tilde{Z}_t} = \ln \tilde{Z}_t + \sum_{\vy\not=[K,\ldots,K]} q_t^\to(\vy) \ln \tilde{Z}_t \\
            & \le 2\ln \tilde{Z}_t = 2\ln \left[1 + (1+e^{-t})^d - 1\right] \le 2\cdot (1+e^{-t})^d - 2.
        \end{aligned}
    \end{equation*}
    Suppose we require the TV distance to be small enough, e.g.,
    \begin{equation*}
        \KL{\vq^\to_t}{\tilde{q}_t} \le \epsilon \quad \Leftrightarrow\quad (1+e^{-t})^d - 1 \le \epsilon/2 \quad \Leftrightarrow\quad d\ln(1+e^{-t})\le \ln(1+\epsilon/2),
    \end{equation*}
    then, since $\ln(1+c)\le c$ when $c>0$, the sufficient condition for the establishment of the above equation is to require
    \begin{equation*}
        d\cdot e^{-t} \le \ln(1+\epsilon/2) \quad \Leftrightarrow\quad t\ge \ln(d/\ln(1+\epsilon/2)) \quad \Leftarrow\quad  t\ge \ln(4d/\epsilon),
    \end{equation*}
    where the last derivation establishes since $\epsilon/4\le \ln(1+\epsilon/2)$ when $\epsilon \le 1$ without loss of generality.
    Hence, the proof is completed.
\end{proof}

\section{Truncated Uniformization Inference Analysis}

\subsection{The Proof of Lemma~\ref{lem:out_degree_rate_wrt_time}}
\label{app_sec:prof_out_degree_rate}

\begin{proof}[The proof of Lemma~\ref{lem:out_degree_rate_wrt_time}]
    According to the definition, we have
\begin{equation*}
    R^\gets_t(\vy) = \sum_{\vy^\prime\not=\vy} R^\gets_t(\vy^\prime, \vy) = \sum_{\vy^\prime\not=\vy} R^\to(\vy, \vy^\prime)\cdot\frac{q^\gets_t(\vy^\prime)}{q^\gets_t(\vy)}
\end{equation*}
Since the definition of the transition rate matrix, i.e., Eq.~\eqref{eq:fwd_transtion_rate_func}, for any $\vy^\prime$ with $\mathrm{Ham}(\vy^\prime, \vy)>1$, it has $R^\to(\vy,\vy^\prime) = 0$. 
Moreover, even when $\mathrm{Ham}(\vy^\prime, \vy) = 1$, it has
\begin{equation*}
    R^\to(\vy,\vy^\prime) = 0 \quad \text{when}\quad \vy_{\DfId{\vy}{\vy^\prime}} \not= K.
\end{equation*}
Define the function to transfer the $i$--th element of $\vy$ ($\vy_i$) from $k^\prime$ to $k$ as 
\begin{equation*}
    \vy[\vy_i\colon k^\prime\to k] = \left[\vy_1, \vy_2, \ldots, \vy_{i-1}, k, \vy_{i+1},\ldots, \vy_d\right].
\end{equation*}
That means $R^\gets_t(\vy)$ can be rewritten as
\begin{equation}
    \label{eq:reverse_transit_upb_1}
        \begin{aligned}
            R^\gets_t(\vy) = \sum_{i, \vy_i=K}\left[\sum_{k=1}^{K-1} R^\to(\vy, \vy[\vy_i\colon K \to k])\cdot \frac{q^\gets_t(\vy[\vy_i\colon  K\to k])}{q^\gets_t(\vy)}\right].
        \end{aligned}
    \end{equation}
    To upper bound the RHS of the above equation, we consider controlling
    \begin{equation}
        \label{eq:reverse_transit_upb_2}
        \begin{aligned}                 &\frac{q^\gets_t(\vy[\vy_i\colon K\to k])}{q^\gets_t(\vy)} = \frac{q^\to_{T-t}(\vy[\vy_i\colon K\to k])}{q^\to_{T-t}(\vy)} = \frac{\sum_{\vy_0\in\gY} q^\to_0(\vy_0)\cdot q^\to_{T-t|0}(\vy[\vy_i\colon K\to k]|\vy_0)}{\sum_{\vy_0\in\gY}q_0^\to(\vy_0)\cdot q^\to_{T-t|0}(\vy|\vy_0)}\\
        & = \frac{\sum_{\vy_0\in\gY}q_0^\to(\vy_0)\cdot q^\to_{T-t|0}(\vy|\vy_0)\cdot \frac{q^\to_{T-t|0}(\vy[\vy_i\colon K\to k]|\vy_0)}{q^\to_{T-t|0}(\vy|\vy_0)}}{\sum_{\vy_0\in\gY}q_0^\to(\vy_0)\cdot q^\to_{T-t|0}(\vy|\vy_0)} = \E_{\rvy_0\sim q^\to_{0|T-t}(\cdot|\vy)}\left[\frac{q^\to_{T-t|0}(\vy[\vy_i\colon K\to k]|\vy_0)}{q^\to_{T-t|0}(\vy|\vy_0)}\right],
        \end{aligned}
    \end{equation}
    where the last equation follows from Bayes' Theorem, i.e.,
    \begin{equation*}
        q^\to_{0|T-t}(\vy_0|\vy)\cdot q^\to_{T-t}(\vy) = q_{T-t|0}^\to(\vy|\vy_0)\cdot q^\to_0(\vy_0)\quad \Leftrightarrow\quad q^\to_{0|T-t}(\vy_0|\vy)\propto  q_{T-t|0}^\to(\vy|\vy_0)\cdot q^\to_0(\vy_0).
    \end{equation*}
    Then, we only need to control $q^\to_{T-t|0}(\vy[\vy_i\rightarrow k]|\vy_0)/q^\to_{T-t|0}(\vy|\vy_0)$ where both the denominator and the numerator can be calculated accurately by Lemma~\ref{lem:fwd_trans_ker}.
    Specifically, we have
    \begin{equation*}
        \begin{aligned}
            q^\to_{T-t|0}(\vy|\vy_0) = \prod_{j\in\{1,\ldots, i-1, i+1, \ldots, d\}} & \left[ \vone_{(K,K)}(\vy_j, \vy_{0,j}) + \left(1- \vone_{(K,K)}(\vy_j, \vy_{0,j})\right)\cdot \vone_{0}(\vy_j-\vy_{0,j})\cdot e^{-(T-t)} \right.\\
            &\left. + \left(1- \vone_{(K,K)}(\vy_j, \vy_{0,j})\right)\cdot \vone_{K}(\vy_j)\cdot (1-e^{-(T-t)}) \right]\cdot\\
            & \left[ \vone_{(K,K)}(K, \vy_{0,i})  + \left(1- \vone_{(K,K)}(K, \vy_{0,i})\right)\cdot (1-e^{-(T-t)}) \right]
        \end{aligned}
    \end{equation*}
    and 
    \begin{equation*}
        \begin{aligned}
            q^\to_{T-t|0}(\vy[\vy_i\colon K\to k]|\vy_0) = \prod_{j\in\{1,\ldots, i-1, i+1, \ldots, d\}} & \left[ \vone_{(K,K)}(\vy_j, \vy_{0,j}) + \left(1- \vone_{(K,K)}(\vy_j, \vy_{0,j})\right)\cdot \vone_{0}(\vy_j-\vy_{0,j})\cdot e^{-(T-t)} \right.\\
            &\left. + \left(1- \vone_{(K,K)}(\vy_j, \vy_{0,j})\right)\cdot \vone_{K}(\vy_j)\cdot (1-e^{-(T-t)}) \right]\cdot\\
            & \left[ \left(1- \vone_{(K,K)}(k, \vy_{0,i})\right)\cdot \vone_{0}(k-\vy_{0,i})\cdot e^{-(T-t)} \right].
        \end{aligned}
    \end{equation*}
    Since the factor except for the $i$--th term will be canceled, we have
    \begin{equation}
        \begin{aligned}
            &\frac{q^\to_{T-t|0}(\vy[\vy_i\rightarrow k]|\vy_0)}{q^\to_{T-t|0}(\vy|\vy_0)} = \frac{ \left(1- \vone_{(K,K)}(k, \vy_{0,i})\right)\cdot \vone_{0}(k-\vy_{0,i})\cdot e^{-(T-t)} }{ \vone_{(K,K)}(K, \vy_{0,i})  + \left(1- \vone_{(K,K)}(K, \vy_{0,i})\right)\cdot (1-e^{-(T-t)}) }\\
            & =  \frac{  \vone_{0}(k-\vy_{0,i})\cdot e^{-(T-t)} }{  1-e^{-(T-t)} }\le \frac{e^{-(T-t)}}{1-e^{-(T-t)}}  = \frac{1}{e^{(T-t)}-1}.
        \end{aligned}
        \label{eq:q_0_k_score_bound}
    \end{equation}
    Plugging this result into Eq.~\eqref{eq:reverse_transit_upb_2}, the density ratio of the reverse process will have
    \begin{equation}
        \frac{q^\gets_t(\vy[\vy_i\rightarrow k])}{q^\gets_t(\vy)} = \E_{\rvy_0\sim q^\to_{0|T-t}(\cdot|\vy)}\left[\frac{q^\to_{T-t|0}(\vy[\vy_i\colon K\to k]|\vy_0)}{q^\to_{T-t|0}(\vy|\vy_0)}\right] \le \frac{1}{e^{(T-t)}-1}.
        \label{eq:density_ratio_reverse}
    \end{equation}
    Combining with the fact, i.e.,
    \begin{equation*}
        R^\to(\vy, \vy[\vy_i\colon K \to k]) = 1
    \end{equation*}
    from Eq.~\eqref{eq:fwd_transtion_rate_func}, Eq.~\eqref{eq:reverse_transit_upb_1} can be upper bounded as
    \begin{equation*}
        R^\gets_t(\vy) =  \sum_{i, \vy_i=K}\left[\sum_{k=1}^{K-1} \frac{q^\gets_t(\vy[\vy_i\colon  K\to k])}{q^\gets_t(\vy)}\right] \le \frac{\mathrm{numK}(\vy)\cdot K}{e^{(T-t)}-1}.
    \end{equation*}
    Hence, the proof is completed.
\end{proof}
\begin{remark}
    Here, an interesting property is that compared with the upper bound of $\beta_t(\vy)$  in the \emph{uniform forward process}~\cite{chen2024convergence}, i.e., 
    \begin{equation*}
        \sum_{\vy^\prime\not=\vy} R^\gets_t(\vy^\prime,\vy) \le K\cdot d\cdot \frac{1+e^{-2(T-t)}}{1-e^{-2(T-t)}}\le K\cdot d\cdot (1+(T-t)^{-1}).
    \end{equation*}
    the upper bound of $\beta_t(\vy)$ in \emph{absorbing forward process} will only be 
    \begin{equation*}
        \sum_{\vy^\prime\not=\vy} R^\gets_t(\vy^\prime,\vy) \le \textcolor{red}{K\cdot \mathrm{numK}(\vy)}\cdot \frac{e^{-(T-t)}}{1-e^{-(T-t)}}.
    \end{equation*}
    The latter upper bound is strictly better compared with the former one, since the number of mask tokens, i.e., $\mathrm{numK}(\vy)\le d$.
    Besides, with the time growth (from $0$ to $T$), $\mathrm{numK}(\vy)$ will be monotonic decrease for $R^\gets_t(\vy)$ (from $d$ to $0$).
    Since the dominating term in the complexity analysis of truncated uniformization is $\beta_t$,  the discrete diffusion models with absorbing forward process are expected to have a better result.
    The mechanism of the acceleration can be explained in one sentence, i.e.,
    \begin{framed}
        \centering
        \emph{At each uniformization step, absorbing the discrete diffusion model knows the token needs (masked token)/ or does not need (unmasked token) to denoise, and an unmasked token will not be denoised twice.}
    \end{framed}
    Rigorously, this property can be summarized by Lemma~\ref{lem:reverse_nonzero_transitions}.
\end{remark}

\begin{lemma}
    \label{lem:reverse_nonzero_transitions}
    Suppose Assumption~\ref{ass:mask_init} hold, and $0<t_0\le t$, we have $q^\gets_{t|t_0}(\vy|\vy_0)\not=0$ if and only if 
    \begin{equation*}
        \vy\in \gY^\gets(\vy_0) = \left\{\vy^\prime| \forall i,\quad \vy_{0,i} = K\  \text{or}\ \vy^\prime_i = \vy_{0,i}  \right\}.
    \end{equation*}
\end{lemma}
\begin{proof}
    According to the Bayes' theorem, for any $t\ge t_0$, it has
    \begin{equation}
        \label{eq:joint_equ_fw_bw}
        \begin{aligned}
            & q^\gets_{t,t_0}(\vy, \vy_0) = q^\gets_{t|t_0}(\vy|\vy_0)\cdot q^\gets_{t_0}(\vy_0) = q^\to_{T-t, T-t_0}(\vy, \vy_0)\\
            & = q^\to_{T-t_0, T-t}(\vy_0, \vy) = q^\to_{T-t_0|T-t}(\vy_0|\vy)\cdot q^\to_{T-t}(\vy),
        \end{aligned}
    \end{equation}
    where the third equation follows from the reversibility of the absorbing forward process shown in~\citet{campbell2022continuous}.
    Following from the forward transition kernel shown in Lemma~\ref{lem:fwd_trans_ker}, we know that 
    \begin{equation}
        \label{eq:condi_fw_not0}
         q^\to_{T-t_0|T-t}(\vy_0|\vy)\not=0\quad  \Leftrightarrow\quad \vy_0\in \gY^\to(\vy) = \left\{\vy^\prime|\ \forall i,\quad \vy^\prime_i = \vy_i\ \text{or}\ \vy^\prime_i = K \right\}.
    \end{equation}
    Combining Assumption~\ref{ass:mask_init} and Lemma~\ref{lem:fwd_trans_ker}, we have $q^\to_{\tau}(\vy)>0$ for all $\vy\in\gY$, which implies
    \begin{equation}
        \label{eq:margin_fw_not0}
        \begin{aligned}
            q^\gets_{t_0}(\vy_0) = q^\to_{T-t_0}(\vy_0)>0\quad \text{and}\quad  q^\to_t(\vy)>0.
        \end{aligned} 
    \end{equation}
    Then, we can summarize 
    \begin{equation*}
        q^\gets_{t|t_0}(\vy|\vy_0)\not=0\quad \Leftrightarrow\quad \vy\in \gY^\gets(\vy_0) = \left\{\vy^\prime| \forall i,\quad \vy_{0,i} = K\  \text{or}\ \vy^\prime_i = \vy_{0,i}  \right\}.
    \end{equation*}
    Hence, the proof is completed.
\end{proof}

\subsection{The convergence of Alg.~\ref{alg:uni_inf}}
\label{app_sec:prof_convergence_reverse}

Suppose, with the infinitesimal reverse transition rate, the particles in Alg.~\ref{alg:uni_inf} during the reverse process are denotes as random variables $\{\hat{\rvy}_t\}_{t=0}^{T-\delta}$, whose underlying distributions are $\hat{q}_t$.
Then, the implementation will be equivalent to the following Poisson process. 
For $t\in(t_{w-1}, t_w]$, $\hat{\rvy}_{t_{w-1}} = \vy_0$ and $\hat{\rvy}_t = \vy$, 
\begin{enumerate}[leftmargin=*]
    \item With probability $\Delta t\cdot \beta_{t_{w}}(\vy_0)$, allow a state transition.
    \item Conditioning on an allowed transition, move from $\vy$ to $\vy^\prime$ with probability
    \begin{equation*}
        \hat{\mM}_{t|t_{w-1}}(\vy^\prime|\vy, \vy_0) = \left\{
            \begin{aligned}
                & \beta_{t_{w}}^{-1}(\vy_0)\cdot \hat{R}_{t,\vy_0}(\vy^\prime, \vy) && \vy^\prime\not=\vy\\
                & 1- \beta_{t_{w}}^{-1}(\vy_0)\hat{R}_{t,\vy_0}(\vy) && \text{otherwise}
            \end{aligned}
        \right. .
    \end{equation*}
\end{enumerate}
Here we should note that
\begin{equation*}
    \hat{R}_{t,\vy_0}(\vy) \le \beta_t(\vy) = K\cdot \numMask{\vy}\cdot \frac{1}{e^{T-t}-1}\le K\cdot \numMask{\vy_0}\cdot \frac{1}{e^{T-t_{w}}-1} = \beta_{t_{w}}(\vy_0), 
\end{equation*}
where the second inequality established since $\numMask{\hat{\rvy}_t}\le \numMask{\hat{\rvy}_{t_{w-1}}}$ and $(e^{T-t}-1)^{-1}$ is monotonic increasing.
Under these two steps, the practical conditional probability satisfies 
\begin{equation}
    \label{eq:prac_hat_infini_ope}
    \begin{aligned}
        \hat{q}_{t+\Delta t| t, t_{w-1}}(\vy^\prime|\vy, \vy_0) = & \left\{
            \begin{aligned}
            & \Delta t\cdot \beta_{t_{w}}(\vy_0) \cdot \hat{R}_{t,\vy_0}(\vy^\prime, \vy)\cdot \beta^{-1}_{t_{w}}(\vy_0)  && \vy^\prime\not=\vy\\
            & 1-\Delta t\cdot \beta_{t_{w}}(\vy_0) + \Delta t \cdot \beta_{t_{w}}(\vy_0) \cdot (1-\beta_{t_{w}}(\vy_0)^{-1}\cdot \hat{R}_{t,\vy_0}(\vy))  && \vy^\prime = \vy,
            \end{aligned}    
        \right. \\
        = & \left\{
            \begin{aligned}
                & \Delta t \cdot \hat{R}_{t,\vy_0}(\vy^\prime,\vy) && \vy^\prime\not=\vy\\
                & 1-\Delta t\cdot \hat{R}_{t,\vy_0}(\vy) && \vy^\prime = \vy
            \end{aligned}
        \right. .
    \end{aligned}
\end{equation}

\begin{lemma}
    \label{lem:lim_deltat_ln}
    Following the notations shown in Section~\ref{sec:app_notations}, we have
    \begin{equation*}
         \lim_{\Delta t\rightarrow 0}\left[\Delta t^{-1} \cdot \ln \frac{1-\sum_{\vy^\prime\not=\vy}q^\gets_{t+\Delta t|t, t_{w-1}}(\vy^\prime|\vy,\vy_0)}{1-\sum_{\vy^\prime\not=\vy}\hat{q}_{t+\Delta t|t, t_{w-1}}(\vy^\prime|\vy,\vy_0)}\right] = \hat{R}_{t,\vy_0}(\vy) - R^\gets_t(\vy).
    \end{equation*}
\end{lemma}
\begin{proof}
    Since we have required $\Delta t\rightarrow 0$, for any $\vy^\prime\not=\vy$, it has
    \begin{equation*}
        \begin{aligned}
            &\hat{q}_{t+\Delta t|t, t_{w-1}}(\vy^\prime|\vy,\vy_0)\rightarrow \hat{q}_{t|t}(\vy^\prime|\vy,\vy_0)=0\\
            &\text{and}\quad q^\gets_{t+\Delta t|t, t_{w-1}}(\vy^\prime|\vy,\vy_0) = q^\gets_{t+\Delta t|t}(\vy^\prime|\vy)\rightarrow q^\gets_{t|t}(\vy^\prime|\vy)=0,
        \end{aligned}
    \end{equation*}
    where the first row follows from Eq.~\eqref{eq:prac_hat_infini_ope} and the second row follows from Lemma.~\ref{lem:absorbing_reverse_de}.
    This automatically makes 
    \begin{equation*}
        \left|\frac{\sum_{\vy^\prime\not=\vy}\left(\hat{q}_{t+\Delta t|t, t_{w-1}}(\vy^\prime|\vy,\vy_0) - q^\gets_{t+\Delta t|t,t_{w-1}}(\vy^\prime|\vy,\vy_0)\right)}{1-\sum_{\vy^\prime\not=\vy}\hat{q}_{t+\Delta t|t,t_{w-1}}(\vy^\prime|\vy,\vy_0)}\right|\le \frac{1}{2}<1.
    \end{equation*}
    Under this condition, we have
    \begin{equation*}
        \begin{aligned}
            & \ln \frac{1-\sum_{\vy^\prime\not=\vy}q^\gets_{t+\Delta t|t, t_{w-1}}(\vy^\prime|\vy,\vy_0)}{1-\sum_{\vy^\prime\not=\vy}\hat{q}_{t+\Delta t|t,t_{w-1}}(\vy^\prime|\vy,\vy_0)} =\ln\left[ 1 + \frac{\sum_{\vy^\prime\not=\vy} \left(\hat{q}_{t+\Delta t|t, t_{w-1}}(\vy^\prime|\vy,\vy_0)-q^\gets_{t+\Delta t|t}(\vy^\prime|\vy)\right)}{1-\sum_{\vy^\prime\not=\vy}\hat{q}_{t+\Delta t|t,t_{w-1}}(\vy^\prime|\vy,\vy_0)}\right]\\
            & = \sum_{i=1}^\infty \frac{(-1)^{i+1}}{i}\cdot \left[\frac{\sum_{\vy^\prime\not=\vy} \left(\hat{q}_{t+\Delta t|t, t_{w-1}}(\vy^\prime|\vy,\vy_0)-q^\gets_{t+\Delta t|t, t_{w-1}}(\vy^\prime|\vy,\vy_0)\right)}{1-\sum_{\vy^\prime\not=\vy}\hat{q}_{t+\Delta t|t,t_{w-1}}(\vy^\prime|\vy,\vy_0)}\right]^i,
        \end{aligned}
    \end{equation*}
    which implies (with the dominated convergence theorem)
    \begin{equation*}
        \begin{aligned}
             & \lim_{\Delta t\rightarrow 0}\left[\Delta t^{-1} \cdot \ln \frac{1-\sum_{\vy^\prime\not=\vy}q^\gets_{t+\Delta t|t,t_{w-1}}(\vy^\prime|\vy,\vy_0)}{1-\sum_{\vy^\prime\not=\vy}\hat{q}_{t+\Delta t|t, t_{w-1}}(\vy^\prime|\vy,\vy_0)}\right]\\
             & = \sum_{i=1}^\infty \frac{(-1)^{i+1}}{i}\cdot \lim_{\Delta t\rightarrow 0} \frac{\sum_{\vy^\prime\not=\vy} \left(\hat{q}_{t+\Delta t|t,t_{w-1}}(\vy^\prime|\vy,\vy_0)-q^\gets_{t+\Delta t|t, t_{w-1}}(\vy^\prime|\vy,\vy_0)\right)}{\Delta t}\\
             & \quad \cdot \lim_{\Delta t \rightarrow 0}\frac{\left(\sum_{\vy^\prime\not=\vy} \left(\hat{q}_{t+\Delta t|t,t_{w-1}}(\vy^\prime|\vy,\vy_0)-q^\gets_{t+\Delta t|t, t_{w-1}}(\vy^\prime|\vy,\vy_0)\right)\right)^{i-1}}{\left(1-\sum_{\vy^\prime\not=\vy}\hat{q}_{t+\Delta t|t, t_{w-1}}(\vy^\prime|\vy,\vy_0)\right)^i}.
        \end{aligned}
    \end{equation*}
    Only when $i=1$, we have
    \begin{equation*}
        \lim_{\Delta t \rightarrow 0}\frac{\left(\sum_{\vy^\prime\not=\vy} \left(\hat{q}_{t+\Delta t|t, t_{w-1}}(\vy^\prime|\vy,\vy_0)-q^\gets_{t+\Delta t|t,t_{w-1}}(\vy^\prime|\vy,\vy_0)\right)\right)^{i-1}}{\left(1-\sum_{\vy^\prime\not=\vy}\hat{q}_{t+\Delta t|t}(\vy^\prime|\vy)\right)^i}=1,
    \end{equation*}
    otherwise it will be equivalent to $0$.
    Therefore, we have
    \begin{equation*}
        \begin{aligned}
            &\lim_{\Delta t\rightarrow 0}\left[\Delta t^{-1} \cdot \ln \frac{1-\sum_{\vy^\prime\not=\vy}q^\gets_{t+\Delta t|t,t_{w-1}}(\vy^\prime|\vy,\vy_0)}{1-\sum_{\vy^\prime\not=\vy}\hat{q}_{t+\Delta t|t,t_{w-1}}(\vy^\prime|\vy,\vy_0)}\right]\\
            & = \lim_{\Delta t\rightarrow 0} \frac{\sum_{\vy^\prime\not=\vy} \left(\hat{q}_{t+\Delta t|t,t_{w-1}}(\vy^\prime|\vy,\vy_0)-q^\gets_{t+\Delta t|t, t_{w-1}}(\vy^\prime|\vy,\vy_0)\right)}{\Delta t}\\
            & = \sum_{\vy^\prime\not=\vy} \left(\hat{R}_{t,\vy_0}(\vy^\prime,\vy) - R^\gets_t(\vy^\prime,\vy)\right) = \hat{R}_{t,\vy_0}(\vy) - R^\gets_t(\vy),
        \end{aligned}
    \end{equation*}
    where the second equation follows from Eq.~\eqref{eq:prac_hat_infini_ope} and the second row follows from Lemma.~\ref{lem:absorbing_reverse_de}.
    Hence, the proof is completed.
\end{proof}

\begin{theorem}[The convergence of Alg.~\ref{alg:uni_inf}]
    \label{thm:convergence_unif_reverse}
    Suppose Assumption~\ref{ass:score_approximation_error} and~\ref{ass:mask_init} hold, if 
    Alg.~\ref{alg:uni_inf} has 
    \begin{equation*}
        t_0 = 0,\quad t_W= T-\delta,\quad \text{and}\quad \epsilon_{\text{score}}\le T^{-1/2}\cdot \epsilon \quad \text{where}\quad T = \ln(4d/\epsilon^2) \quad \text{and}\quad \delta\le d^{-1}\epsilon,
    \end{equation*}   
    the TV distance between the target discrete distribution $q_*$ and the underlying distribution of the output particle $\hat{q}_{T-\delta}$ will satisfy $\TVD{q_*}{\hat{q}_{T-\delta}}\le 2\epsilon$.
\end{theorem}
\begin{proof}
Here we provide the upper bound of TV distance accumulation in a specific segment, e.g., from $t_{w-1}$ to $t_w$.
According to the chain rule of KL divergence, i.e., Lemma~\ref{lem:chain_kl}, we have
\begin{equation}
    \label{ineq:convergence_start}
    \begin{aligned}
    &\KL{q^\gets_{t_w}}{\hat{q}_{t_w}} \le \KL{q^\gets_{t_{w-1}}}{\hat{q}_{t_{w-1}}} + \E_{\rvy_0\sim q^\gets_{t_{w-1}}}\left[\KL{q^\gets_{t_w|t_{w-1}}(\cdot|\rvy_0)}{\hat{q}_{t_w|t_{w-1}}(\cdot|\rvy_0)}\right]\\
    & = \KL{q^\gets_{t_{w-1}}}{\hat{q}_{t_{w-1}}} + \int_{t_{w-1}}^{t_w} \der \E_{\rvy_0\sim q^\gets_{t_{w-1}}}\left[\KL{q^\gets_{t|t_{w-1}}(\cdot|\rvy_0)}{\hat{q}_{t|t_{w-1}}(\cdot|\rvy_0)}\right]
    \end{aligned}
\end{equation}
Then, it has
\begin{equation*}
    \begin{aligned}
        & \der \E_{\rvy_0\sim q^\gets_{t_{w-1}}}\left[\KL{q^\gets_{t|t_{w-1}}(\cdot|\rvy^\gets_0)}{\hat{q}_{t|t_{w-1}}(\cdot|\rvy^\gets_0)}\right]/\der t\\
        & = \lim_{\Delta \rightarrow 0}\ (\Delta t)^{-1}\cdot \E_{\rvy_0\sim q^\gets_{t_{w-1}}}\left[\KL{q^\gets_{t+\Delta t|t_{w-1}}(\cdot|\rvy_0)}{\hat{q}_{t+\Delta t|t_{w-1}}(\cdot|\rvy_0)}  - \KL{q^\gets_{t|t_{w-1}}(\cdot|\rvy_0)}{\hat{q}_{t|t_{w-1}}(\cdot|\rvy_0)}\right]\\
        & \le  \lim_{\Delta \rightarrow 0}\ (\Delta t)^{-1}\cdot \E_{\rvy_0\sim q^\gets_{t_{w-1}}}\left[\E_{\rvy\sim q^\gets_{t|{t_{w-1}}}(\cdot|\rvy_0)}\left(\KL{q^\gets_{t+\Delta t|t, t_{w-1}}(\cdot|\rvy, \rvy_0)}{\hat{q}_{t+\Delta t|t, t_{w-1}}(\cdot|\rvy, \rvy_0)}\right)\right]
    \end{aligned}
\end{equation*}
where the inequality follows from the chain rule of the KL divergence, i.e., Lemma~\ref{lem:chain_kl}. 
Then, it has
\begin{equation}
    \label{ineq:dKL_ineq_init}
    \begin{aligned}
        & \der \E_{\rvy_0\sim q^\gets_{t_{w-1}}}\left[\KL{q^\gets_{t|t_{w-1}}(\cdot|\rvy_0)}{\hat{q}_{t|t_{w-1}}(\cdot|\rvy_0)}\right]/\der t\\
        & \le \sum_{\vy\in \gY, \vy_0\in \gY^\to(\vy)} q^\gets_{t, t_{w-1}}(\vy, \vy_0)\cdot  \underbrace{\lim_{\Delta t\to 0} \left[\frac{\KL{q^\gets_{t+\Delta t|t,t_{w-1}}(\cdot|\vy, \vy_0)}{\hat{q}_{t+\Delta t|t,t_{w-1}}(\cdot|\vy, \vy_0)}}{\Delta t}\right]}_{\text{Term 1}}
    \end{aligned}
\end{equation}
where the inequality and the notation $\gY^\to(\cdot)$ follows from Lemma~\ref{lem:reverse_nonzero_transitions}.
For each $\vy^\gets\in \gY, \vy^\gets_0\in \gY^\to(\vy)$, we focus on Term 1 of Eq.~\eqref{ineq:dKL_ineq_init}, and have
\begin{equation}
    \label{eq:term1_equ}
    \begin{aligned}
        \text{Term 1} & = \lim_{\Delta t\rightarrow 0}\left[\Delta t^{-1}\cdot \sum_{\vy^\prime\in\gY} q^\gets_{t+\Delta t|t, t_{w-1}}(\vy^\prime|\vy, \vy_0)\cdot \ln \frac{q^\gets_{t+\Delta t|t, t_{w-1}}(\vy^\prime|\vy, \vy_0)}{\hat{q}_{t+\Delta t|t, t_{w-1}}(\vy^\prime|\vy,\vy_0)}\right]\\
        & = \underbrace{\lim_{\Delta t\rightarrow 0}\left[\sum_{\vy^\prime\not=\vy}\frac{q^\gets_{t+\Delta t|t, t_{w-1}}(\vy^\prime|\vy,\vy_0)}{\Delta t} \cdot \ln \frac{q^\gets_{t+\Delta t|t, t_{w-1}}(\vy^\prime|\vy,\vy_0)}{\hat{q}_{t+\Delta t|t, t_{w-1}}(\vy^\prime|\vy,\vy_0)} \right]}_{\text{Term 1.1}}+\\
        &\quad \underbrace{\lim_{\Delta t\rightarrow 0}\left[\Delta t^{-1}\cdot \left(1-\sum_{\vy^\prime\not=\vy}q^\gets_{t+\Delta t|t, t_{w-1}}(\vy^\prime|\vy,\vy_0)\right) \cdot \ln \frac{1-\sum_{\vy^\prime\not=\vy}q^\gets_{t+\Delta t|t, t_{w-1}}(\vy^\prime|\vy,\vy_0)}{1-\sum_{\vy^\prime\not=\vy}\hat{q}_{t+\Delta t|t,t_{w-1}}(\vy^\prime|\vy,\vy_0)} \right]}_{\text{Term 1.2}}.
    \end{aligned}
\end{equation}
For Term 1.1, we have
\begin{equation}
    \begin{aligned}
        \label{ineq:term1.1_equ}
        \text{Term 1.1} = & \sum_{\vy^\prime\not=\vy} \lim_{\Delta t\rightarrow 0}\left[\frac{q^\gets_{t+\Delta t|t,t_{w-1}}(\vy^\prime|\vy,\vy_0)}{\Delta t}\right]\cdot \lim_{\Delta t\rightarrow 0}\left[ \ln \frac{q^\gets_{t+\Delta t|t, t_{w-1}}(\vy^\prime|\vy,\vy_0)}{\hat{q}_{t+\Delta t|t,t_{w-1}}(\vy^\prime|\vy,\vy_0)}\right]\\
        = & \sum_{\vy^\prime\not=\vy} R^\gets_{t}(\vy^\prime,\vy)\cdot \ln\left[\lim_{\Delta t\rightarrow 0}\left(\frac{q^\gets_{t+\Delta t|t,t_{w-1}}(\vy^\prime|\vy,\vy_0)}{\Delta t}\cdot \frac{\Delta t}{\hat{q}_{t+\Delta t|t, t_{w-1}}(\vy^\prime|\vy,\vy_0)}\right)\right]\\
        = & \sum_{\vy^\prime\not=\vy} R^\gets_t(\vy^\prime,\vy) \cdot \ln\frac{R_t^\gets(\vy^\prime,\vy)}{\hat{R}_{t,\vy_0}(\vy^\prime,\vy)}, 
    \end{aligned}
\end{equation}
where the last equation follows from Lemma~\ref{lem:absorbing_reverse_de} and Eq.~\eqref{eq:prac_hat_infini_ope}.
For Term 1.2, we have
\begin{equation}
    \begin{aligned}
        \label{ineq:term1.2_equ}
        \text{Term 1.2} = & \lim_{\Delta t\rightarrow 0} \left[1-\sum_{\vy^\prime\not=\vy}q^\gets_{t+\Delta t|t, t_{w-1}}(\vy^\prime|\vy,\vy_0)\right]\\
        & \cdot \lim_{\Delta t\rightarrow 0}\left[\Delta t^{-1} \cdot \ln \frac{1-\sum_{\vy^\prime\not=\vy}q^\gets_{t+\Delta t|t, t_{w-1}}(\vy^\prime|\vy,\vy_0)}{1-\sum_{\vy^\prime\not=\vy}\hat{q}_{t+\Delta t|t, t_{w-1}}(\vy^\prime|\vy,\vy_0)}\right]\le  1\cdot (\hat{R}_{t,\vy_0}(\vy) - R^\gets_t(\vy))
    \end{aligned}
\end{equation}
where the first inequality follows from Lemma~\ref{lem:lim_deltat_ln}.
Plugging Eq.~\eqref{ineq:term1.1_equ}, Eq.~\eqref{ineq:term1.2_equ} and Eq.~\eqref{eq:term1_equ}, into Eq.~\eqref{ineq:dKL_ineq_init} we have
\begin{equation}
    \label{ineq:dKL_ineq_mid}
    \begin{aligned}
        & \der \E_{\rvy_0\sim q^\gets_{t_{w-1}}}\left[\KL{q^\gets_{t|t_{w-1}}(\cdot|\rvy_0)}{\hat{q}_{t|t_{w-1}}(\cdot|\rvy_0)}\right]/\der t\\
        & \le \sum_{\vy\in \gY, \vy_0\in \gY^\to(\vy)} q^\gets_{t, t_{w-1}}(\vy, \vy_0)\cdot  \left(\sum_{\vy^\prime\not=\vy} R^\gets_t(\vy^\prime,\vy) \cdot \ln\frac{R_t^\gets(\vy^\prime,\vy)}{\hat{R}_{t,\vy_0}(\vy^\prime,\vy)} + \hat{R}_{t,\vy_0}(\vy) - R^\gets_t(\vy)\right).
    \end{aligned}
\end{equation}
Then, for any $\vy\in \gY$ and $\vy_0\in \gY^\to(\vy)$, we have
\begin{equation}
    \label{eq:dKL_gap_ScoEst_err}
    \begin{aligned}
        &\sum_{\vy^\prime\not=\vy} R^\gets_t(\vy^\prime,\vy) \cdot \ln\frac{R_t^\gets(\vy^\prime,\vy)}{\hat{R}_{t,\vy_0}(\vy^\prime,\vy)} + \hat{R}_{t,\vy_0}(\vy) - R^\gets_t(\vy)\\
        & = \sum_{\vy^\prime\not=\vy} R^\gets_t(\vy^\prime,\vy)\ln \frac{R^\gets_t(\vy^\prime,\vy)}{\tilde{R}_t(\vy^\prime,\vy)} + \tilde{R}_t(\vy) - R^\gets_t(\vy)\\
        &\quad \underbrace{+\sum_{\vy^\prime\not=\vy}R^\gets_t(\vy^\prime,\vy)\ln\frac{\tilde{R}_t(\vy^\prime,\vy)}{\hat{R}_{t,\vy_0}(\vy^\prime,\vy)} + \hat{R}_{t,\vy_0}(\vy) - \tilde{R}_t(\vy)}_{\text{Term 2}}.
    \end{aligned}
\end{equation}
When $\tilde{R}_t(\vy)\le \beta_{t_w}(\vy_0)$, due to Eq.~\eqref{def:prac_infi_oper_1}, we have 
\begin{equation*}
    \hat{R}_{t,\vy_0}(\vy^\prime, \vy) = \tilde{R}_t(\vy^\prime,\vy)\quad \text{and}\quad \hat{R}_{t,\vy_0}(\vy)=\sum_{\vy^\prime\not=\vy}\hat{R}_{t,\vy_0}(\vy^\prime,\vy) = \sum_{\vy^\prime\not=\vy}\tilde{R}_t(\vy^\prime,\vy) = \tilde{R}_t(\vy)
\end{equation*}
which implies $\text{Term 2}=0$ in Eq.~\eqref{eq:dKL_gap_ScoEst_err}.
Otherwise, we have
\begin{equation*}
    \frac{\hat{R}_{t,\vy_0}(\vy^\prime, \vy)}{\tilde{R}_t(\vy^\prime,\vy)} = \frac{\beta_{t_w}(\vy_0)}{\tilde{R}_t(\vy)} \quad\text{and}\quad \frac{\hat{R}_{t,\vy_0}(\vy)}{\tilde{R}_t(\vy)}=\frac{\beta_{t_w}(\vy_0)}{\tilde{R}_t(\vy)},
\end{equation*}
which implies
\begin{equation*}
    \begin{aligned}
        \text{Term 2} & = \sum_{\vy^\prime\not=\vy}R_t^\gets(\vy^\prime,\vy)\cdot \ln \frac{\tilde{R}_t(\vy)}{\beta_{t_w}(\vy_0)} + \beta_{t_w}(\vy_0) - \tilde{R}_t(\vy)\\
        & = R^\gets_t(\vy)\cdot \ln \left[1+ \frac{\tilde{R}_t(\vy)-\beta_{t_w}(\vy_0)}{\beta_{t_w}(\vy_0)}\right]+\beta_{t_w}(\vy_0) - \tilde{R}_t(\vy)\\
        & \le \beta_{t_w}(\vy_0)\cdot \left[\frac{\tilde{R}_t(\vy)-\beta_{t_w}(\vy_0)}{\beta_{t_w}(\vy_0)}\right] + \beta_{t_w}(\vy_0) -\tilde{R}_t(\vy) = 0,
    \end{aligned}
\end{equation*}
where the last inequality follows from 
\begin{equation*}
    \vy_0\in \gY^\to(\vy)\quad \Rightarrow\quad \numMask{\vy}\le \numMask{\vy_0}   \quad \Rightarrow\quad R^\gets_t(\vy)\le \beta_{t_w}(\vy_0).
\end{equation*}
Combining with Eq.~\eqref{eq:dKL_gap_ScoEst_err} and Eq.~\eqref{ineq:dKL_ineq_mid}, we have 
\begin{equation}
    \label{ineq:ineq:dKL_ineq_fin}
    \begin{aligned}
        & \der \E_{\rvy_0\sim q^\gets_{t_{w-1}}}\left[\KL{q^\gets_{t|t_{w-1}}(\cdot|\rvy_0)}{\hat{q}_{t|t_{w-1}}(\cdot|\rvy_0)}\right]/\der t\\
        & \le \sum_{\vy\in \gY} q^\gets_t(\vy) \cdot \left( \sum_{\vy^\prime\not=\vy} R^\gets_t(\vy^\prime,\vy) \cdot \ln\frac{R_t^\gets(\vy^\prime,\vy)}{\tilde{R}_t(\vy^\prime,\vy)} + \tilde{R}_t(\vy) - R^\gets_t(\vy) \right)\\
        &=\sum_{\vy\in \gY} q^\gets_t(\vy) \cdot \left( \sum_{\vy^\prime\not=\vy} R^\gets_t(\vy^\prime,\vy) \cdot \ln\frac{R_t^\gets(\vy^\prime,\vy)}{\tilde{R}_t(\vy^\prime,\vy)} + \sum_{\vy'\neq \vy}\tilde{R}_t(\vy',\vy) -\sum_{\vy'\neq \vy} R^\gets_t(\vy',\vy) \right)\\
        & = \sum_{\vy\in\gY} q^\gets_t(\vy)\cdot\sum_{\vy^\prime\neq \vy} R^\to(\vy,\vy^\prime)\cdot \left[-\frac{q_t^\gets(\vy^\prime)}{q^\gets_t(\vy)} + \tilde{v}_{t,\vy}(\vy^\prime)+ \frac{q_t^\gets(\vy^\prime)}{q^\gets_t(\vy)}\ln \frac{q_t^\gets(\vy^\prime)}{q^\gets_t(\vy) \tilde{v}_{t,\vy}(\vy^\prime)}\right]\\
        &=\sum_{\vy\in\gY} q^\gets_t(\vy)\cdot\sum_{\vy^\prime\neq \vy} R^\to(\vy,\vy^\prime) \Breg{\frac{q_t^\gets(\vy^\prime)}{q^\gets_t(\vy)}}{\tilde{v}_{t,\vy}(\vy^\prime)},
    \end{aligned}
\end{equation}
where $D_\phi$ is the Bregman divergence  with $\phi(c)=c\ln c$ (as Eq.~\eqref{eq:score_estimation_main}), and the last equation follows from the definition of Bregman divergence: 
\[
D_\phi(u\Vert v)=\phi(u)-\phi(v)-\langle\nabla \phi(v),u-v\rangle = u \ln \frac{u}{v}-u+v.
\]
Therefore, Eq.~\eqref{ineq:convergence_start} can be rewritten as
\begin{equation*}
    \KL{q^\gets_{t_w}}{\hat{q}_{t_w}} \le \KL{q^\gets_{t_{w-1}}}{\hat{q}_{t_{w-1}}} + \int_{t_{w-1}}^{t_w} \E_{\rvy\sim q^\to_{T-t}}\left[\sum_{\vy^\prime\not=\rvy} R^\to(\rvy, \vy^\prime)\cdot \Breg{v_{t, \rvy}(\vy^\prime)}{\tilde{v}_{t, \rvy}(\vy^\prime)} \right]\der t.
\end{equation*}
With a recursive manner, we have
\begin{equation*}
    \begin{aligned}
        &\KL{q^\gets_{T-\delta}}{\hat{q}_{T-\delta}}\le \KL{q^\gets_0}{\hat{q}_0} +  L_{\text{SE}}(\tilde{v}) = \KL{q^\to_T}{\hat{q}_0} + L_{\text{SE}}(\tilde{v}) \le (1+e^{-T})^d - 1 + T \epsilon^2_{\text{score}},
    \end{aligned}
\end{equation*}
where the last inequality follows from Lemma~\ref{lem:fwd_convergence_0} and Assumption~\ref{ass:score_approximation_error}
\begin{equation*}
    \hat{q}_0(\vy) = \tilde{q}_T(\vy) \propto \exp(-T\cdot (d-\numMask{\vy})).
\end{equation*}
If we set
\begin{equation*}
    T\ge\ln(4d/\epsilon^2)\quad \text{and}\quad \epsilon_{\text{score}}\le T^{-1/2}\cdot \epsilon,
\end{equation*}
it has $(1+e^{-T})^d - 1\le \epsilon^2$ and $T \epsilon^2_{\text{score}}\le \epsilon^2$, which means $\KL{q^\gets_{T-\delta}}{\hat{q}_{T-\delta}}\le 2\epsilon^2$.

\paragraph{Bounding $\TVD{q_*}{q^\to _{\delta}}$} We adopt the proof strategy of Theorem~6 in \cite{chen2024convergence}.
Consider the forward process $(X_t)_{t\geq0}$. By the coupling characterization of the total variation distance, we have
\[
\TVD{q_*}{q^\to_{\delta}}\coloneqq\inf_{\gamma \in \Gamma(q_*,q^\to_{\delta})} \Pr_{(u,v)\sim \gamma} [u \neq v] \leq \Pr(\rvy\neq \rvy^\prime),
\]
where $\Gamma(q_*,q^\to_{\delta})$ is the set of all couplings of $(q_*,q^\to_{\delta})$, and the inequality holds because $(\rvy, \rvy^\prime)$ gives a coupling of $(q_*,q^\to_{\delta})$.
Without loss of generality, we suppose $q^\to_0(\vy)>0$ for all $\numMask{\vy}=0$, then,
combining the transition kernel given Lemma~\ref{lem:fwd_trans_ker} and Assumption~\ref{ass:mask_init}, we have
\begin{equation*}
    \begin{aligned}
        \Pr(\rvy = \rvy^\prime)=\sum_{\vy\in\gY,\numMask{\vy}=0} q^\to_0[\vy]\cdot q^\to_{\delta|0}(\vy|\vy) = \sum_{\vy\in\gY, \numMask{\vy}=0} q_0^\to(\vy)\cdot  e^{-\delta d} = e^{-\delta d}.
    \end{aligned}
\end{equation*}
Thus, by choosing $\delta\le \epsilon/d$, we have
\begin{equation}
    \label{ineq:early_stopping_tv}
    \delta\le d^{-1}\epsilon\le d^{-1}\cdot \ln\left(\frac{1}{1-\epsilon}\right)\quad \Rightarrow\quad e^{\delta d}\le \frac{1}{1-\epsilon}\quad \Rightarrow\quad  \TVD{q_*}{q^\to_{\delta}}\le 1-e^{-\delta d}\le \epsilon.
\end{equation}
Finally, we have
\begin{equation*}
    \TVD{q_0^\to}{\hat{q}^\gets_{T-\delta}} \le \TVD{q_0^\to}{q_\delta^\to} + \TVD{q_{T-\delta}^\gets}{\hat{q}_{T-\delta}} \le \epsilon + \sqrt{\frac{1}{2}\KL{q^\gets_{T-\delta}}{\hat{q}_{T-\delta}}}\le 2\epsilon.
\end{equation*}
Hence the proof is completed.
\end{proof}

\subsection{The complexity of Alg.~\ref{alg:uni_inf}}
\label{app_sec:prof_complexity_reverse}

\begin{theorem}[The complexity of Alg.~\ref{alg:uni_inf}]
    \label{thm:mask_unif_complexity}
    Suppose Assumption~\ref{ass:score_approximation_error} and~\ref{ass:mask_init} hold, following from the settings shown in Theorem~\ref{thm:convergence_unif_reverse}, if we implement Alg.~\ref{alg:uni_inf} with
    \begin{equation*}
        t_{w} - t_{w-1} = \eta \quad \text{where} \quad w\in\{1,2,\ldots W\},\quad W = (T-\delta)/\eta,\quad \eta = \epsilon/2d,\quad  \text{and}\quad \epsilon<1
    \end{equation*}
    the expectation of iteration/score estimation complexity of Alg.~\ref{alg:uni_inf} will be upper bounded by
    \begin{equation*}
        2K(d-\epsilon^2/4) + 12Kd\ln d
    \end{equation*}
    to achieve $\TVD{q_*}{\hat{q}}\le 2\epsilon$ where $\hat{p}$ denotes the underlying distribution of generated samples.
\end{theorem}
\begin{proof}
    We denote $\{\hat{\rvy}_t\}_{t=0}^{T-\delta}$ to present the reverse process.
    For a specific trajectory, e.g., $\{\hat{\rvy}_t\}_{t=0}^{T-\delta} = \{\hat{\vy}\}_{t=0}^{T-\delta}$, the total expected iteration number will be equivalent to the summation of Poisson expectations of $W$ segments, i.e.,
    \begin{equation*}
        \sum_{i=1}^{W} \beta_{t_w}(\hat{\vy}_{t_{w-1}})\cdot (t_w - t_{w-1}) = \frac{K\cdot \numMask{\hat{\vy}_{t_{w-1}}}}{e^{T-t_w} -1}\cdot (t_{w} - t_{w-1}),
    \end{equation*}
    which means the expected iteration number of the reverse process can be written as
    \begin{equation}
        \label{eq:exp_iter_num}
        \begin{aligned}
            \E\left[\sum_{w=1}^{W} \beta_{t_w}(\hat{\rvy}_{t_{w-1}})\cdot (t_w - t_{w-1})\right] = \sum_{w=1}^W \E[\mathrm{numK}(\hat{\rvy}_{t_{w-1}})]\cdot  \frac{K}{e^{(T-t_w)}-1}\cdot (t_w - t_{w-1}).
        \end{aligned}
    \end{equation}
    Although $\E[\mathrm{numK}(\hat{\rvy}_{t_{w-1}})]$ is respect to the practical distribution $\hat{\rvy}_{t_{w-1}}\sim \hat{q}_{t_{w-1}}$, we can approximate it by the forward marginal distribution, i.e.,
    \begin{equation*}
        \E[\mathrm{numK}(\rvy^\gets_{t_{w-1}})] = \E[\mathrm{numK}(\rvy^\to_{T - t_{w-1}})]\quad \text{where}\quad \rvy^\gets_{t_{w-1}}\sim q^\gets_{t_{w-1}}\  \text{and}\ \rvy^\to_{T-t_{w-1}}\sim q^\to_{t_{w-1}}
    \end{equation*}
    Specifically, with Assumption~\ref{ass:mask_init}, we have $\E[\numMask{\rvy_0^\to}]=0$.
    Under this condition, the transition kernel becomes
    \begin{equation*}
        \begin{aligned}
            q^\to_{t|0}(\vy|\vy^\prime) =\prod_{i=1}^{d} & \left[ \underbrace{\left(1- \vone_{(K,K)}(\vy_i, \vy_i^\prime)\right)\cdot \vone_{0}(\vy_i-\vy_i^\prime)\cdot e^{-t}}_{\text{remain non-mask token}} \right.\\
            &\left. + \underbrace{\left(1- \vone_{(K,K)}(\vy_i, \vy_i^\prime)\right)\cdot \vone_{K}(\vy_i)\cdot (1-e^{-t}) }_{\text{turn into mask token}}\right].
        \end{aligned}
    \end{equation*}
    due to Lemma~\ref{lem:fwd_trans_ker}.
    Let $\Pr[\numMask{\rvy^\to_{t}} = k]$ be the probability that exactly $k$ out of the $d$ coordinates are mask tokens ($K$) at time $t$.
    Because each of the $d$ coordinates evolves independently (and identically, each with probability $1-e^{-t}$ of being the mask token at time $t$), we get a standard Binomial random variable:
    \begin{itemize}
        \item Each coordinate is $K$ with probability $1-e^{-t}$.
        \item Each coordinate is non-$K$ with probability $e^{-t}$.
    \end{itemize}
    Hence, we have
    \begin{equation*}
        \Pr[\numMask{\rvy^\to_{t}} = k] = C_d^k\cdot (1-e^{-t})^k\cdot (e^{-t})^{d-k}\quad \text{and}\quad \E[\numMask{\rvy^\to_{t}} = k] = d\cdot (1-e^{-t}).
    \end{equation*}
    Then, for any $w$, we have $\E[\mathrm{numK}(\rvy^\to_{T - t_{w-1}})] = d\cdot (1-e^{-(T-t_{w-1})})$.
    Under the settings shown in Theorem~\ref{thm:convergence_unif_reverse}, we have
    \begin{equation*}
        \TVD{q^\to_{T-t_{w-1}}}{\hat{q}_{t_{w-1}}} \le \TVD{q^\to_{T-t_{W}}}{\hat{q}_{t_{W}}} \le 2\epsilon,
    \end{equation*}
    which implies
    \begin{equation*}
        \left|\E[\mathrm{numK}(\hat{\rvy}_{t_{w-1}})] - \E[\mathrm{numK}({\rvy}^\gets_{t_{w-1}})] \right| \le d\cdot \TVD{q^\to_{T-t_{w-1}}}{\hat{q}_{t_{w-1}}} \le 2d\epsilon.
    \end{equation*}
    Then, we have
    \begin{equation}
        \label{ineq:upb_exp_numK_haty}
        \begin{aligned}
            & \E[\mathrm{numK}(\hat{\rvy}_{t_{w-1}})]\le \E[\mathrm{numK}({\rvy}^\gets_{t_{w-1}})] + 2d\epsilon = \E[\mathrm{numK}(\rvy^\to_{T - t_{w-1}})]+2d\epsilon\\
            & = d\cdot (1-e^{-(T-t_{w-1})}) + 2d\epsilon.
        \end{aligned}
    \end{equation}
    Plugging Eq.~\eqref{ineq:upb_exp_numK_haty} into Eq.~\eqref{eq:exp_iter_num}, we have
    \begin{equation}
        \label{eq:exp_iter_num_mid}
        \begin{aligned}
            &\E\left[\sum_{w=1}^{W} \beta_{t_w}(\hat{\rvy}_{t_{w-1}})\cdot (t_w - t_{w-1})\right]\\
            & \le  \sum_{w=1}^W d\cdot (1-e^{-(T-t_{w-1})}) \frac{K}{e^{(T-t_w)}-1}\cdot (t_w - t_{w-1})\\
            & \quad + \sum_{w=1}^W 2d\epsilon \cdot \frac{K}{e^{(T-t_w)}-1}\cdot (t_w - t_{w-1}) \\
            & = \underbrace{Kd\cdot \sum_{w=1}^W e^{-(T-t_{w})}\cdot (t_{w}-t_{w-1})\cdot \frac{1-e^{-(T-t_{w-1})}}{1-e^{-(T-t_w)}}}_{\text{Term 1}}\\
            & \quad + \underbrace{2Kd\epsilon\cdot\sum_{w=1}^W \left(e^{T-t_w} - 1\right)^{-1}\cdot (t_w - t_{w-1})}_{\text{Term 2}}
        \end{aligned}     
    \end{equation}
    Then, we suppose the segments share the same length $\eta$, i.e.,
    \begin{equation*}
        t_{w} - t_{w-1} = \eta \quad \text{where} \quad w\in\{1,2,\ldots W\},\quad W = (T-\delta)/\eta,\quad  \text{and}\quad \eta = \epsilon/2d.
    \end{equation*}
    Under these conditions, we have
    \begin{equation}
        \label{ineq:domi_fac_upb}
        \begin{aligned}
            & \eta \le \frac{\delta}{2} \le \ln(\frac{1}{2} + \frac{e^{\delta}}{2}) \quad \Rightarrow\quad e^{\eta}\le \frac{e^{\delta}}{2} + \frac{1}{2} \le \frac{e^{(T-t_{w-1})}}{2} + \frac{1}{2}\quad \forall\ w\in\{1,\ldots, W\}\\
            & \Rightarrow\quad  e^{\eta}\le \frac{1+e^{-(T-t_{w-1})}}{2e^{-(T-t_{w-1})}} \quad \Rightarrow\quad 2\cdot e^{-(T-t_{w-1}-\eta)} \le 1 + e^{-(T-t_{w-1})}\\
            & \Rightarrow\quad 1-e^{-(T-t_{w-1})}\le 2-2e^{-(T-t_{w-1}-\eta)}\quad \Rightarrow \quad \frac{1-e^{-(T-t_{w-1})}}{1-e^{-(T-t_w)}}\le 2.
        \end{aligned}
    \end{equation}
    Plugging these results into Term 1 of  Eq.~\eqref{eq:exp_iter_num_mid}, we have
    \begin{equation}
        \label{eq:exp_iter_num_fin}
        \begin{aligned}
            &\text{Term 1} =   2Kd\cdot \sum_{w=1}^W e^{-(T-t_w)} \cdot \eta = 2Kd\cdot \sum_{w=1}^W e^{-(T-w\eta)} \cdot \eta\\
            & =2Kd\cdot \eta\cdot e^{-T}\cdot \frac{e^{(W+1)\eta}-e^{\eta}}{e^{\eta} - 1 } \le 2Kd\cdot e^{\eta}\cdot \left(e^{-\delta} - e^{-T}\right) \le 2Kd\cdot (1-e^{-T})\\
            & = 2Kd\cdot\left(1-\frac{\epsilon^2}{4d}\right).
        \end{aligned}     
    \end{equation}
    Moreover, we have
    \begin{equation*}
        \begin{aligned}
            \frac{e^{T-t_{w-1}}-1}{e^{T-t_w}-1} = \frac{e^{T-t_{w-1}}}{e^{T-t_w}}\cdot \frac{1-e^{-(T-t_{w-1})}}{1-e^{-(T-t_w)}} \le e^{\eta}\cdot 2 \le 2e,
        \end{aligned}
    \end{equation*}
    where the first inequality follows from Eq.~\eqref{ineq:domi_fac_upb} and the last inequality is established when $\eta\le 1$.
    Then, Term 2 of Eq.~\eqref{eq:exp_iter_num_mid} can be upper bounded as
    \begin{equation}
        \label{ineq:term2_upb_mid}
        \begin{aligned}
            & \text{Term 2} =  2Kd\epsilon \cdot \sum_{w=1}^W  \frac{\eta}{e^{T-t_w} - 1}\le 4e\cdot Kd\epsilon \cdot \sum_{w=1}^W \frac{\eta}{e^{T-t_{w-1}}-1} \le 4e\cdot Kd\epsilon \cdot \sum_{w=1}^W \frac{\eta}{T-t_{w-1}}\\
            & \le  4e\cdot dK\epsilon \cdot \int_0^{T-\delta} \frac{1}{T-t} \der t = 4e\cdot dK\epsilon \cdot \ln \frac{T}{\delta} \le 4e\cdot dK\epsilon \cdot \ln \frac{4d^2}{\epsilon^3} \le 12e\cdot Kd\ln d\cdot \epsilon\ln\frac{1}{\epsilon}
        \end{aligned}
    \end{equation}
    where the last inequality follows from
    \begin{equation*}
        4\le d \quad \text{and}\quad \ln \frac{d^3}{\epsilon^3} = 3\ln\frac{d}{\epsilon}\le 3\ln d \ln \frac{1}{\epsilon}
    \end{equation*}
    without loss of generality. 
    Moreover, when $\epsilon<1$, we have 
    \begin{equation*}
        \epsilon \ln \frac{1}{\epsilon}  \le e^{-1},
    \end{equation*}
    which follows from the monotonicity of the function $x\ln x$.
    Under this condition, the RHS of Eq.~\eqref{ineq:term2_upb_mid} has the following bound
    \begin{equation}
        \label{ineq:term2_upb_fin}
        \begin{aligned}
            \text{Term 2}  \le 12\cdot Kd\ln d.
        \end{aligned}
    \end{equation}
    Finally, plugging Eq.~\eqref{eq:exp_iter_num_fin} and Eq.~\eqref{ineq:term2_upb_fin} into Eq.~\eqref{eq:exp_iter_num_mid}, the expected calls of discrete scores will be 
    \begin{equation*}
        \E\left[\sum_{w=1}^{W} \beta_{t_w}(\hat{\rvy}_{t_{w-1}})\cdot (t_w - t_{w-1})\right]\le 2K(d-\epsilon^2) + 12Kd\ln d. 
    \end{equation*}
    Hence, the proof is completed.
\end{proof}

\subsection{The Proof of Corollary~\ref{cor:conv_wo_A2}}
\label{app_sec:prof_cor_conv_wo_A2}
\begin{proof}
    Similar to the proof shown in Theorem~\ref{thm:mask_unif_complexity}, the expected iteration number of the reverse process can be written as
    \begin{equation*}
        \begin{aligned}
            \E\left[\sum_{w=1}^{W} \beta_{t_w}(\hat{\rvy}_{t_{w-1}})\cdot (t_w - t_{w-1})\right] = \sum_{w=1}^W \E[\mathrm{numK}(\hat{\rvy}_{t_{w-1}})]\cdot  \frac{K}{e^{(T-t_w)}-1}\cdot (t_w - t_{w-1}).
        \end{aligned}
    \end{equation*} 
    Although $\E[\mathrm{numK}(\hat{\rvy}_{t_{w-1}})]$ is respect to the practical distribution $\hat{\rvy}_{t_{w-1}}\sim \hat{q}_{t_{w-1}}$, we can approximate it by the forward marginal distribution, i.e.,
    \begin{equation*}
        \E[\mathrm{numK}(\rvy^\gets_{t_{w-1}})] = \E[\mathrm{numK}(\rvy^\to_{T - t_{w-1}})]\quad \text{where}\quad \rvy^\gets_{t_{w-1}}\sim q^\gets_{t_{w-1}}\  \text{and}\ \rvy^\to_{T-t_{w-1}}\sim q^\to_{t_{w-1}}.
    \end{equation*}
    Let $\Pr[\numMask{\rvy^\to_{t}} = k]$ be the probability that exactly $k$ out of the $d$ coordinates are mask tokens ($K$) at time $t$ presented as
    \begin{equation*}
        \mathbb{P}[\numMask{\rvy^\to_t} = k] = \sum_{i=0}^k \mathbb{P}[\numMask{\rvy^\to_t} = k|\numMask{\rvy^\to_0} = i]\cdot \mathrm{Pr}[\numMask{\rvy^\to_0} = i].
    \end{equation*}
    Because each of the $d$ coordinates evolves independently (and identically, each with probability $1-e^{-t}$ of being the mask token at time $t$), we get a standard Binomial random variable:
    \begin{itemize}
        \item Each coordinate is $K$ with probability $1-e^{-t}$.
        \item Each coordinate is non-$K$ with probability $e^{-t}$.
    \end{itemize}
    Hence, we have
    \begin{equation*}
        \begin{aligned}
             \mathbb{P}[\mathrm{numK}(\rvy^\to_t) = k|\numMask{\rvy^\to_0} = i] =  \underbrace{C_{d-i}^{k-i}}_{\text{unmask}\to \text{mask count}} \cdot \underbrace{(1-e^{-t})^{k-i}}_{\text{prob of mask transition}}\cdot \underbrace{(e^{-t})^{d-k}}_{\text{prob of unmask kept}}.
        \end{aligned}
    \end{equation*}
    Under this condition, the expected number of MASK token at forward time $t$ will become
    \begin{equation}
        \label{eq:exp_mask_num_general_init}
        \begin{aligned}
            & \mathbb{E}[\numMask{\rvy^\to_t}] = \sum_{k=0}^d k\cdot \mathbb{P}[\numMask{\rvy^\to_t} = k]\\
            & = \sum_{k=0}^d \sum_{i=0}^k C_{d-i}^{k-i}\cdot (1-e^{-t})^{k-i}\cdot (e^{-t})^{d-k}\cdot \mathbb{P}[\numMask{\rvy^\to_0}= i]\\
            & = \sum_{i=0}^d \mathbb{P}[\numMask{\rvy^\to_0} = i]\cdot  \underbrace{\sum_{k=i}^d C_{d-i}^{k-i} \cdot (1-e^{-t})^{k-i}\cdot (e^{-t})^{d-k}}_{\text{Term 1}}.
            \end{aligned}
    \end{equation}
    For Term 1 in Eq.~\eqref{eq:exp_mask_num_general_init}, suppose $j=k-i$, we have
    \begin{equation*}
        \begin{aligned}
            & \text{Term 1} = \sum_{j=0}^{d-i} (j+i)\cdot C_{d-i}^j \cdot (1-e^{-t})^{j}\cdot (e^{-t})^{(d-i-j)}\\
            & = \sum_{j=0}^{d-i} j\cdot C_{d-i}^j \cdot (1-e^{-t})^{j}\cdot (e^{-t})^{(d-i-j)} + i\cdot \sum_{j=0}^{d-i} C_{d-i}^j \cdot (1-e^{-t})^{j}\cdot (e^{-t})^{(d-i-j)}\\
            & = \sum_{j=0}^{d-i} j\cdot C_{d-i}^j \cdot (1-e^{-t})^{j}\cdot (e^{-t})^{(d-i-j)} + i\cdot (1-e^{-t} + e^{-t})^{d-i}= d-(d-i)e^{-t}
        \end{aligned}
    \end{equation*}
    where the last equation follows from the expectation of binomial distributions.
    Then, Eq.~\eqref{eq:exp_mask_num_general_init} can be written as 
    \begin{equation*}
        \begin{aligned}
            \mathbb{E}[\numMask{\rvy^\to_t}]  = \sum_{i=0}^d \mathbb{P}[\numMask{\rvy^\to_0} = i]\cdot \left(d-d\cdot e^{-t} + i\cdot e^{-t}\right)\\
            = d\cdot \left(1-(1-\E[\numMask{\rvy^\to_0}]/d)\cdot e^{-t}\right).
        \end{aligned}
    \end{equation*}
    Without loss of generality, we suppose
    \begin{equation*}
        r_0 \coloneqq 1 - \E[\numMask{\rvy^\to_0}]/d >0.
    \end{equation*}
    Then, following from Eq.~\eqref{ineq:upb_exp_numK_haty}, we have
    \begin{equation*}
        \E[\mathrm{numK}(\hat{\rvy}_{t_{w-1}})]\le  d\cdot (1-r_0\cdot e^{-(T-t_{w-1})}) + 2d\epsilon,
    \end{equation*}
    and
    \begin{equation}
        \label{ineq:comp_upb_cor_init}
        \begin{aligned}
            \E\left[\sum_{w=1}^{W} \beta_{t_w}(\hat{\rvy}_{t_{w-1}})\cdot (t_w - t_{w-1})\right] \le & \underbrace{Kd\cdot \sum_{w=1}^W e^{-(T-t_{w})}\cdot (t_{w}-t_{w-1})\cdot \frac{1-r_0\cdot e^{-(T-t_{w-1})}}{1-e^{-(T-t_w)}}}_{\text{Term 1}}\\
            & \quad + \underbrace{2Kd\epsilon\cdot\sum_{w=1}^W \left(e^{T-t_w} - 1\right)^{-1}\cdot (t_w - t_{w-1})}_{\text{Term 2}}
        \end{aligned}     
    \end{equation}
    Here the second term can be upper bounded as 
    \begin{equation*}
            \text{Term 2}  \le 12\cdot Kd\ln d
    \end{equation*}
    by choosing the mixing time $T$ and early stopping time $\delta$ as Theorem~\ref{thm:convergence_unif_reverse}, which follows from Eq.~\eqref{ineq:term2_upb_mid}.
    
    For Term 1 of Eq.~\eqref{ineq:comp_upb_cor_init}, we will discuss it in categories.
    Suppose the expected number of mask token satisfies
    \begin{equation*}
        \E[\numMask{\rvy^\to_0}] \le C_0 \cdot\epsilon\quad \Leftrightarrow\quad r_0\ge 1-C_0 \cdot\epsilon/d,
    \end{equation*}
    and the segments share the same length $\eta$, i.e.,
    \begin{equation*}
        t_{w} - t_{w-1} = \eta \quad \text{where} \quad w\in\{1,2,\ldots W\},\quad W = (T-\delta)/\eta,\quad  \text{and}\quad \eta = \epsilon/2d, 
    \end{equation*}
    following from Eq.~\eqref{ineq:domi_fac_upb}, we have
    \begin{equation*}
        \eta\le \frac{\delta}{2} \Rightarrow \quad \frac{1-e^{-(T-t_{w-1})}}{1-e^{-(T-t_w)}}\le 2.
    \end{equation*}
    Combining with the following fact, i.e.,
    \begin{equation*}
        \begin{aligned}
            &\frac{(1-r_0)\cdot e^{-(T-t_{w-1})}}{1-e^{-(T-t_w)}} = \frac{1-r_0}{e^{(T-t_{w-1})}-e^{\eta}} = (1-r_0)\cdot e^{-\eta}\cdot \left(e^{T-t_{w-1}-\eta} - 1\right)^{-1}\\
            & \le (1-r_0)\cdot \delta^{-1} = C_0,
        \end{aligned}
    \end{equation*}
    Eq.~\eqref{eq:exp_iter_num_fin} demonstrates that
    \begin{equation*}
        \begin{aligned}
            \text{Term 1} \le (2+C_0)\cdot Kd\cdot\left(1-\frac{\epsilon^2}{4d}\right).
        \end{aligned}
    \end{equation*}
    On the other hand, we have
    \begin{equation*}
        \begin{aligned}
            &\text{Term 1} \le Kd\cdot \sum_{w=1}^W \frac{\eta}{e^{T-t_w}-1} \le Kd\sum_{w=1}^W \frac{\eta}{T-t_w} \le 1.5Kd\sum_{w=1}^W \frac{\eta}{T-t_{w-1}}\\
            & \lesssim 1.5Kd\cdot\int_{\delta}^1 t^{-1}\der t\le 1.5Kd\ln(1/\delta) = 1.5Kd\ln(d/\epsilon),
        \end{aligned}
    \end{equation*}
    where the forth inequality follows from the choice of $\eta$, i.e.,
    \begin{equation*}
        \eta\le \delta/2 \quad \Rightarrow\quad (T-t_{w-1}) - (T-t_w) =\eta \le \frac{\delta}{2}\le \frac{T-t_w}{2}.
    \end{equation*}
    Hence, the total complexity will be
    \begin{equation*}
        \min\left\{O(Kd\ln(d/\epsilon)), O\left(Kd\cdot \frac{\E[\numMask{\rvy^\to_0}]}{\epsilon}\right)\right\} + O(Kd\ln d).
    \end{equation*}
    Hence, the proof is completed.
\end{proof}

\begin{corollary}
    \label{cor:convergence_unif_reverse}
    Suppose Assumption~\ref{ass:score_approximation_error} and~\ref{ass:mask_init} hold, if 
    Alg.~\ref{alg:uni_inf} has 
    \begin{equation*}
        t_0 = 0,\quad t_W= T-\delta,\quad \text{and}\quad \epsilon_{\text{score}}\le T^{-1/2}\cdot \epsilon \quad \text{where}\quad T = \ln(4d/\epsilon^2) \quad \text{and}\quad \delta\le d^{-1}\epsilon,
    \end{equation*} 
    and draw initial $\overline{\rvy}_0\sim \delta_{[\idxK,\ldots, \idxK]}(\cdot)$,
    the TV distance between the target discrete distribution $q_*$ and the underlying distribution of the output particle $\overline{q}_{T-\delta}$ will satisfy $\TVD{q_*}{\overline{q}_{T-\delta}}\le 2.5\epsilon$.
\end{corollary}
\begin{proof}
    We consider a stochastic process $\{\overline{\rvy}_t\}_{t=0}^{T-\delta}$ which satisfies $\overline{\rvy}_t\sim \overline{q}_t$. The initial distribution is $\overline{\rvy}_0 \sim \overline{q}_0 = \delta_{[\idxK,\idxK,\ldots, \idxK]}(\vy)$.
    Suppose the joint and conditional distribution are
    \begin{equation*}
        (\overline{\rvy}_{t^\prime}, \overline{\rvy}_t) \sim \overline{q}_{t^\prime, t}\quad \text{and}\quad \overline{q}_{t|t^\prime}(\overline{\vy}_{t}||\overline{\vy}_{t^\prime}) = \overline{q}_{t,t^\prime}(\overline{\vy}_t, \overline{\vy}_{t^\prime}) / \overline{q}_{t^\prime}(\overline{\vy}_{t^\prime})\quad \text{where}\quad t>t^\prime.
    \end{equation*}
    Specifically, we suppose the random variables $\{\overline{\rvy}_t\}_{t=0}^{T-\delta}$ share the same transition as that in $\{\hat{\rvy}_t\}_{t=0}^{T-\delta}$ shown in Theorem~\ref{thm:convergence_unif_reverse}, which means $\overline{q}_{t|t^\prime} = \hat{q}_{t|t^\prime}$ for any $t>t^\prime$, which implies $\{\overline{\rvy}_t\}_{t=0}^{T-\delta}$ can be implemented by
    \begin{enumerate}
        \item Initialize the particles as $\overline{\rvy}_0 \sim \overline{q}_0 = \delta_{[\idxK,\idxK,\ldots, \idxK]}(\vy)$
        \item Update $\{\overline{\rvy}_t\}_{t>0}^{T-\delta}$ with Alg.~\ref{alg:uni_inf}
    \end{enumerate}
    Then, due to the chain rule of TV distance, i.e., Lemma~\ref{lem:tv_chain_rule}, we have
    \begin{equation}
        \label{ineq:tv_gap_hat_overline}
        \begin{aligned}
            & \TVD{\hat{q}_{T-\delta}}{\overline{q}_{T-\delta}} \le \TVD{\hat{q}_{T-\delta, 0}}{\overline{q}_{T-\delta, 0}}\\
            & \le \TVD{\hat{q}_0}{\overline{q}_0} + \E_{\hat{\rvy}_0\sim \hat{q}_0}\left[\TVD{\hat{q}_{T-\delta|0}}{\overline{q}_{T-\delta|0}}\right] = \TVD{\hat{q}_0}{\overline{q}_0}.
        \end{aligned}
    \end{equation}
    Since $\overline{q}_0$ is the mask token dirac measure, we have
    \begin{equation*}
        \TVD{\hat{q}_0}{\overline{q}_0} = 1-\hat{q}_0([\idxK, \idxK,\ldots, \idxK]).
    \end{equation*}
    According to the proof of Lemma~\ref{lem:fwd_convergence_0}, we can easily find that
    \begin{equation*}
        \hat{q}_0([\idxK,\ldots, \idxK]) = \frac{1}{(1+e^{-T})^d}.
    \end{equation*}
    By requiring $T\ge \ln(4d/\epsilon)$ and $\epsilon\le 1$, we have
    \begin{equation*}
        \begin{aligned}
            & T\ge \ln(4d/\epsilon) \quad \Rightarrow t\ge \ln(d/\ln(1+\epsilon/2))\quad \Leftrightarrow \quad d\cdot e^{-T}\le \ln(1+\epsilon/2)\\
            & \Rightarrow d\ln(1+e^{-T})\le \ln(1+\epsilon/2)\quad \Leftrightarrow \quad (1+e^{-T})^d-1 \le \epsilon/2.
        \end{aligned}
    \end{equation*}
    That means
    \begin{equation*}
        \TVD{\hat{q}_0}{\overline{q}_0} = 1- \frac{1}{(1+e^{-T})^d} \le 1- \frac{1}{1+\epsilon/2} \le \epsilon/2.
    \end{equation*}
    Plugging this inequality into Eq.~\ref{ineq:tv_gap_hat_overline}, we have
    $\TVD{\hat{q}_{T-\delta}}{\overline{q}_{T-\delta}}  \le \epsilon/2$.
    Then combining it with Theorem~\ref{thm:convergence_unif_reverse}, i.e., $\TVD{q_*, \hat{q}_{T-\delta}}\le 2\epsilon$, we have $\TVD{q_*, \overline{q}_{T-\delta}}\le 2.5\epsilon$.
    Hence, the proof is completed.
\end{proof}

\section{The Convergence under Time-Independent Score Parameterization}

In the following, we will prove that the distribution generated by first hitting sampling~\citep{zheng2024masked} approaches to the target data distribution $p_*$ in TV distance.
The core step is to introduce our AATU as the reference probability path. 

We starts from some additional notations. Specifically, suppose following two elements $\vy,\hat{\vy}\in \gY = \{1,2,\ldots, \idxK\}^d$ satisfying 
\begin{equation*}
    \vy = [\vy_1,\ldots, \vy_i,\ldots, \vy_d] \qquad  \hat{\vy} = [\vy_1, \ldots, \hat{\vy}_i, \ldots,  \vy_d],
\end{equation*}
which means the Hamming distance between $\vy$ and $\hat{\vy}$ is $1$ and they are only different at $i$--th coordinate. 
Suppose $\vy_i = \idxK$ and $\hat{\vy}_i\not=\idxK$, then we can define the conditional distribution at specific coordinate, e.g., $i$, given unmask tokens $\vy_{\gK^c(\vy)}$
\begin{equation*}
    q_{0,i}(\hat{\vy}_i|\vy_{\gK^c(\vy)}) = \frac{\sum_{\tilde{\vy}\in \gY^+, \tilde{\vy}_{\gK^c(\hat{\vy})} = \hat{\vy}_{\gK^c(\hat{\vy})}} q_0(\tilde{\vy})}{\sum_{\tilde{\vy}\in \gY^+, \tilde{\vy}_{\gK^c(\vy)} = \vy_{\gK^c(\vy)}} q_0(\tilde{\vy})}
\end{equation*}

\paragraph{Bridge the discrete score estimation error and the pretrained masked diffusion models in Alg~\ref{alg:dlm_imple}.} We need to note that the output of pretrained masked diffusion model satisfies
\begin{equation*}
    p_{\theta,i}(\cdot|\vy, \tilde{\tau}_{n+1}) =  p_{\theta,i}(\cdot|\vy) \approx q_{0,i}(\hat{\vy}_i|\vy_{\gK^c(\vy)}),
\end{equation*}
where first equation comes from the time-independent parameterization, and the second approximation comes from the training objective, i.e.,
\begin{equation}
    \label{eq:fhs_training}
    \gL_{\vw}^d(\vy_0) = \sum_{i=1}^d \vw_i \cdot \E_{\mathbb{P}_{\rvy}[\numMask{\rvy}=i|\vy_0]} \left[\sum_{\rvy_i=\idxK}-\log p_{\theta,i}(\vy_{0,i}|\rvy)\right].
\end{equation}
With proper settings on $\vw$ and the change of summation order, the above training loss of Alg.~\ref{alg:dlm_imple} will be equivalent to the $\lambda$--DCE loss shown in~\citet{ou2024your}, i.e.,
\begin{equation*}
    \gL_{\lambda-\mathrm{DCE}}(\vy_0) = \E_{\lambda\sim \mathrm{Uniform}(0,1)}\frac{1}{\lambda}\cdot \E_{\rvy_\lambda\sim q^\to_{\lambda|0}(\cdot|\vy_0)}\left[\sum_{\vy_{\lambda,i}=\idxK}-\log p_{\theta}(\vy_{0,i}|\rvy)\right].
\end{equation*}
Then, following from Appendix C.1 and Appendix C.2 in~\cite{ou2024your}, by choosing $\lambda(t) = 1-e^{-t}$, with change of variable, the $\lambda$--DCE loss will be equivalent to the denoising score entropy loss, i.e.,
\begin{equation*}
    \begin{aligned}
        \gL_{\mathrm{DSE}}(\vy_0) = & \int_0^T \E_{\rvy_t\sim q^\to_{t|0}(\cdot|\vy_0)}\left[\sum_{\vy^\prime\not= \rvy_t}R^\to(\rvy_t,\vy^\prime)\cdot \left(\frac{e^{-t}}{1-e^{-t}}\cdot p_\theta(\vy^\prime_{\DfId{\vy^\prime}{\rvy_t}}|\rvy_t) \right.\right.\\
        &\quad \left.\left. - \frac{e^{-t}}{1-e^{-t}}\cdot  \delta_{\vy_{0,\DfId{\rvy_t}{\vy^\prime}}}(\vy^\prime_{\DfId{\rvy_t}{\vy^\prime}})\cdot \log \left(\frac{e^{-t}}{1-e^{-t}}\cdot p_\theta(\vy^\prime_{\DfId{\vy^\prime}{\rvy_t}}|\rvy_t)\right)\right) \right] \der t.
    \end{aligned}
\end{equation*}
Following from Theorem 3.4 in~\citet{lou2024discrete}, we note that DSE and SE share the same minimum, i.e.,
\begin{equation}
    \label{eq:fms_training_equ}
    \begin{aligned}
        & \argmin_{\theta}\  \E_{\rvy_0\sim q_*}[\gL_{\mathrm{DSE}}(\rvy_0)]\\
        & = \argmin_{\theta}\ \int_0^T \E_{\rvy_t\sim q^\to_t}\left[\sum_{\vy^\prime\not= \rvy_t}R^\to(\rvy_t,\vy^\prime)\cdot \left(\frac{e^{-t}}{1-e^{-t}}\cdot p_\theta(\vy^\prime_{\DfId{\vy^\prime}{\rvy_t}}|\rvy_t) \right.\right.\\
        &\qquad \qquad \qquad \left.\left. \frac{q^\to_t(\vy^\prime)}{q_t^\to(\rvy_t)}\cdot \log \left(\frac{e^{-t}}{1-e^{-t}}\cdot p_\theta(\vy^\prime_{\DfId{\vy^\prime}{\rvy_t}}|\rvy_t)\right)  \right)\right] \coloneqq \argmin_\theta\ \gL_{\mathrm{SE}}(\theta)
    \end{aligned}
\end{equation}
By supposing 
\begin{equation*}
    \tilde{v}_{t,\vy_t}(\vy^\prime) \coloneqq  \frac{e^{-t}}{1-e^{-t}}\cdot p_\theta(\vy^\prime_{\DfId{\vy^\prime}{\vy_t}}|\vy_t) \quad \text{where}\quad \mathrm{Ham}(\vy^\prime,\vy_t)=1\quad \text{and}\quad \vy^\prime_{\DfId{\vy^\prime}{\vy_t}}\not=\idxK,
\end{equation*}
we know the Eq.~\ref{eq:fms_training_equ} exactly matches Eq.~\ref{eq:score_estimation_main}.

Therefore, optimizing Eq.~\ref{eq:fhs_training} in Alg.~\ref{alg:dlm_imple} is equivalent to parameterizing the discrete score as
\begin{equation*}
    \begin{aligned}
        &\frac{q^\to_t(\vy^\prime)}{q^\to_t(\vy_t)} = \frac{e^{-t}}{1-e^{-t}}\cdot q_{0,\DfId{\vy^\prime}{\vy_t}}(\vy^\prime_{\DfId{\vy^\prime}{\vy_t}}|\vy_{t,\gK(\vy_t)}) = v_{t,\vy_t}(\vy^\prime)\\
        & \approx \tilde{v}_{t,\vy_t}(\vy^\prime) \coloneqq  \frac{e^{-t}}{1-e^{-t}}\cdot p_\theta(\vy^\prime_{\DfId{\vy^\prime}{\vy_t}}|\vy_t),
    \end{aligned}
\end{equation*}
and optimize Eq.~\ref{eq:score_estimation_main}.
Following the analysis paradigm in this paper, we assume Assumption~\ref{ass:score_approximation_error} is also satisfies for this parametrization.

\subsection{The Proof of Theorem~\ref{thm:convergence_fhs_reverse}}
\label{app_sec: convergence_fhs_reverse}
\begin{proof}
    Under this time-independent parameterization, we suppose the trajectory of AATU as $\{\hat{\rvy}_t\}_{t=0}^T$ whose underlying distribution is denoted as $\hat{\rvy}_t \sim \hat{q}_t$.
For Alg.~\ref{alg:dlm_imple}, we consider a sequence of random variables $\{\overline{\rvy}_k\}_{k\in\{0,1,\ldots, d\}}$ where $\overline{\rvy}_k$ denotes the random variables after $(d-k)$-step update of Alg.~\ref{alg:dlm_imple}. We have $\numMask{\overline{\rvy}_k} = k$.
To investigate the TV distance between $\hat{\rvy}_{T-\delta}$ and $\overline{\rvy}_0$, we have
\begin{equation}
    \label{ineq:fhs_tv_init}
    \begin{aligned}
        &\TVD{\hat{q}_{T-\delta}}{\overline{q}_0} = \frac{1}{2}\cdot  \sum_{\vy, \numMask{\vy} = 0} \left| \overline{q}_0(\vy) - \hat{q}_{T-\delta}(\vy)  \right| + \frac{1}{2}\cdot \sum_{\vy, \numMask{\vy}\not=0} \hat{q}_{T-\delta}(\vy)\\
        & = \frac{1}{2}\cdot \sum_{\vy, \numMask{\vy} = 0} \left|\overline{q}_0(\vy) - \hat{q}_{T-\delta}(\vy)  \right| +  \overline{q}_0(\vy) - \hat{q}_{T-\delta}(\vy) \le \sum_{\vy,\numMask{\vy}=0} \left|\hat{q}_{T-\delta}(\vy) - \overline{q}_0(\vy)\right|
    \end{aligned}
\end{equation}
Currently, we define a distribution sequence
\begin{equation*}
    \{p_k\}_{k\in\{0,1,\ldots, d\}}\quad \text{where}\quad p_k(t) = \mathrm{Pr}\left[\text{the $k$-th transition happens at time $t$}\right].
\end{equation*}
Besides, suppose that at the transition time $t$ the particle is $\vy^\prime$, AATU implies the transition from $\vy^\prime$ to $\vy$ follows 
\begin{equation*}
    \mathrm{Pr}[\vy|\text{transition time}=t \text{ and particle is }\vy^\prime] = \hat{R}_t(\vy,\vy^\prime) / \hat{R}_t(\vy^\prime)
\end{equation*}

Under this setting, we have
\begin{equation}
    \label{ineq:matu_time_independent_para_calc}
    \begin{aligned}
        \hat{q}_{T-\delta}(\vy) & = \int_0^{T-\delta} p_{\numMask{\vy}}(t)\cdot \sum_{\vy^\prime, \numMask{\vy^\prime}=\numMask{\vy}-1} \hat{q}_{t}(\vy^\prime)\cdot \mathrm{Pr}[\vy|\text{transition time}=t \text{ and particle is }\vy^\prime] \der t\\
        & = \int_0^{T-\delta} p_{\numMask{\vy}}(t)\cdot \sum_{\vy^\prime, \numMask{\vy^\prime}=\numMask{\vy}-1} \hat{q}_{t}(\vy^\prime)\cdot \frac{\hat{R}_t(\vy,\vy^\prime)}{\hat{R}_t(\vy^\prime)}\der t\\
        & = \int_0^{T-\delta} p_{\numMask{\vy}}(t)\cdot \sum_{\vy^\prime, \numMask{\vy^\prime}=\numMask{\vy}-1} \hat{q}_{t}(\vy^\prime)\cdot \frac{\tilde{R}_t(\vy,\vy^\prime)}{\tilde{R}_t(\vy^\prime)}\der t
    \end{aligned}
\end{equation}
Due to the time-independent paramterization of the discrete score, we have
\begin{equation*}
    \tilde{R}_t(\vy,\vy^\prime) = R^\to(\vy^\prime,\vy) \cdot \tilde{v}_{t,\vy^\prime}(\vy) = R^\to(\vy^\prime,\vy)\cdot \frac{e^{-t}}{1-e^{-t}}\cdot p_\theta(\vy_{\DfId{\vy}{\vy^\prime}}|\vy^\prime),
\end{equation*}
which implies it has
\begin{equation*}
    \frac{\tilde{R}_t(\vy,\vy^\prime)}{\tilde{R}_t(\vy^\prime)} = \frac{p_\theta(\vy_{\DfId{\vy}{\vy^\prime}}|\vy^\prime)}{\sum_{\vy\not=\vy^\prime, \mathrm{Ham}(\vy,\vy^\prime)=1} p_\theta(\vy_{\DfId{\vy}{\vy^\prime}}|\vy^\prime)} = p_\theta(\vy_{\DfId{\vy}{\vy^\prime}}|\vy^\prime).
\end{equation*}
Plugging this equation into Eq.~\ref{ineq:matu_time_independent_para_calc}, we have
\begin{equation*}
    \begin{aligned}
        & \hat{q}_{T-\delta}(\vy) =  \sum_{\vy^\prime, \numMask{\vy^\prime}=\numMask{\vy}-1} p_\theta(\vy_{\DfId{\vy}{\vy^\prime}}|\vy^\prime)\int_0^{T-\delta} p_{\numMask{\vy}}(t)\cdot  \hat{q}_{t}(\vy^\prime) \der t\\
        = & \sum_{\vy^\prime, \numMask{\vy^\prime}=\numMask{\vy}-1} p_\theta(\vy_{\DfId{\vy}{\vy^\prime}}|\vy^\prime) \cdot  \sum_{\vy^{\prime\prime}, \numMask{\vy^{\prime\prime}} = \numMask{\vy^\prime} -1} p_\theta(\vy^\prime_{\DfId{\vy^{\prime}}{\vy^{\prime\prime}}}|\vy^{\prime\prime})\\
        &\cdot \ldots \cdot  \int_{t_{\numMask{\vy}},t_{\numMask{\vy}-1},\ldots, t_1} p_{\numMask{\vy}, \numMask{\vy},\ldots, 1}(t_{\numMask{\vy}-1},\ldots, t_1)\cdot \hat{q}_{t_1}([\idxK,\ldots,\idxK]) \der t_1.
    \end{aligned}
\end{equation*}
Since each absorbing state is denoised exactly once, a particle that undergoes $d$ denoising steps during the reverse process is guaranteed to result in a fully non-absorbing sample at the end of inference.
Therefore, we have
\begin{equation*}
    \begin{aligned}
        &\int_{t_{\numMask{\vy}},t_{\numMask{\vy}-1},\ldots, t_1} p_{\numMask{\vy}, \numMask{\vy},\ldots, 1}(t_{\numMask{\vy}-1},\ldots, t_1)\cdot \hat{q}_{t_1}([\idxK,\ldots,\idxK]) \der t_1\\
        &= \sum_{\vy^\prime, \numMask{\vy^\prime}=0} \hat{q}_{T-\delta}(\vy^\prime).
    \end{aligned}
\end{equation*}
Due to the TV convergence of $\hat{q}_{T-\delta}$ shown in Theorem~\ref{thm:convergence_unif_reverse}, we have $\TVD{q_*}{\hat{q}_{T-\delta}}\le 2\epsilon$, which implies
\begin{equation}
    \label{ineq:sum_non_absorbing_bound}
    \begin{aligned}
        & \TVD{q_*}{\hat{q}_{T-\delta}} = \frac{1}{2}\sum_{\vy^\prime, \numMask{\vy^\prime}=0} (q_*(\vy^\prime) - \hat{q}_{T-\delta}(\vy^\prime)) + \sum_{\vy^\prime, \numMask{\vy^\prime}\not=0} \hat{q}_{T-\delta}(\vy^\prime)\\
        & = \sum_{\vy^\prime, \numMask{\vy^\prime}=0} (q_*(\vy^\prime) - \hat{q}_{T-\delta}(\vy^\prime)) = 1- \sum_{\vy^\prime, \numMask{\vy^\prime}=0} \hat{q}_{T-\delta}(\vy^\prime) \le 2\epsilon.
    \end{aligned}
\end{equation}
All of the above three equations come from Assumption~\ref{ass:mask_init}.

According to the update of Alg.~\ref{alg:dlm_imple}, we can easily find that 
\begin{equation*}
    \begin{aligned}
        \overline{q}_0(\vy) = & \sum_{\vy^\prime, \numMask{\vy^\prime}=\numMask{\vy}-1} p_\theta(\vy_{\DfId{\vy}{\vy^\prime}}|\vy^\prime) \cdot  \sum_{\vy^{\prime\prime}, \numMask{\vy^{\prime\prime}} = \numMask{\vy^\prime} -1} p_\theta(\vy^\prime_{\DfId{\vy^{\prime}}{\vy^{\prime\prime}}}|\vy^{\prime\prime})\\
        & \cdot \ldots\cdot \sum_{\vy^{(1)}, \numMask{\vy^{(1)}} = 1} p_\theta(\vy^{(1)}_{\DfId{\vy^{(1)}}{[\idxK,\ldots,\idxK]}}|[\idxK,\ldots,\idxK])  \underbrace{\overline{q}_d([\idxK,\ldots,\idxK])}_{=1}.
    \end{aligned}
\end{equation*}
Suppose the conditional distribution as
\begin{equation*}
    \begin{aligned}
        \overline{p}_{\theta}(\vy|[\idxK,\ldots,\idxK]) = & \sum_{\vy^\prime, \numMask{\vy^\prime}=\numMask{\vy}-1} p_\theta(\vy_{\DfId{\vy}{\vy^\prime}}|\vy^\prime) \cdot  \sum_{\vy^{\prime\prime}, \numMask{\vy^{\prime\prime}} = \numMask{\vy^\prime} -1} p_\theta(\vy^\prime_{\DfId{\vy^{\prime}}{\vy^{\prime\prime}}}|\vy^{\prime\prime})\\
        & \cdot \ldots\cdot \sum_{\vy^{(1)}, \numMask{\vy^{(1)}} = 1} p_\theta(\vy^{(1)}_{\DfId{\vy^{(1)}}{[\idxK,\ldots,\idxK]}}|[\idxK,\ldots,\idxK]),
    \end{aligned}
\end{equation*}
then we have
\begin{equation}
    \label{ineq:fhs_from_end_to_start}
    \begin{aligned}
        & \overline{q}_0(\vy) - \hat{q}_{T-\delta}(\vy)  \le \overline{p}_{\theta}(\vy|[\idxK,\ldots,\idxK])\cdot \left(\overline{q}_d([\idxK,\ldots,\idxK]) - \hat{q}_0([\idxK,\ldots,\idxK])\right)\\
        & = \overline{p}_{\theta}(\vy|[\idxK,\ldots,\idxK])\cdot \left(1-\sum_{\vy^\prime, \numMask{\vy^\prime}=0} \hat{q}_{T-\delta}(\vy^\prime)\right)
    \end{aligned}
\end{equation}
Combining Eq.~\ref{ineq:fhs_tv_init}, Eq.~\ref{ineq:fhs_from_end_to_start} and Eq.~\ref{ineq:sum_non_absorbing_bound}, we have
\begin{equation*}
    \TVD{\hat{q}_{T-\delta}}{\overline{q}_0} \le 2\epsilon\quad \text{and}\quad \TVD{q_*}{\overline{q}_0} \le \TVD{\hat{q}_{T-\delta}}{\overline{q}_0} + \TVD{q_*}{\hat{q}_{T-\delta}}\le 4\epsilon
\end{equation*}
where last inequality follows from Theorem~\ref{thm:convergence_unif_reverse}.
Hence, the proof is completed.
\end{proof}

\section{Technical Lemmas}

\begin{lemma}[Basic Kronecker product]
    \label{lem:basic_prod_K_prod}
    Supppose the Kronecker product for $n$ matrices defined on $\R^{d\times d}$, i.e.,
    \begin{equation*}
        \overline{\mA}\coloneqq \mA_1 \otimes \mA_2 \otimes \ldots \otimes \mA_n,
    \end{equation*}
    then we have
    \begin{equation*}
        \overline{\mA}_{[a_{1,i}, a_{2,i},\ldots, a_{n,i}],[a_{1,j}, a_{2,j},\ldots, a_{n,j}]} \coloneqq \overline{\mA}_{\sum_{k=1}^n a_{k,i}\cdot d^{n-k}, \sum_{k=1}^n a_{k,j}\cdot d^{n-k}} = \prod_{k=1}^n [\mA_{k}]_{a_{k,i}, a_{k,j}}.
    \end{equation*}
\end{lemma}
\begin{proof}
    This lemma can easily be proved by the definition of Kronecker product.
\end{proof}

\begin{lemma}[Mixed-product property of Kronecker product]
    \label{lem:mix_prod_K_prod}
    Suppose the matrices $\mA, \mB, \mC,\mD \in \R^{d\times d}$, then, the products $\mA\mC$ and $\mB\mD$ are well-defined. We have
    \begin{equation*}
        (\mA\otimes \mB)(\mC\otimes \mD) = (\mA\mC)\otimes (\mB\mD).
    \end{equation*}
\end{lemma}
\begin{proof}
We prove this by examining the product on the left-hand side, \((\mA \otimes \mB)\,(\mC \otimes \mD)\), 
and showing it coincides block-by-block with \((\mA\mC)\otimes(\mB\mD)\).

We starts from the definition of Kronecker products in blocks.
By definition, the Kronecker product \(\mA \otimes \mB\) can be seen as an \((d \times d)\) block matrix
in which the \((i,j)\)-th block is \(a_{ij}\,\mB\). Hence,
\[
  \mA \otimes \mB
  \;=\;
  \begin{pmatrix}
    a_{11} \mB & a_{12} \mB & \cdots & a_{1n} \mB \\[4pt]
    a_{21} \mB & a_{22} \mB & \cdots & a_{2n} \mB \\[4pt]
    \vdots   & \vdots   & \ddots & \vdots   \\[4pt]
    a_{m1} \mB & a_{m2} \mB & \cdots & a_{mn} \mB
  \end{pmatrix}.
\]
Similarly,
\[
  \mC \otimes \mD
  \;=\;
  \begin{pmatrix}
    c_{11} \mD & c_{12} \mD & \cdots & c_{1r} \mD \\[4pt]
    c_{21} \mD & c_{22} \mD & \cdots & c_{2r} \mD \\[4pt]
    \vdots   & \vdots   & \ddots & \vdots   \\[4pt]
    c_{n1} \mD & c_{n2} \mD & \cdots & c_{nr} \mD
  \end{pmatrix}.
\]

Then, we form the Product \((\mA \otimes \mB)(\mC \otimes \mD)\).
When multiplying two block matrices, we sum over the matching inner block dimensions. Specifically,
the \((i,k)\)-block of \((\mA \otimes \mB)\,(\mC \otimes \mD)\) is given by
\begin{equation*}
  \sum_{j=1}^n \Bigl( \, (a_{ij} \mB) \,\bigl(c_{jk} \mD\bigr) \Bigr).
\end{equation*}
Inside each term, we treat \(a_{ij}\,\mB\) and \(c_{jk}\,\mD\) as scalar-matrix products. We can rewrite
the expression as:
\[
  \sum_{j=1}^n \,
    a_{ij}\,c_{jk}\;\Bigl( \mB \mD \Bigr)
  \;=\;
  \Bigl(\sum_{j=1}^n a_{ij}\,c_{jk}\Bigr)
  \; \mB \mD.
\]
Notice that the factor \(\sum_{j=1}^n a_{ij}\,c_{jk}\) is precisely \((\mA\mC)_{ik}\), the \((i,k)\)-th
entry of the matrix product \(\mA\mC\). Thus, each \((i,k)\)-block of \((\mA \otimes \mB)(\mC \otimes \mD)\)
simplifies to
\[
  (\mA\mC)_{ik}\;(\mB\mD).
\]

Now observe that the Kronecker product \((\mA\mC) \otimes (\mB\mD)\) can also be viewed as an
\((m \times r)\) block matrix whose \((i,k)\)-th block is
\[
  (\mA\mC)_{ik}\;(\mB\mD).
\]
Hence, the \((i,k)\)-th block of \((\mA\mC)\otimes (\mB\mD)\) matches exactly with the \((i,k)\)-th block
we computed for \((\mA \otimes \mB)(\mC \otimes \mD)\).
Since these two matrices agree in every block of a \( d^2\times d^2\) partition, we conclude
\[
  (\mA \otimes \mB)\,(\mC \otimes \mD)
  \;=\;
  (\mA\mC)\otimes (\mB\mD),
\]
as desired.
\end{proof}

\begin{lemma}[Kolmogorov backward theorem, adapted from Theorem 5.11 in ~\cite{sarkka2019applied}]
    \label{lem:bkw_kolmo}
    For a specific SDE, if we denote the transition density from $\rvx(s)$ to $\rvy(t)$ as $p(\vy,t|\vx, s)$
    , then it solves the backward Kolmogorov equation
    \begin{equation*}
        -\frac{\partial p(\vy,t|\vx, s)}{\partial s} = \gL p(\vy,t|\vx,s)
    \end{equation*}
    where $\gL$ denotes the infinitesimal operator of the SDE.
\end{lemma}

\begin{lemma}[The chain rule of KL divergence]
    \label{lem:chain_kl}
    Consider four random variables, $\rvx, \rvz, \tilde{\rvx}, \tilde{\rvz}$, whose underlying distributions are denoted as $p_x, p_z, q_x, q_z$.
    Suppose $p_{x,z}$ and $q_{x,z}$ denotes the densities of joint distributions of $(\rvx,\rvz)$ and $(\tilde{\rvx},\tilde{\rvz})$, which we write in terms of the conditionals and marginals as
    \begin{equation*}
        \begin{aligned}
        &p_{x,z}(\vx,\vz) = p_{x|z}(\vx|\vz)\cdot p_z(\vz)=p_{z|x}(\vz|\vx)\cdot p_{x}(\vx)\\
        &q_{x,z}(\vx,\vz)=q_{x|z}(\vx|\vz)\cdot q_z(\vz) = q_{z|x}(\vz|\vx)\cdot q_x(\vx).
        \end{aligned}
    \end{equation*}
    then we have
    \begin{equation*}
        \begin{aligned}
            \KL{p_{x,z}}{q_{x,z}} = & \KL{p_z}{q_z} + \E_{\rvz\sim p_z}\left[\KL{p_{x|z}(\cdot|\rvz)}{q_{x|z}(\cdot|\rvz)}\right]\\
            = & \KL{p_x}{q_x}+\E_{\rvx \sim p_x}\left[\KL{p_{z|x}(\cdot|\rvx)}{q_{z|x}(\cdot|\rvx)}\right]
        \end{aligned}
    \end{equation*}
    where the latter equation implies
    \begin{equation*}
        \KL{p_x}{q_x}\le \KL{p_{x,z}}{q_{x,z}}.
    \end{equation*}
\end{lemma}

\begin{lemma}[The chain rule of TV distance]
    \label{lem:tv_chain_rule}
    Consider four random variables, $\rvx, \rvz, \tilde{\rvx}, \tilde{\rvz}$, whose underlying distributions are denoted as $p_x, p_z, q_x, q_z$.
    Suppose $p_{x,z}$ and $q_{x,z}$ denotes the densities of joint distributions of $(\rvx,\rvz)$ and $(\tilde{\rvx},\tilde{\rvz})$, which we write in terms of the conditionals and marginals as
    \begin{equation*}
        \begin{aligned}
        &p_{x,z}(\vx,\vz) = p_{x|z}(\vx|\vz)\cdot p_z(\vz)=p_{z|x}(\vz|\vx)\cdot p_{x}(\vx)\\
        &q_{x,z}(\vx,\vz)=q_{x|z}(\vx|\vz)\cdot q_z(\vz) = q_{z|x}(\vz|\vx)\cdot q_x(\vx).
        \end{aligned}
    \end{equation*}
    then we have
    \begin{equation*}
        \begin{aligned}
            \TVD{p_{x,z}}{q_{x,z}} \le  \min & \left\{ \TVD{p_z}{q_z} + \E_{\rvz\sim p_z}\left[\TVD{p_{x|z}(\cdot|\rvz)}{q_{x|z}(\cdot|\rvz)}\right],\right.\\
            &\quad  \left.\TVD{p_x}{q_x}+\E_{\rvx \sim p_x}\left[\TVD{p_{z|x}(\cdot|\rvx)}{q_{z|x}(\cdot|\rvx)}\right]\right\}.
        \end{aligned}
    \end{equation*}
    Besides, we have
    \begin{equation*}
        \TVD{p_x}{q_x}\le \TVD{p_{x,z}}{q_{x,z}}.
    \end{equation*}
\end{lemma}

\end{document}